\documentclass[twoside,11pt]{article}

\usepackage{jmlr2e}

\usepackage{amsmath}
\usepackage{multirow}
\usepackage{relsize} 
\usepackage{rotating}

\usepackage[hsb]{xcolor}
\usepackage{bm}

\usepackage{fontawesome}

\usepackage{caption}
\usepackage{subcaption}

\usepackage[inline]{enumitem}
\makeatletter

\newcommand{\inlineitem}[1][]{%
	\ifnum\enit@type=\tw@
	{\descriptionlabel{#1}}
	\hspace{\labelsep}%
	\else
	\ifnum\enit@type=\z@
	\refstepcounter{\@listctr}\fi
	~\@itemlabel\hspace{\labelsep}%
	\fi}

\usepackage[ruled]{algorithm}
\usepackage{algpseudocode}
\newcommand{\best}{\mbox{\footnotesize best}}
\newcommand{\old}{\mbox{\footnotesize old}}
\algnewcommand{\lIf}[1]{\State\algorithmicif\ #1\ \algorithmicthen}
\algnewcommand{\EndlIf}{\unskip\ \algorithmicend\ \algorithmicif}

\newcommand{\GGM}{GGM}
\newcommand{\RCON}{RCON}%
\newcommand{\PDRCON}{pdRCON}
\newcommand{\PDCG}{pdCG}

\newcommand{\ie}{i.e.}
\newcommand{\Power}{\mathcal{P}ow}
\newcommand{\Part}{\mathcal{P}art}

\newcommand{\G}{\mathcal{G}}
\newcommand{\Hs}{\mathcal{H}}
\newcommand{\V}{\mathcal{V}}
\newcommand{\E}{\mathcal{E}}
\newcommand{\M}{\mathcal{M}}
\newcommand{\F}{\mathcal{F}}
\renewcommand{\P}{\mathcal{P}}
\newcommand{\U}{\mathcal{U}}
\newcommand{\Ls}{\mathcal{L}}

\newcommand{\LL}{\mathbb{L}}
\newcommand{\EE}{\mathbb{E}}

\DeclareMathOperator*{\ePPV}{ePPV}
\DeclareMathOperator*{\eTP}{eTP}
\DeclareMathOperator*{\eP}{eP}
\DeclareMathOperator*{\eTPR}{eTPR}
\DeclareMathOperator*{\eTNR}{eTNR}
\DeclareMathOperator*{\eTN}{eTN}
\DeclareMathOperator*{\eN}{eN}

\DeclareMathOperator*{\sPPV}{sPPV}
\DeclareMathOperator*{\sTP}{sTP}
\DeclareMathOperator*{\sP}{sP}
\DeclareMathOperator*{\sTPR}{sTPR}
\DeclareMathOperator*{\sTNR}{sTNR}
\DeclareMathOperator*{\sTN}{sTN}
\DeclareMathOperator*{\sN}{sN}

\newcommand\precdot{\mathrel{\ooalign{$\prec$\cr
			\hidewidth\raise0.02ex\hbox{$\cdot\mkern0.1mu$}\cr}}}
		
\usepackage{lastpage}
\jmlrheading{25}{2024}{1-\pageref{LastPage}}{3/23; Revised
11/23}{4/24}{23-0295}{Alberto Roverato and Dung Ngoc Nguyen}

\ShortHeadings{Exploration of the search space of GGMs for paired data}{Roverato and  Nguyen}
\firstpageno{1}

\begin{document}

\title{Exploration of the Search Space of Gaussian Graphical Models for Paired Data}

\author{\name Alberto Roverato \email alberto.roverato@unipd.it \\
      \addr Department of Statistical Sciences\\
       University of Padova, Italy.\\
      \AND
       \name Dung Ngoc Nguyen \email ngocdung.nguyen@csiro.au \\
       \addr Department of Statistical Sciences\\
       University of Padova, Italy.\\
       CSIRO Agriculture and Food\\
       Canberra, ACT, Australia.}

\editor{Jin Tian}

\maketitle

\begin{abstract}
We consider the problem of learning a Gaussian graphical model in the case where the observations come from two dependent groups sharing the same variables. We focus on a family of coloured Gaussian graphical models specifically suited for the paired data problem. Commonly, graphical models are ordered by the submodel relationship so that the search space is a lattice, called the model inclusion lattice. We introduce a novel order between models, named the twin order. We show that, embedded with this order, the model space is a lattice that, unlike the model inclusion lattice, is distributive. Furthermore, we provide the relevant rules for the computation of the neighbours of a model. The latter are more efficient than the same operations in the model inclusion lattice, and are then exploited to achieve a more efficient exploration of the search space. These results can be applied to improve the efficiency of both greedy and Bayesian model search procedures. Here, we implement a stepwise backward elimination procedure and evaluate its performance both on synthetic and real-world data.
\end{abstract}


\section{Introduction}

A Gaussian graphical model (\GGM) is a family of multivariate normal distributions whose conditional independence structure is represented by an undirected graph. The vertices of the graph correspond to the variables and every edge missing from the graph implies that the corresponding entry of the concentration matrix, that is the inverse of the covariance matrix, is equal to zero; see \cite{lauritzen1996graphical}.
Gaussian graphical models are widely applied to the joint learning of multiple networks, where the observations come from two or more groups sharing the same variables. The association structure of each group is represented by a network and it is expected that there are similarities between the groups. In this framework, the literature has mostly focused on the case where the groups are independent so that every network is a distinct unit, disconnected from the other networks; see  \citet{tsai2022joint} for a recent review. We consider the case where groups cannot be assumed to be independent and, more specifically, we focus  on the case of paired data, with exactly two dependent groups. The rest of this section is devoted to the presentation of the problem considered in this paper, whereas an overview of related works in this field, as well as of a description of some areas of application, are deferred to Section~\ref{SEC:related.works}.

Coloured \GGM{s} \citep{hojsgaard2008graphical} are undirected graphical models with additional symmetry restrictions in the form of equality constraints on the parameters, which are then depicted on the dependence graph of the model by colouring of edges and vertices. Equality constraints allow one to disclose symmetries concerning both the structure of the network and the values of parameters associated with vertices and edges and, in addition, have the practical advantage of reducing the number of parameters. \citet{roverato2022modelinclusion} introduced a subfamily of coloured \GGM{s} specifically designed to suit the paired data problem that they called \RCON\ models for paired data (\PDRCON). They approached the problem by considering a single coloured \GGM\ comprising the variables of both the first  and the second group. In this way, the resulting model has a graph for each of the two groups and the cross-graph dependence is explicitly represented by the edges across groups; see also  \citet{ranciati2021fused} and \citet{ranciati2023application}.

Although the symmetry restrictions implied by a coloured \GGM\ may usefully reduce the model dimension, the problem of model identification is much more challenging than with classical \GGM{s} because both the dimensionality and the complexity of the search spaces highly increase; see \citet{gao2015estimation,massam2018bayesian,li2020bayesian} and references therein. For the construction of efficient model selection methods, it is therefore imperative to understand the structure of model classes.

A statistical model is a family of probability distributions, and if a model is contained in another model then it is called a submodel of the latter. We can also say that a model is ``larger” than any of its submodels, and model inclusion is typically used to embed a model class with a partial order. In this way, one can easily obtain that the family of \GGM{s} forms a complete distributive lattice. \cite{gehrmann2011lattices} considered four relevant subfamilies of coloured \GGM{s} and showed that they constitute a lattice with respect to model inclusion. However, the structure of such lattices is rather complicated and this makes the identification of neighbouring models, and therefore the implementation of procedures for the exploration of model spaces, much less efficient than for classical \GGM{s}. Furthermore, none of the existing lattices of coloured \GGM{s} satisfies the distributivity property, which is a fundamental property that facilitates the implementation of efficient procedures and representation in lattices; see, among others, \cite{habib2001efficient} and \cite{davey2002introduction}. \citet{roverato2022modelinclusion}  showed that, under the model inclusion order, the class of \PDRCON\ models identifies a proper complete sublattice of the lattice of coloured graphical models, although also in this case the distributive property is not satisfied.

We introduce a novel partial order for the class of \PDRCON\ models that coincides with the model inclusion order if two models are model inclusion comparable but that also includes order relationships between certain models which are model inclusion incomparable. We show that the class of \PDRCON\ models forms a complete lattice also with respect to this order, that we call the \emph{twin lattice}. The twin lattice is distributive and its exploration is more efficient than that of the model inclusion lattice. Hence, the twin lattice can be used to improve the efficiency of procedures, either Bayesian and frequentist, which explore the model space moving between neighbouring models. More specifically, the focus of this paper is on stepwise greedy search procedures, and we show how the twin lattice can be exploited to improve efficiency in the identification of neighbouring submodels.

One way to increase the efficiency of greedy search procedures is by applying the, so-called, \emph{principle of coherence} \citep{gabriel1969simultaneous} that is used as a strategy for pruning the search space. We show that for the family of \PDRCON\ models the twin lattice allows a more straightforward implementation of the principle of coherence.

We implement a stepwise backward elimination procedure with local moves on the twin lattice which satisfies the coherence principle, and we show that it is more efficient than an equivalent procedure on the model inclusion lattice. This procedure is implemented in the statistical programming language \textsf{R} and its behavior is investigated both on synthetic  and  real-world data.

The rest of the paper is organized as follows. Background on lattice theory, coloured graphical models and the structure of their model space is given in  Section~\ref{SEC:background.notation}, whereas Section~\ref{SEC:RCON.models.for.paired.data} introduces the family of \RCON\ models for paired data. Section~\ref{SEC:related.works} contains an overview of the related literature and a discussion of some issues concerning the applications.  In Section~\ref{SEC:alternative.lattice.structure.for.P}, we introduce the twin lattice, derive its properties and describe its relationships with the model inclusion lattice. Section~\ref{SEC:dimension.and.choerence} deals with the dimension of the search space and the implementation of the principle of coherence. The greedy search procedure is described in Section~\ref{SEC:stepwise.procedure}, and then its application to both synthetic and real-world data is presented in Section~\ref{SEC:applications}. Finally, Section~\ref{SEC:conclusions} contains a brief discussion. Proofs are deferred to Section~\ref{SUP.SEC:proofs} of the Appendix.

\section{Background and Notation}\label{SEC:background.notation}

In this section, we review the elements and notation of lattice theory and of coloured graphical models, as required for this paper; for a more comprehensive account we refer to \citet{davey2002introduction}, \citet{lauritzen1996graphical} and \citet{hojsgaard2008graphical}.

\subsection{Partial Orders and Lattices}

A partially ordered set, or \emph{poset}, $\langle A,  \preceq \rangle$ is a structure where $A$ is a set and $\preceq$ is a partial order on $A$. If the elements $a,b \in A$ are such that  $a \preceq b$ with $a \neq b$, then we write $a \prec b$ and say that $a$ is  \emph{smaller} than $b$. Furthermore, if it also holds that there is no element $c \in A$ such that $a \prec c \prec b$ then we say that $a$ is \emph{covered} by $b$ and write $a \precdot b$. If both $a \not\preceq b$ and $b \not\preceq a$ then $a$ and $b$ are \emph{incomparable}.

An element $a \in A$ is called an \emph{upper bound} of the subset $H\subseteq A$ if it has the property that $h \preceq a$ for all $h \in H$ and, furthermore, it is the \emph{supremum} of $H$, denoted by $a = \sup H$, if every upper bound $b$ of $H$ satisfies $a \preceq b$. Accordingly, $a \in A$ is a \emph{lower bound} of $H$ if $a \preceq h$ for all $h \in H$, and it is  the \emph{infimum} of $H$, $a = \inf H$, if for every lower bound $b$ of $H$ it holds that $b \preceq a$.
If for a poset $\langle A,  \preceq \rangle$ the element $\sup A$ exists, then it is called the \emph{maximum element}, or the \emph{unit}, of $A$ and denoted by  $\hat{1}$. Dually, if $\inf A$ exists, then this element is called the \emph{minimum element}, or the \emph{zero}, of $A$ and denoted by  $\hat{0}$.

A poset $\langle A, \preceq \rangle$ is called a \emph{lattice} if $\inf H$ and $\sup H$ exist for every finite nonempty subset $H$ of $A$. A lattice is called \emph{complete} if also both $\inf A$ and $\sup A$ exist. If $\langle A, \preceq \rangle$ is a lattice, then for every pair of elements $a, b \in A$, we write $a \wedge b$ for $\inf \{a,b\}$ and $a \vee b$ for $\sup \{a,b\}$ and refer to $\wedge$ as the \emph{meet operation} and to $\vee$ as the \emph{join operation}. A lattice  $A$ is called \emph{distributive} if  the operations of join and meet distribute over each other; formally, if for all $a, b, c \in A$, $a \vee (b\wedge c) = (a \vee b) \wedge (a\vee c)$.  Finally, it is useful to represent the structure of a lattice graphically by means of a \emph{Hasse diagram} that is a graph with $A$ as vertex set and where two elements $a, b \in A$ are joined by an undirected edge, with the vertex $b$ appearing above $a$, whenever $a \precdot b$; see Figures~\ref{FIG:model-inclusion-lattice} and \ref{FIG:twin-lattice} for examples.

\subsection{Graphical Models and Coloured Graphical Models}

For a finite set $V = \{1,\ldots, p\}$ we let $Y = Y_{V}$ be a continuous random vector indexed by $V$ and we denote by $\Sigma = (\sigma_{ij})_{i,j \in V}$ and $\Sigma^{-1}=\Theta=(\theta_{ij})_{i,j \in V}$ the covariance and concentration matrix of $Y$, respectively. An undirected graph with vertex set $V$ is a pair $G = (V,E)$ where $E$ is an edge set that is a set of unordered pairs of distinct vertices; formally $E\subseteq F_{V}$ with $F_{V}=\{(i,j)\mid  i,j\in V\mbox{ and }i<j\}$. Thus, $F_{V}$ is the edge set of the \emph{complete} graph, that is the graph in which each pair of distinct vertices is connected by an edge. Note that, when it is not clear from the context which graph is under consideration, we will write $V_{G}$ and $E_{G}$ to denote the vertex set and edge set of the graph $G$.

We say that the concentration matrix $\Theta$ is \emph{adapted} to the graph $G$ if every missing edge of $G$ corresponds to a zero entry in $\Theta$; formally, $(i,j) \notin E$, with $i<j$, implies that $\theta_{ij} = \theta_{ji}=0$. A \emph{Gaussian graphical model} (\GGM) with graph $G$ is the family of Gaussian distributions whose concentration matrix is adapted to $G$; see \cite{lauritzen1996graphical}.  We denote by $M\equiv M(V)$ the family of \GGM{s} for $Y_{V}$ and by $M(G)\in M$ the \GGM\ represented by the graph $G=(V, E)$.

A \emph{colouring} of $G=(V, E)$ is a pair $(\V, \E)$ where $\V=\{V_{1},\ldots, V_{v}\}$ is a  partition of $V$ into vertex colour classes and, similarly, $\E=\{E_{1},\ldots, E_{e}\}$ is a partition of $E$ into edge colour classes. Accordingly,  $\G=(\V, \E)$ is a \emph{coloured graph}. Similarly to the notation used for uncoloured graphs, we may write $\V_{\G}$ and $\E_{\G}$ to denote the vertex and edge colour classes of $\G$, respectively. In the graphical representation, all the vertices belonging to a same colour class are depicted of the same colour, and similarly for edges. Furthermore, in order to make coloured graphs readable also in black and white printing, we put a common symbol next to every vertex or edge of the same colour. The only exception to this rule is for vertices and edges belonging to colour classes with a single element, called \emph{atomic}, which are all depicted in black with no symbol next to them.

\cite{hojsgaard2008graphical} introduced \emph{coloured} \GGM{s}, which are \GGM{s} with additional restrictions on the parameter space. The models are represented by coloured graphs, where parameters that are associated with edges or vertices of the same colour are restricted to being identical. In this paper, we focus on the family of coloured \GGM{s} called \RCON\ models by \citet{hojsgaard2008graphical}, because they place equality Restrictions on the entries of the CONcentration matrix. More specifically, in the \RCON\ model with coloured graph $\G=(\V, \E)$ every vertex colour class $V_{i}$, $i=1,\ldots, v$, identifies a set of diagonal concentrations whose value is constrained to be equal, and similarly for edge colour classes which identify subsets of off-diagonal concentrations. We denote by $\M\equiv \M(V)$ the family of \RCON\ models for $Y_{V}$ and by $\M(\G)\in \M$ the \RCON\ model represented by $\G$.

We close this section by noticing that for the families of graphical models we consider every model is uniquely represented by a coloured graph, and in the rest of this paper, with a slight abuse of notation, we will not make an explicit distinction between sets of models and sets of graphs, thereby equivalently writing, for example, $\M(\G)\in \M$ and $\G\in \M$.

\subsection{Exploration of Model Spaces} \label{SEC:lattice-structure-RCON}

The implementation of most model selection procedures requires the exploration of the space of candidate structures, \ie, of the model space. It is therefore fundamental to understand the structure of model spaces so as to improve efficiency and to identify suitable neighbouring relationships to be used in greedy search procedures. In the families of graphical models we consider, this is typically achieved by embedding the model space with the partial order defined through the \emph{model inclusion}, \ie, the \emph{submodel}, relationship thereby obtaining a lattice structure. In the rest of this section, we review the relevant properties of the model inclusion lattices of both undirected and coloured graphical models, and discuss their use in the exploration of model spaces.

Two lattices which become relevant in this context are the so-called \emph{subset} and \emph{partition} lattices. The power set $\Power(A)$ of any finite set $A$ is  naturally embedded with the set inclusion order, $\subseteq$, and forms a complete distributive lattice. This is a well-known lattice structure whose meet and join operations can be efficiently computed because they are the set intersection $\cap$ and the set union $\cup$ operations, respectively \citep[see][]{davey2002introduction}. A partition of $A$ is a collection of nonempty, pairwise disjoint, subsets of $A$ whose union is $A$. The family of partitions $\Part(A)$ of a set $A$ is typically embedded with the partial order where the partition $P_{1}$ is smaller than $P_{2}$ if $P_{1}$ is finer than $P_{2}$, that is if every set in $P_{2}$ can be expressed as a union of sets in $P_{1}$. In this way, one obtains the so-called partition lattice, and it is well-understood that distributivity does not hold for this lattice and, furthermore, the implementation of the meet and of the join operations is more involved, and considerably less efficient, than in the subset lattice; see \cite{canfield2001meet} and \cite{pittel2000typical}.

Given two graphs $G$ and $H$, with vertex set $V$, we write $M(H)\preceq_{s} M(G)$ to denote that $M(H)$ is a submodel of $M(G)$. For the family of \GGM{s}, model inclusion coincides with the subset relationship between edge sets, so that for every pair of graphs, $G$ and $H$, with vertex set $V$, it holds that  $M(H)\preceq_{s} M(G)$ if and only if $E_{H}\subseteq E_{G}$. Hence, $\langle M, \preceq_{s}\rangle$ coincides with the subset lattice of $F_{V}$ and, consequently, the meet and the join operations take an especially simple form because $M(G)\vee M(H)$ and $M(G)\wedge M(H)$ are the models represented by the graphs with edge sets $E_{G}\cup E_{H}$ and $E_{G}\cap E_{H}$, respectively. Furthermore, the lattice $\langle M, \preceq_{s}\rangle$ is complete and distributive.

Consider now the family $\M$ of \RCON\ models for $Y_{V}$ and let $\G$ and $\Hs$ be two coloured graphs with vertex set $V$. It was shown by \cite{gehrmann2011lattices} that  $\M(\Hs)$ is a submodel of $\M(\G)$, \ie,  $\M(\Hs)\preceq_{s}\M(\G)$,
if and only if all of the three following conditions hold true,
\begin{enumerate}[label = (S\arabic*), itemindent=1eM, ref= (S\arabic*) ]
	\item the edge set of $\Hs$ is a subset of the edge set of $\G$;\label{S1}
	\item every colour class in $\V_{\Hs}$ is a union of colour classes in $\V_{\G}$;\label{S2}
	\item every colour class in $\E_{\Hs}$ is a union of colour classes in $\E_{\G}$;\label{S3}
\end{enumerate}
and we remark that condition \ref{S1} can be considered as redundant because, in fact, it is implied by \ref{S3}.
Furthermore, \cite{gehrmann2011lattices} showed that $\langle \M, \preceq_{s} \rangle$ is a complete lattice, although non-distributive. Colour classes are partitions of the vertex and edge set, respectively, and \citet[][eqn.~9]{gehrmann2011lattices} provided explicit rules for the computation of the meet $\M(\G)\wedge \M(\Hs)$ and the join $\M(\G)\vee \M(\Hs)$, which follow from the same operations in the partition lattice and share an equivalent level of computational complexity. The exploration of the space of \RCON\ models is thus much more challenging than the exploration of \GGM{s} because, firstly, as shown by \citet{gehrmann2011lattices} the dimension of $\langle \M, \preceq_{s}\rangle$ is much larger than that of $\langle M, \preceq_{s}\rangle$ and it grows super-exponentially with the number of variables and, secondly,  the exploration of the space through neighbouring models, which requires the application of the meet and the join of models, is more computationally demanding than in $\langle M, \preceq_{s}\rangle$. Indeed, the existing procedures for learning coloured graphical models can deal with a small number of variables \citep{gehrmann2011lattices} or impose substantial approximations \citep{li2018approximate}.

\section{\RCON\ Models for Paired Data} \label{SEC:RCON.models.for.paired.data}

\citet{roverato2022modelinclusion} introduced \PDRCON\ models, which is a subfamily of \RCON\ models specifically designed to deal with paired data. In this section, we introduce the twin-pairing function that allows us to deal efficiently with the paired data setting, and use it to define the family of \PDRCON\ models. We start with a small instance of paired data problem that will be used hereafter as a running example.
\begin{example}[Frets' Heads]
	\citet[][Section~8.4]{whittaker1990graphical} fitted a \GGM\ to the Frets' Heads data  which consist of measurements of the  head \emph{Height} and head \emph{Breadth} of the first and second adult sons in a sample of 25 families. Variables are therefore naturally split  into the variables associated with the first son, $Y_{H}$ and $Y_{B}$, and the second son, $Y_{H^{\prime}}$ and $Y_{B^{\prime}}$, thereby identifying the $L$eft and $R$ight groups depicted in Figure~\ref{FIG:FH.exa-A}, whereas the corresponding block partition of the concentration matrix $\Theta$ is given in Figure~\ref{FIG:FH.exa-con-i}. Note that, for the presentation of the results below it is convenient to use a numerical vertex set, that is  $V=\{1,2,3,4\}$; however, in order to improve readability, in this example we also add subscripts and write $L = \{1_{H}, 2_{B}\}$ and $R = \{3_{H^{\prime}}, 4_{B^{\prime}}\}$.
	\begin{figure}[t]
		\centering
		\begin{subfigure}{0.3\textwidth}
			\centering
			\includegraphics[scale=0.85]{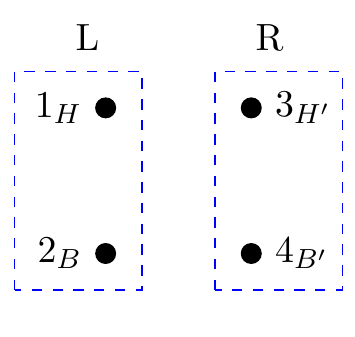}
			\caption{}\label{FIG:FH.exa-A}
		\end{subfigure}
		\begin{subfigure}{0.3\textwidth}
			\centering
			\includegraphics[scale=0.85]{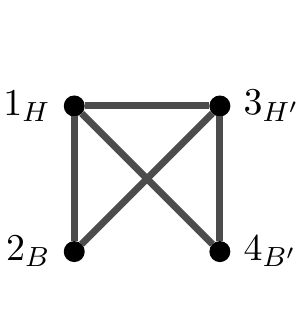}
			\caption{}\label{FIG:FH.exa-B}
		\end{subfigure}
		\begin{subfigure}{0.3\textwidth}
			\centering
			\includegraphics[scale=0.85]{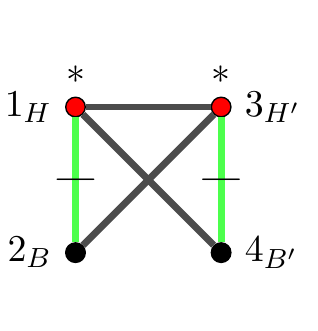}
			\caption{}\label{FIG:FH.exa-C}
		\end{subfigure}
		
		\bigskip
		
		\begin{subfigure}{0.3\textwidth}
			\centering
			\includegraphics[scale=0.85]{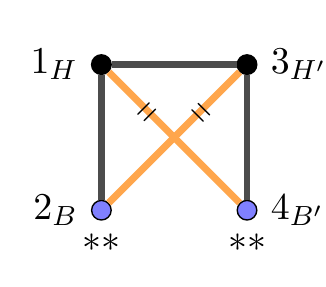}
			\caption{}\label{FIG:FH.exa-D}
		\end{subfigure}
		\begin{subfigure}{0.3\textwidth}
			\centering
			\includegraphics[scale=0.85]{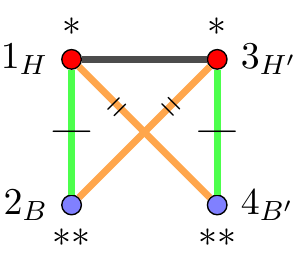}
			\caption{}\label{FIG:FH.exa-E}
		\end{subfigure}
		\begin{subfigure}{0.3\textwidth}
			\centering
			\includegraphics[scale=0.85]{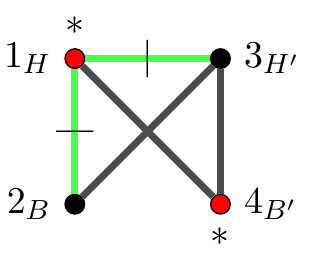}
			\caption{}\label{FIG:FH.exa-F}
		\end{subfigure}
		\caption{Frets' Heads example: (a) gives the partition of the vertex set into the two groups; the graphs (b) to (e) are \PDCG{s} whereas (f) is a coloured graph but not a \PDCG.}\label{FIG:FH.exa}
	\end{figure}

\newcommand{\red}[1]{\textcolor{red}{\bm{#1}}}
\newcommand{\green}[1]{\textcolor{green}{\bm{#1}}}
\newcommand{\blue}[1]{\textcolor{blue!70}{\bm{#1}}}
\newcommand{\orange}[1]{\textcolor{orange}{\bm{#1}}}

\begin{figure}[b]
\small
\renewcommand\thesubfigure{\roman{subfigure}}
\centering
\begin{subfigure}{0.19\textwidth}
\begin{minipage}[c][2.8cm][c]{0.19\textwidth}
\begin{eqnarray*}
    \left(
    \begin{array}{cc}
    \Theta_{LL} & \Theta_{LR}\\
    \Theta_{RL} & \Theta_{RR}\\
    \end{array}
    \right)
\end{eqnarray*}
\end{minipage}
	\caption{}\label{FIG:FH.exa-con-i}
\end{subfigure}
%
	\begin{subfigure}{0.24\textwidth}
	\centering
    \begin{eqnarray*}
        \left[
        \begin{array}{cc|cc}
            \red{\alpha} &\green{\delta} & \theta_{13}  &\theta_{14} \\[1px]
                        &\theta_{22} & \theta_{23}  &0           \\[1px]
        \hline
                        &            & \red{\alpha}  &\green{\delta} \\[1px]
                        &            &              &\theta_{44}
        \end{array}
        \right]
    \end{eqnarray*}
	\caption{}\label{FIG:FH.exa-con-ii}
	\end{subfigure}
	\begin{subfigure}{0.27\textwidth}
	\centering
    \begin{eqnarray*}
        \left[
        \begin{array}{cc|cc}
            \theta_{11} &\theta_{12} & \theta_{13}  &\orange{\gamma} \\[1px]
                        &\blue{\beta} & \orange{\gamma}  &0           \\[1px]
        \hline
                        &            & \theta_{33}  &\theta_{34} \\[1px]
                        &            &              &\blue{\beta}
        \end{array}
        \right]
    \end{eqnarray*}
	\caption{}\label{FIG:FH.exa-con-iii}
	\end{subfigure}
	\begin{subfigure}{0.24\textwidth}
	\centering
    \begin{eqnarray*}
        \left[
        \begin{array}{cc|cc}
            \red{\alpha} &\green{\delta} & \theta_{13}  &\orange{\gamma} \\[1px]
                        &\blue{\beta} & \orange{\gamma} &0           \\[1px]
        \hline
                        &            & \red{\alpha}  &\green{\delta} \\[1px]
                        &            &              &\blue{\beta}
        \end{array}
        \right]
    \end{eqnarray*}
	\caption{}\label{FIG:FH.exa-con-iv}
	\end{subfigure}
\caption{Frets' Heads example: (i) shows the block partition of the concentration matrix, whereas (ii) to (iv) give the concentration matrices for the three \PDCG{s} of Figure~\ref{FIG:FH.exa}, with vertical and horizontal lines to highlight the four block partition. More specifically, matrix (ii) corresponds to graph (c); matrix (iii) to graph (d) and matrix (iv) to graph (e).}\label{FIG:FH.exa.con}
\renewcommand\thesubfigure{\alph{subfigure}}
\end{figure}

\end{example}

In paired data problems, every variable in $Y_{V}$ has a homologous, or \emph{twin}, variable in $Y_{V}$, and the  \emph{twin-pairing} function $\tau$, or twin function for short, associates every element $i\in V$ with its twin $\tau(i)\in V$, in such a way that  if $\tau(i)=j$ then $\tau(j)=i$. We then naturally extend the use of the twin function to edges as well as to collections of vertices and edges. Formally, for $(i, j)\in F_{V}$ we let $\tau(i,j)=(\tau(i),\tau(j))$ whereas for $A\subseteq V$ we set $\tau(A)=\{\tau(i)\mid i\in A\}$ and if $E\subseteq F_{V}$ then $\tau(E)=\{\tau(i,j)\mid (i,j)\in E\}$, with the convention that $\tau(\emptyset)=\emptyset$. Note that it may happen that  $\tau(i, j)=(\tau(i), \tau(j))$ is such that $\tau(i)>\tau(j)$ and in this case we will implicitly assume that the endpoints of such pairs are reversed, so as to obtain a proper edge belonging to $F_{V}$.

The twin function can be used to partition the set $V$ into two sets, $L\cup R=V$, with $L\cap R=\emptyset$, such that $\tau(L)=R$. This partition is not unique, and the theory developed in this paper requires the identification of one of such partitions, which can  be arbitrarily chosen. Nevertheless, there typically exists a natural partition where $L$ is associated with the first group and $R$ with the second group and in our examples we will, without loss of generality, always refer to such natural partition.  Accordingly, to facilitate the interpretation of a model, we will depict the coloured graph with vertices $L$ on the $L$eft and vertices $R$ on the $R$ight.

In this context, the focus is on similarities and differences between the two groups. Any vertex $i\in V$ and its twin $\tau(i)$ are associated with the diagonal entries $\theta_{ii}$ and $\theta_{\tau(i)\tau(i)}$  of $\Theta$, respectively, and thus a hypothesis of interest is given by the equality restriction $\theta_{ii}=\theta_{\tau(i)\tau(i)}$. The latter is encoded by the vertex colour class $\{i, \tau(i)\}$, which we call \emph{twin-pairing}. Similarly, the edge $(i,j)$ and its twin $\tau(i,j)$ are associated with the off-diagonal concentrations $\theta_{ij}$ and $\theta_{\tau(i)\tau(j)}$, respectively, and the interest is for the equality $\theta_{ij}=\theta_{\tau(i)\tau(j)}$ implied by the edge twin-pairing colour class $\{(i,j), \tau(i, j)\}$. More formally, a twin-pairing colour class has cardinality two and contains either a pair of twin vertices $\{i, \tau(i)\}$ or a pair of twin edges $\{(i,j), \tau(i, j)\}$, with $i\neq \tau(j)$ so that $(i,j)\neq \tau(i,j)$.
\begin{definition}[\citealp{roverato2022modelinclusion}]\label{DEF:PDCG}
	Let $\G=(\V, \E)$ be a coloured graph with vertex set $V$ and let $\tau$ be a twin-pairing function on $V$. We say that $\G=(\V, \E)$ is a \emph{coloured graph for paired data} (\PDCG) if every colour class is either atomic or twin-pairing. If an \RCON\ model is defined by a \PDCG\ then we say that it is an \emph{\RCON\ model for paired data} (\PDRCON) and denote it by $\P(\G)$. Moreover, we will denote by $\P=\P(V)$ the family of \PDRCON\ models for $Y_{V}$.
\end{definition}
The family of \PDRCON\ models $\P$ is a subfamily of \RCON\ models $\M$, \ie, $\P\subseteq \M$, and therefore it makes sense to embed this space with the model inclusion order $\preceq_{s}$. \citet{roverato2022modelinclusion} showed that $\langle \P, \preceq_{s}\rangle$ is a proper sublattice of $\langle \M, \preceq_{s}\rangle$; however, as well as $\langle \M, \preceq_{s}\rangle$ also $\langle \P, \preceq_{s}\rangle$ is non-distributive.

For the interpretation of a \PDRCON\ model, it is useful to distinguish between two different types of edge pairings, which we now describe for the case where $L$ and $R$ correspond to the natural group partition of $V$. Indeed, in this case the comparison of the association structure within the first group with that within the second group is carried out considering the pairs $i<j$ such that $i,j\in L$, because in this case $\tau(i),\tau(j)\in R$. Hence, the simultaneous absence or presence of the edges $(i,j)$ and $\tau(i,j)$ identifies a similarity in the internal association structure of the two groups. Moreover, if $(i,j)$ and $\tau(i,j)$ are both present, then it may further hold true that $\theta_{ij}=\theta_{\tau(i)\tau(j)}$, which is a constraint encoded by the twin-pairing colour class $\{(i,j), \tau(i,j)\}$. On the other hand, the analysis of the across-group relationships involves the vertices $i,j\in V$ with $i\in L$, $j\in R$ and $i\neq \tau(j)$, and also in this case the simultaneous absence or presence of both $(i,j)$ and $\tau(i,j)$ identifies a similarity in the association structure between the across-group variables $Y_{i}$ and $Y_{j}$ and that of their across-group twins $Y_{\tau(i)}$ and $Y_{\tau(j)}$ with, potentially, $\theta_{ij}=\theta_{\tau(i)\tau(j)}$. We also note that symmetries involving the internal structure of the two groups correspond to similarities between the diagonal blocks $\Theta_{LL}$ and $\Theta_{RR}$ of $\Theta$, whereas every across-group symmetry corresponds to a similarity between two entries of $\Theta$ belonging to the upper and  lower triangular part, respectively, of the off-diagonal block $\Theta_{LR}$.
\begin{example}[Frets' Heads continued]
The five graphs of Figure~\ref{FIG:FH.exa} have all the same edge set and thus differ only for the colour classes.  The graph in Figure~\ref{FIG:FH.exa-B} has only atomic colour classes and represents an uncoloured \GGM\ or, equivalently, a trivial \PDRCON\ model with no equality restrictions. The graph in Figure~\ref{FIG:FH.exa-C} is  associated with the concentration matrix in Figure~\ref{FIG:FH.exa-con-ii}, and it represents a \PDRCON\ model with one vertex twin-pairing colour class $\{1_{H}, 3_{H^{\prime}}\}$, and one edge twin-pairing colour class $\{(1_{H}, 2_{B}), (3_{H^{\prime}}, 4_{B^{\prime}})\}$, the latter involving one edge within the first group and one edge within the second group.
	The graph in Figure~\ref{FIG:FH.exa-D} is  associated with the concentration matrix in Figure~\ref{FIG:FH.exa-con-iii}, and it represents a \PDRCON\ model with one vertex twin-pairing colour class $\{2_{B}, 4_{B^{\prime}}\}$, and one edge twin-pairing colour class $\{(1_{H}, 4_{B^{\prime}}), (2_{B}, 3_{H^{\prime}})\}$, the latter involving edges across groups. The graph in Figure~\ref{FIG:FH.exa-E} is  associated with the concentration matrix in Figure~\ref{FIG:FH.exa-con-iv}, and it represents a \PDRCON\ model with all possible twin-paring classes and, finally, the graph in Figure~\ref{FIG:FH.exa-F} is not a \PDCG\ because neither the vertex colour class $\{1_{H}, 4_{B^{\prime}}\}$ nor the edge colour class $\{(1_{H}, 3_{H^{\prime}}), (1_{H}, 2_{B})\}$ are twin-pairing.
\end{example}

\section{Related Works and Application Issues}\label{SEC:related.works}
The problem of comparing the distribution of a set of variables between two experimental conditions, or groups, is common to many statistical applications. In this context, paired data commonly arise in paired design studies, where each subject is measured twice at two different time points or under two different situations, as well as from matched observational studies. For instance, in biomedical studies  paired designs are frequently used because they make it possible to account for individual-specific effects, and consequently to reduce the variability in the observations due to differences between individuals. This also extends to high-throughput experimental techniques, with a relevant example provided by cancer genomics where control samples are often obtained from histologically normal tissues adjacent to the tumor (NAT) \citep{hardcastle2013empirical,aran2017comprehensive}.

When the interest is for the association structure represented by a \GGM{}, then the analysis of paired data amounts to the joint learning of the structure of a graph for  each of the two groups, by accounting for the cross-graph association structure. We approach this problem by considering the family of \PDRCON\ models, introduced by \citet{roverato2022modelinclusion}. More specifically, \citet{roverato2022modelinclusion} showed that \PDRCON\ models form a proper sublattice of the lattice of coloured graphical models, gave
rules for the computation of  the meet and joint operations, and implemented such rules in a stepwise model search procedure. Hence, \citet{roverato2022modelinclusion}
considered the results provided by \citet{gehrmann2011lattices} for \RCON\ models, and derived the specific form they take when restricted to the subset of \PDRCON\ models.  In this way, some improvement in terms of efficiency could be achieved but, nevertheless, the  \PDRCON\ model inclusion sublattice still behaves like a partition lattice, and shares the inefficiencies of the latter. In the next section, we introduce a lattice for the family of \PDRCON\ models that behaves like the subset lattice, and this allows us to exploit the properties of the latter in the exploration of the model space, as well as  to gain insight into the structure of this model class.

One way to avoid the explicit exploration of the model space is by using penalized likelihood methods, which can be applied to problems of larger dimensions. \citet{ranciati2023application}, elaborating on previous work by  \citet{ranciati2021fused}, introduced a graphical lasso method for learning \PDRCON\ models; see also \citet{li2021penalized} and \citet{wit2015factorial} for applications of penalized likelihood methods to coloured graphical modelling. The applications considered in \citet{ranciati2021fused} and \citet{ranciati2023application} concern the identification of brain networks from fMRI data and of gene networks from breast cancer gene expression data, respectively, which are contexts where variables are measured on the same scale. Indeed, the scale the variables are measured plays a relevant role in the procedures for learning \PDRCON\ models. \citet[][Section~8]{hojsgaard2008graphical} remarked that the comparison of concentration values is meaningful only when variables are measured on comparable scales, and they recommended that \RCON\ models should be used only in this case. We note, however, that a less stringent condition is required in \PDRCON\ models because, the fact that only twin-pairing colour classes are allowed, implies that only homologous variables need to have comparable scales, that is a condition easily satisfied in practice. Hence, \PDRCON\ models can be meaningfully applied also when the scales of non-homologous variables are not comparable. However, in this case the application of graphical lasso methods is problematic. Indeed, on the one hand the result of the graphical lasso is not invariant to scalar multiplications of the variables and, for this reason, it is common practice to apply it to standardized data; see, among others, \citet[][p.~8]{hastie2015statistical} and \citet{carter2023partial}. On the other hand, as noticed by \citet[][Section~3.4]{hojsgaard2008graphical}, \RCON\ models are not invariant under rescaling, in the sense that standardization will not preserve the original structure of colour classes. In Section~\ref{SEC:applications}, we present two applications: the first to the same fMRI data previously analysed by \citet{ranciati2021fused} and \citet{roverato2022modelinclusion}, and the second to a dataset on air quality data where non-homologous variables have different scales, and a comparison with the graphical lasso procedure of \citet{ranciati2023application} is carried out.

An alternative approach to the joint learning of dependent \GGM{s} was introduced by \citet{xie2016joint} and \citet{zhang2022bayesian} where multiple groups are considered and the cross-graph dependence is modelled by means of a latent vector representing systemic variation  manifesting simultaneously in all groups. The latter methods have the advantage that they can deal with an arbitrary number of dependent groups but, on the other hand, they do not provide an explicit representation of the association structure between groups. \citet{wit2015factorial} introduced factorial \GGM{s}, which are suited for dealing with longitudinal multivariate data observed at $T$ time points; see also \citet{vinciotti2016model}. In factorial \GGM{s} the concentration matrix is partitioned into blocks in such a way that each of the $T$ diagonal blocks gives the graph structure at a different time point, whereas the off-diagonal blocks encode the cross-graphs dependence structure. More specifically, factorial \GGM{s} are \RCON\ models because they also allow to include equality constraints between entries of the concentration matrix and, thus, \PDRCON\ models could be regarded as a subfamily of factorial \GGM{s}. However, the available theory of factorial \GGM{s} does not include methods for learning the colour classes, which are assumed to be known.



\section{An Alternative Lattice Structure of $\mathbf{\P}$} \label{SEC:alternative.lattice.structure.for.P}

We now introduce a novel lattice structure for the family of \PDRCON\ models that is based on an alternative representation of \PDCG{s}.

\subsection{An Alternative Characterization of Coloured Graphs for Paired Data}\label{SEC:Notations}

The classical representation of coloured graphs is through their colour classes. In this section we introduce a different representation of \PDCG{s} that will allow us to deal more efficiently with these objects. A detailed example that may help to follow the steps required for the construction of the novel representation is given in Section~\ref{SUP.SEC:edge.set.partition} of the Appendix.

If $(L, R)$ is a partition induced by the twin function $\tau$ we can, without loss of generality, number the variables so that $L=\{1,\ldots, q\}$ and  $R=\{q+1,\ldots, p\}$. The latter numbering allows us to identify the following three subsets of $F_{V}$,
\begin{align*}
	F_{L}=\{(i,j)\in F_{V}\mid i<\tau(j)\},\; F_{R}=\{(i,j)\in F_{V}\mid i>\tau(j)\},\; F_{T}=\{(i,j)\in F_{V}\mid i=\tau(j)\}.
\end{align*}
It is straightforward to see that the triplet $(F_{L}, F_{R}, F_{T})$ forms a partition of $F_{V}$ with the property that $F_{R}=\tau(F_{L})$, $F_{L}=\tau(F_{R})$ and $F_{T}=\tau(F_{T})$. The subscripts we use follow from the fact that, for every $(i,j)\in F_{V}$, it holds that if $i,j\in L$ then  $(i,j)\in F_{L}$ whereas if $i,j\in R$ then  $(i,j)\in F_{R}$. We remark, however, that both  $F_{L}$ and $F_{R}$ contain edges $(i,j)$ such that $i\in L$ and $j\in R$. More specifically, the partition of the
cross-graph edges between $F_{L}$ and $F_{R}$ is not unique, because any partition with the property that if $(i, j)\in F_{L}$ then $\tau(i,j)\in F_{R}$ can be equivalently used. Indeed, the partition implied by $\tau$ merely depends on the variable numbering. Finally, the subscript of $F_{T}$ recalls that every edge in this set links a vertex to its twin.

For a coloured graph $\G=(\V, \E)$, we set
\begin{align}\label{EQN:uncoloured.version}
	V_{\G} = \cup_{j = 1}^{v} V_{j}\qquad\mbox{and}\qquad E_{\G} = \cup_{j = 1}^{e} E_{j},
\end{align}
and call $(V_{\G}, E_{\G})$ the \emph{uncoloured version} of $\G$. Hence, if $G=(V, E)$ is the uncoloured version of  $\G$ then
the above partition of $F_{V}$ naturally induces a partition  $(E_{L}, E_{R}, E_{T})$ of $E$ where $E_{L}=E\cap F_{L}$,  $E_{R}=E\cap F_{R}$ and $E_{T}=E\cap F_{T}$. Then, we introduce the following two sets which can be associated to a \PDCG\ $\G=(\V, \E)$,
\begin{align}\label{EQ:bbL.and.BBE}
	\LL = \{i \in L \mid \{ i\} \in \V\}
	\quad\mbox{and}\quad \EE = \{(i,j) \in E_{L} \cap \tau(E_{R}) \mid \{(i, j)\} \in \E\},
\end{align}
so that $\LL\subseteq L$  is made up of the vertices in $L$ which belong to an atomic colour class, whereas an edge $(i,j)$ belongs to $\EE$  if and only if (i) $(i,j)$ belongs to $E_{L}$, and (ii) its twin $\tau(i,j)$ is present in the graph and (iii) both $(i,j)$ and $\tau(i, j)$ form atomic colour classes. Hence, through equations (\ref{EQN:uncoloured.version}) and (\ref{EQ:bbL.and.BBE}), we can associate to any \PDCG\ $\G$ two sets of vertices, that is $V$ and $\LL$, and two sets of edges, that is $E$ and $\EE$.
\begin{example}[Frets' Heads continued]
All the \PDCG{s} of Figure~\ref{FIG:FH.exa} have common vertex set $V=\{1_{H}, 2_{B}, 3_{H^{\prime}}, 4_{B^{\prime}}\}$ and edge set $E=\{(1_{H}, 2_{B}), (1_{H}, 3_{H^{\prime}}), (1_{H}, 4_{B^{\prime}}), (2_{B}, 3_{H^{\prime}}),\allowbreak (3_{H^{\prime}}, 4_{B^{\prime}})\}$. The graph in Figure~\ref{FIG:FH.exa-B} has no twin-pairing colour classes so that $\LL=L=\{1_{H}, 2_{B}\}$ and $\EE=E_{L}\cap\tau(E_{R})=\{(1_{H}, 2_{B}), (1_{H}, 4_{B^{\prime}})\}$. For the graph in  Figure~\ref{FIG:FH.exa-C} one has $\LL=\{2_{B}\}$ and $\EE=\{(1_{H}, 4_{B^{\prime}})\}$ whereas the graph  Figure~\ref{FIG:FH.exa-D} is such that $\LL=\{1_{H}\}$ and $\EE=\{(1_{H}, 2_{B})\}$. Finally, the graph in  Figure~\ref{FIG:FH.exa-E} has all possible twin-pairing colour classes and therefore $\LL=\EE=\emptyset$.
\end{example}

We will show below that every \PDCG\ can be uniquely  represented by its uncoloured version, together with the  subset of vertices $\LL$ that characterizes the vertex colour classes, and the subset of edges $\EE$ that characterizes the edge colour classes. Hence, the quadruplet $(V, E, \LL, \EE)$ provides an alternative representation of \PDCG{s}. Firstly, we show that from the four sets $(V, E, \LL, \EE)$ it is possible to recover the coloured graph $\G=(\V, \E)$ from where they were computed.
\begin{proposition}\label{THM:new.notation.inverse.function}
	Let $\G=(\V, \E)$ be a \PDCG\ and let $(V, E, \LL, \EE)$ the quadruplet obtained from the application of  (\ref{EQN:uncoloured.version}) and (\ref{EQ:bbL.and.BBE}) to $\G$ with respect to a partition $(L, R)$ induced by the twin function $\tau$. Then, the colour classes $(\V, \E)$ of $\G$ can be recovered from the quadruplet as follows,
	\begin{enumerate}[label = (\roman*), ref=(\roman*) ]
		\item the twin-pairing colour classes in $\V$ are the sets
		$\{i, \tau(i)\}$ for all $i\in L\setminus\LL$, and the remaining vertices in $V$ form vertex atomic classes;
		\item the twin-pairing colour classes in $\E$ are the sets $\{(i,j), \tau(i,j)\}$ for all $(i,j)\in (E_{L}\cap \tau(E_{R}))\setminus \EE$ and the remaining edges in $E$ form  edge atomic classes.
	\end{enumerate}
\end{proposition}
\begin{proof}
	See Section~\ref{SUP.SEC:proof-inverse-function} of the Appendix.
\end{proof}

The quadruplet  $(V, E, \LL, \EE)$  characterizes a \PDCG\ by means of an uncoloured undirected graph together with a subset of its vertices, $\LL$, and a subset of its edges, $\EE$.  However, not all quadruplets of this type can be used to represent a \PDCG\ and we need the following.
\begin{definition} Let $G=(V, E)$ be an undirected graph and $\tau$ a twin function on $V$. We say that the collection of sets $(V, E, \LL, \EE)$ is \emph{compatible} with the partition $(L, R)$ induced by $\tau$ if $\LL\subseteq L$ and $\EE\subseteq E_{L}\cap\tau(E_{R})$ .
\end{definition}
And we can now give the main result of this section.
\begin{theorem}\label{THM:existing-novel-notations}
	Let $\tau$ a twin function on $V$ and $(L, R)$ a partition of $V$ induced by $\tau$. Then,
	equations (\ref{EQN:uncoloured.version}) and (\ref{EQ:bbL.and.BBE}) establish a one-to-one correspondence between the family of \PDCG{s} $\P$ and the collection of quadruplets $(V, E, \LL, \EE)$ compatible with  $(L,R)$.
\end{theorem}
\begin{proof}
	See Section~\ref{SUP.SEC:proof-existing-novel-notations} of the Appendix.
\end{proof}
We have thus shown that the family of compatible quadruplets $(V, E, \LL, \EE)$ provides a representation of $\P$ alternative to colour classes. In fact, such alternative representation may look less intuitive than the traditional representation. However, it is easier to implement in computer programs and, as shown below, this representation naturally leads to the definition of a useful alternative lattice structure of $\P$.


\subsection{The Twin Lattice}\label{SEC:def-twin-lattice}

In this section, we introduce a novel partial order for \PDCG{s} and show that, embedded with this order, the set $\P$ forms a distributive lattice that we call the twin lattice. Unlike the model inclusion lattice that inherits its properties from the partition lattice, the twin lattice behaves in a way similar to the subset lattice and therefore it satisfies the distributivity property, and the join and meet operations can be efficiently computed as union and intersection of sets, respectively.

For the theory developed in the following, we need to identify a partition $(L, R)$ of $V$ that is induced by $\tau$, but otherwise arbitrary. Furthermore, a \PDCG\  $\G\in \P$ will be equivalently represented by using its colour classes, $(\V, \E)$ or its quadruplet $(V, E, \LL, \EE)$. We can now introduce the novel twin order.
\begin{definition}\label{DEF:twin-order}
	Let $\Hs=(V, E_{\Hs}, \LL_{\Hs}, \EE_{\Hs})$ and $\G=(V, E_{\G}, \LL_{\G}, \EE_{\G})$ be two \PDCG{s} in $\P$.  Then we say  that $\Hs\preceq_{t} \G$ if and only if
	\begin{enumerate}[label = (T\arabic*), ref=(T\arabic*) ]
		\item $E_{\Hs}\subseteq E_{\G}$, \label{iSV}
		\inlineitem  $\LL_{\Hs} \subseteq \LL_{\G}$, \label{iiSV}
		\inlineitem $\EE_{\Hs} \subseteq \EE_{\G}$.\label{iiiSV}
	\end{enumerate}
	and call $\preceq_{t}$ the \emph{twin order}.
\end{definition}
We can thus see that the twin order is naturally associated with the alternative representation of \PDCG{s} introduced in the previous section, and that, when using that representation, it is straightforward to check whether the order relationship holds for two graphs. Because the twin order generalizes the subset order, also the associated lattice turns out to be a natural extension of the subset lattice and to share its  properties.
\begin{theorem}\label{THM:LatticePropertiesofOrderTau}
	The family of \PDCG{}s on the vertex set $V$, equipped with the twin order, that is $\langle \P, \preceq_{t}\rangle$, forms a complete distributive lattice, that we call the \emph{twin lattice}, where
	if $\G,\Hs\in\P$,
	\begin{enumerate}[label = (\roman*), ref=(\roman*) ]
		\item\label{meet.twin.lattice} the meet of $\G$ and $\Hs$ can be computed as
		\begin{align*}
			\G \wedge_{t} \Hs = \left(V,\; E_{\G} \cap E_{\Hs},\; \LL_{\G}\cap \LL_{\Hs},\; \EE_{\G} \cap \EE_{\Hs} \right);
		\end{align*}

		\item\label{join.twin.lattice} the join  of $\G$ and $\Hs$ can be computed as
		\begin{align*}
			\G \vee_{t} \Hs =  \left(V,\; E_{\G} \cup E_{\Hs},\; \LL_{\G}\cup \LL_{\Hs},\; \EE_{\G} \cup \EE_{\Hs} \right).
		\end{align*}
	\end{enumerate}
	Furthermore, the unit is $\hat{1}=\left(V,\; F_{V},\; L,\; F_{L} \right)$, that is the uncoloured complete graph, whereas the zero is  $\hat{0}=\left(V,\;\emptyset,\;\emptyset,\; \emptyset \right)$, that is the graph with no edge and such that all vertices belong to  twin-pairing classes.
\end{theorem}
\begin{proof}
	See Section~\ref{SUP.SEC:proof-LatticePropertiesofOrderTau} of the Appendix.
\end{proof}


\subsection{Application of the Twin Lattice in Model Search} \label{SEC:relationships.between.lattices}


In order to exploit the properties of the twin lattice in model search it is useful to investigate the existing relationships between the twin lattice and the model inclusion lattice. We first show that the twin order $\preceq_{t}$ is a refinement of the model inclusion order $\preceq_{s}$.
\begin{proposition}\label{THM:relation-modelinclusion-twinorder}
	For any $\Hs, \G \in \P$, the following relationships between the model inclusion order and the twin order exist,
	\begin{enumerate}[label = (\roman*), ref=(\roman*) ]
		\item if $\Hs \preceq_{s} \G$ then $\Hs \preceq_{t} \G$\label{rel.s.t.i}
		\item if $\Hs \preceq_{s} \G$ and $\Hs \precdot_{t} \G$ then $\Hs \precdot_{s} \G$.\label{rel.s.t.ii}
	\end{enumerate}
\end{proposition}
\begin{proof}
	See Section~\ref{SUP.SEC:proof-relation-modelinclusion-twinorder} of the Appendix.
\end{proof}
We remark that the reverse of (i) in Proposition~\ref{THM:relation-modelinclusion-twinorder} does not hold true, and it is not difficult to find two \PDCG{s}, $\G$ and $\Hs$, such that $\G \preceq_{t} \Hs$ but that are model inclusion incomparable.

For the exploration of the model space it is useful to be able to identify the neighbours of a model $\P(\G)$. Here we focus on stepwise backward elimination procedures and therefore on the  neighbouring submodels of $\P(\G)$, which are the models $\P(\Hs)\subset\P(\G)$ such that there is no graph $\F\in \P$ with $\P(\Hs)\subset\P(\F)\subset\P(\G)$ or, equivalently, the models $\P(\Hs)$ such as $\Hs$ is covered by $\G$ in the model inclusion lattice, $\Hs\precdot_{s}\G$.
\begin{proposition}\label{THM:neighbor-submodels}
	Let $\P(\G)$ be the \PDRCON\ model represented by the graph
	$\G = (V, E_{\G}, \LL_{\G}, \EE_{\G})$.
	Then the set of \PDCG{s} representing the neighbouring submodels of $\P(\G)$, that is the subset of graphs $\Hs\in \P$ such that $\Hs\precdot_{s} \G$, is made up of,
	\begin{enumerate}[label = (\alph*), ref=(\alph*)]
		\item all the graphs obtained by merging exactly two vertex atomic colour classes of $\G$ to obtain a vertex twin-pairing colour class, which can be formally computed as
		\begin{tabbing}
			\=xxxxx\=x\kill
			\>(i)    \>$\Hs=(V, E_{\G},\LL_{\G}\setminus \{i\}, \EE_{\G})$ for all $i\in \LL_{\G}$;
		\end{tabbing}
		\item all the graphs obtained by merging exactly two edge atomic colour classes of $\G$ to obtain an edge twin-pairing colour class, which can be formally computed as
		\begin{tabbing}
			\=xxxxx\= x\kill
			\>(ii)   \>$\Hs=(V, E_{\G}, \LL_{\G}, \EE_{\G}\setminus \{(i,j)\})$ for all $(i, j)\in \EE_{\G}$;
		\end{tabbing}
		\item all the graphs obtained by removing exactly one edge atomic colour class from $\G$, which can be formally computed as
		\begin{tabbing}
			\=xxxxx\= $\Hs=(V,\;$\=$E_{\G} \setminus \{(i,\tau(i))\}$\=$,\,\LL_{\G}$
			\=$,\,\EE_{\G} \setminus \{(i,j)\}$\=$)$x\=xx\kill
			\>(iii)   \>$\Hs=(V,$\>$E_{\G}\setminus{(i,j)}$\>$,\,\LL_{\G}$\>$,\,\EE_{\G} \setminus \{(i,j)\}$\>$)$\> for all $(i,j)\in \EE_{\G}$;\\[.5ex]
			\>(iv)   \>$\Hs=(V,$\>$E_{\G}\setminus{\tau(i,j)}$\>$,\,\LL_{\G}$\>$,\,\EE_{\G} \setminus \{(i,j)\}$\>$)$\> for all $(i,j)\in \EE_{\G}$;\\[.5ex]
			\>(v)  \>$\Hs=(V,$\>$E_{\G}\setminus{(i,j)}$\>$,\,\LL_{\G}$\>$,\,\EE_{\G}$\>$)$\> for all $(i,j)\in E_{\G}$ such that $\tau(i,j)\not\in E_{\G}$;\\[.5ex]
			\>(vi)  \>$\Hs=(V,$\>$E_{\G} \setminus \{(i,\tau(i))\}$\>$,\,\LL_{\G}$\>$,\,\EE_{\G}$\>$)$\> for all $i\in V$ such that  $(i,\tau(i)) \in E_{\G}$;\\
		\end{tabbing}
		
		\item all the graphs obtained by removing exactly one edge twin-pairing colour class from $\G$, which can be formally computed as
		
		\begin{tabbing}
			\=xxxxx\=$\Hs=(V, E_{\G}\setminus\{(i,j), \tau(i,j)\}, \LL_{\G}, \EE_{\G})$ x\=x \kill
			\>(vii) \>$\Hs=(V, E_{\G}\setminus\{(i,j), \tau(i,j)\}, \LL_{\G}, \EE_{\G})$\> for all $(i,j), \tau(i,j) \in E_{\G}$ such that\\[0.5ex]
			\>\>\>$(i,j)\neq \tau(i,j)$and both $(i,j)\notin \EE_{\G}$\\[0.5ex]
            \>\>\>and $\tau(i,j)\notin \EE_{\G}.$
		\end{tabbing}
	\end{enumerate}
\end{proposition}
\begin{proof}
	See Section~\ref{SUP.SEC:proof-neighbor-submodels} of the Appendix.
\end{proof}
\begin{example}[Frets' Heads continued]
Figure~\ref{FIG:model-inclusion-lattice} gives the structure of a portion of the Hasse diagram of the model inclusion lattice of \PDRCON\ models in the Frets' heads example. The graph on the top represents the saturated model, and it is linked to its neighbouring submodels, represented by the coloured graphs in the highlighted area. If we number the graphs in the highlighted area from left to right, then graphs 1 and 2 are obtained from the graph of the saturated model by applying (i) of Proposition~\ref{THM:neighbor-submodels},
graphs 5 and 8 from point (ii), graphs 6,  7, 9 and 10 from points (iii) and  (iv) and, finally, graphs 3 and 4 from point (vi). Points (v) and (vii) cannot be applied to the saturated model but they can be applied to obtain the two graphs depicted in the lowest level. More specifically, the two graphs can be obtained both by applying (v) to the graphs 6, 7, 9 and  10, and by applying (vii) to the graphs 5 and 8.
\begin{figure}[tbp]
	\begin{center}
		\includegraphics[scale=.8]{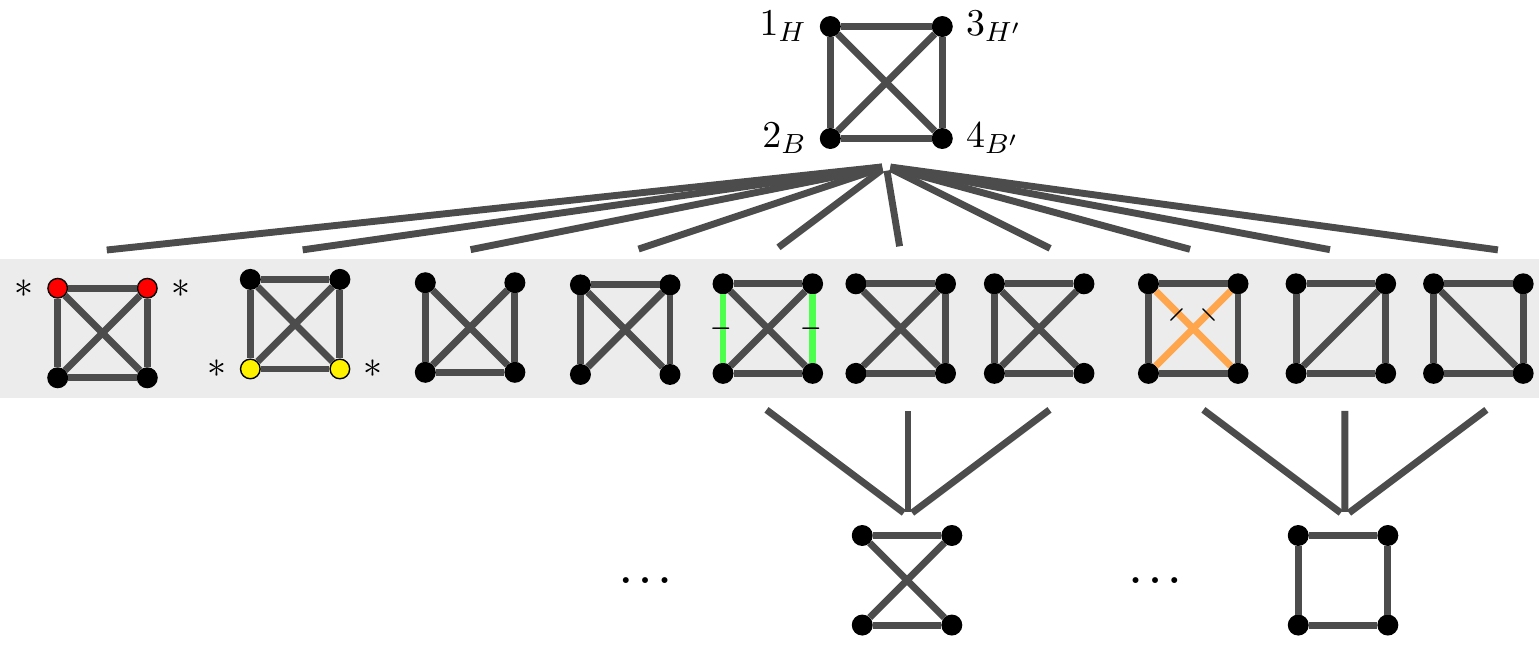}
	\end{center}
	\caption{Part of the Hasse diagram of the model inclusion lattice for the Frets' heads example. The graph representing the saturated model is depicted on the top and followed, below, by the  graphs representing its neighbouring submodels, which are in the highlighted area. Then, two graphs of the third level are given.}\label{FIG:model-inclusion-lattice}
\end{figure}
\end{example}

The usefulness of  this proposition  stands on the fact that it provides a way to construct all of the neighbouring submodels of a given \PDRCON\ model $\P(\G)$ and, furthermore, because it is stated by using the alternative representation of \PDCG{s}, it allows us to carry out a comparison between the model inclusion and the twin lattice. It is obvious, by construction, that any two neighbouring submodels of $\P(\G)$ are $\preceq_{s}$--incomparable. However, they are not typically $\preceq_{t}$--incomparable, as shown as follows.
\begin{corollary}\label{THM:two.layer}
	Let $\Hs_{1}$, $\Hs_{2}$ and $\G$ be three \PDCG{s} such that $\Hs_{1},\Hs_{2}\precdot_{s} \G$ or, equivalently, such that $\P(\Hs_{1})$ and $\P(\Hs_{2})$ are two neighbouring submodels of $\P(\G)$. Then $\Hs_{1}\preceq_{t}\Hs_{2}$ if and only if, for a given edge  $(i,j)\in\EE_{\G}$,  $\Hs_{2}$ is obtained from (ii) of  Proposition~\ref{THM:neighbor-submodels} and $\Hs_{1}$ is obtained from either (iii) or (iv) of the same proposition. More specifically, in the latter case it holds that $\Hs_{1}\precdot_{t}\Hs_{2}$ whereas, in all other cases, $\Hs_{1}$ and $\Hs_{2}$ are $\preceq_{t}$--incomparable.
\end{corollary}
\begin{proof}
	This result follows immediately from the application of the twin order in Definition~\ref{DEF:twin-order} to the comparison of all the pairs of neighbouring submodels of $\G$ as given in  Proposition~\ref{THM:neighbor-submodels}.
\end{proof}

In the graphical representation of the model inclusion lattice provided by the Hasse diagram every model is depicted above its submodels, and Proposition~\ref{THM:relation-modelinclusion-twinorder} shows that this is also true in the Hasse diagram of the twin lattice. Furthermore, in the Hasse diagram of the model inclusion lattice every model is linked by an edge to each of its neighbouring submodels, which are pairwise $\preceq_{s}$--incomparable. Theorem~\ref{THM:neighbor-submodels} and Corollary~\ref{THM:two.layer} show that
the twin order $\preceq_{t}$ can be use to partition the set of neighbouring submodels of any model $\P(\G)$ into two subsets, that we refer to as the \emph{upper} and \emph{lower-layer} neighbouring submodels of $\P(\G)$. The upper-layer neighbouring submodels are obtained by applying  to $\G$ the points (i), (ii), (v), (vi) and (vii) of Proposition~\ref{THM:neighbor-submodels}, whereas the lower-layer submodels come from points (iii) and (iv). The names upper and lower layer are suggested by the fact that every graph in the lower layer is $\preceq_{t}$--smaller than some graph in the upper layer, and therefore  in the Hasse diagram of the twin lattice the lower-layer graphs are represented below the upper-layer graphs. More specifically, within each of the two layers the models are both $\preceq_{s}$-- and  $\preceq_{t}$--incomparable, whereas every model in the lower layer is smaller, in the twin order sense, than some model in the upper layer. Thus, the twin order induces a partial ordering within the set of neighbouring submodels of any given model $\P(\G)$, which will prove useful both in the implementation of the coherence principle and in the exploration of the model space.
\begin{example}[Frets' Heads continued]
Figure~\ref{FIG:twin-lattice} gives the portion of the Hasse diagram of the twin lattice that refines the model inclusion lattice of Figure~\ref{FIG:model-inclusion-lattice} into the two-layer structure, so that every model in the lower layer is $\preceq_{t}$--smaller than some models in the upper, and within each layer,  models are pairwise both $\preceq_{s}$-- and $\preceq_{t}$--incomparable.
\begin{figure}[tbp]
	\begin{center}
		\includegraphics[scale=.8]{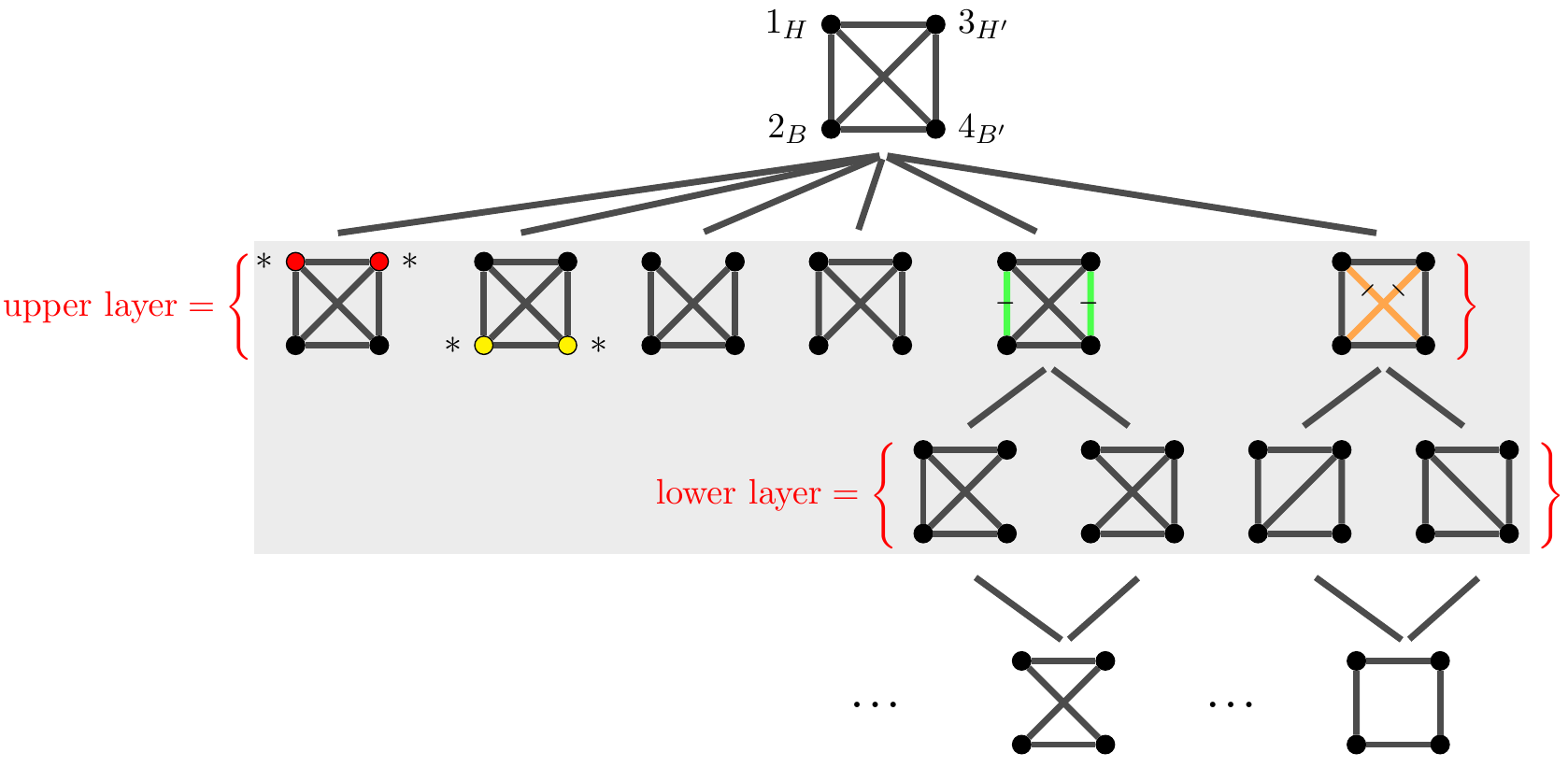}
	\end{center}
	\caption{Part of the Hasse diagram of the twin lattice for the Frets' heads example. The graph representing the saturated model is depicted on the top. The graphs in the highlighted area are those representing the model inclusion neighbouring submodels of the saturated model, which are organized in a two-layer structure, to be compared with Figure~\ref{FIG:model-inclusion-lattice}.}\label{FIG:twin-lattice}
\end{figure}

\end{example}

The stepwise procedure that we consider in the forthcoming sections involves the iterative computation of suitable subsets of neighbouring submodels obtained by meet operation within the model inclusion lattice, that is $\wedge_{s}$. The rest of this section provides some results that allow us to compute such neighbouring submodels by using the more efficient meet operation within the twin lattice, $\wedge_{t}$, which amounts to straightforward set intersection operations.

Firstly, we show that the $\wedge_{s}$ and the $\wedge_{t}$ meet operations are equivalent when applied to two neighbouring submodels  which are $\preceq_{t}$--incomparable.
\begin{corollary}\label{THM:equivalent.meet}
	Let $\Hs_{1}$, $\Hs_{2}$ and $\G$ be three \PDCG{s} such that $\Hs_{1},\Hs_{2}\precdot_{s} \G$ or, equivalently, such that $\P(\Hs_{1})$ and $\P(\Hs_{2})$ are two neighbouring submodels of $\P(\G)$. Then, if $\Hs_{1}$ and $\Hs_{2}$ are $\preceq_{t}$--incomparable it holds that
	$\Hs_{1}\wedge_{s}\Hs_{2}=\Hs_{1}\wedge_{t}\Hs_{2}$. On the other hand, if $\Hs_{1}\preceq_{t}\Hs_{2}$ then $\Hs_{1}\wedge_{s}\Hs_{2}=\Hs_{1}\wedge_{t}\Hs_{1}^{\prime}$ where, $\Hs_{1}$ and $\Hs_{1}^{\prime}$ are obtained one from (iii)  of  Proposition~\ref{THM:neighbor-submodels} and the other from (iv) of the same proposition.
\end{corollary}
\begin{proof}
	See Section~\ref{SUP.SEC:proof-equivalent.meet} of the Appendix.
\end{proof}
Finally, we provide a result that makes it possible to use the meet $\wedge_{t}$ in place of $\wedge_{s}$ at every step of the greedy search procedure described in Section~\ref{SEC:stepwise.procedure}.
\begin{corollary}\label{THM:iterative.application.meet}
	Let $\G$ be a \PDCG\ and, furthermore, let $\mathcal{A}$ be a set of \PDCG{s} such that (i) the graphs in $\mathcal{A}$ are pairwise $\preceq_{t}$--incomparable and, (ii) for every $\F\in \mathcal{A}$ it holds that $\F\precdot_{s} \G$, so that $\P(\F)$ is a neighbouring submodel of $\P(\G)$. Then, for every $\Hs\in\mathcal{A}$ the set
	\begin{align*}
		\mathcal{B}=\{\F \wedge_{s} \Hs \mid \F \in  \mathcal{A} \setminus \{\Hs\}\}=\{\F \wedge_{t} \Hs \mid \F \in  \mathcal{A} \setminus \{\Hs\}\}
	\end{align*}
	has the same properties as $\mathcal{A}$, that is, (i) the graphs in $\mathcal{B}$ are pairwise $\preceq_{t}$--incomparable and, (ii) for every $\F\in \mathcal{B}$ it holds that $\F\precdot_{s} \Hs$, so that $\P(\F)$ is a neighbouring submodel of $\P(\Hs)$.
\end{corollary}
\begin{proof}
	See Section~\ref{SUP.SEC:proof-iterative.application.meet} of the Appendix.
\end{proof}

\section{Dimension of the Search Space and Implementation of the Principle of Coherence}\label{SEC:dimension.and.choerence}


One major challenge in  learning the structure of a coloured graphical model is the dimension of the search space that is extremely large even when the number of variables, $|V|=p$, is small. The dimension of the space $M(V)$ of \GGM{s} is well-known to be $|M(V)|=2^{{p \choose 2}}$.
The dimension of the search space $\M$ of \RCON\ models was computed in \citet[][eqn. (7)]{gehrmann2011lattices}, where it is shown, for example, that for $p=5$ there are $1\,024$ undirected \GGM{s} but $35\,285\,640$ \RCON\ models. The family of \PDRCON\ models $\P$ forms a proper subset of \RCON\ models, $\P\subset\M$, however the dimension of  $\P$ is still much larger than that of \GGM{s}. It is shown in Section~\ref{SUP.SEC:search.space.dimension} of the Appendix that the dimension of $\P(V)$ can be computed as
\begin{eqnarray*}
	|\P(V)|=2^{p/2}\sum_{i=0}^{p(p-2)/4} {\frac{p(p-2)}{4}\choose i}\; 2^{{p \choose 2}-2i},
\end{eqnarray*}
so that, for example, in the application of Section~\ref{SEC:real.data} where $p=36$, the number of  \PDRCON\ models is $10^{35}$ times larger than that of \GGM{s}, formally, $|\P|> |M|\times 10^{35}$.

One way to increase the efficiency of greedy search procedures is by applying the, so-called, \emph{principle of coherence} that is used as a strategy for pruning the search space. The latter was  introduced in \citet{gabriel1969simultaneous} where it is stated that: ``in any procedure involving multiple comparisons no hypothesis should be accepted if any hypothesis implied by it is rejected". We remark that, for convenience, we say ``accepted” instead of the more correct ``non-rejected”. Consider some goodness-of-fit test for testing models at a given level $\alpha$ so that for every model in a given class we can apply the test and determine whether the model is rejected or accepted. In graphical modelling, the principle of coherence is typically implemented by requiring that we should not accept a model while rejecting a larger model; see, among others,
\citet[Chapter 6]{edwards2000introduction}, \citet{madigan1994model},  \citet[][p.~256]{cowell1999probabilistic}. Thus, if a model is rejected then also all its submodels should be rejected, and the model inclusion lattice allows a straightforward implementation of this pruning procedure because it is sufficient to remove from the Hasse diagram all the  paths descending from the rejected model. However, we note that for \PDRCON\ models this implementation of the coherence principle is not sufficient to avoid incoherent steps. This is due to the fact that, unlike the lattice of \GGM{s}, the model inclusion lattice of \PDRCON\ models is non-distributive. To clarify this issue, consider  the five graphs $\G_{i}$, $i=1,\ldots,5$ in Figure~\ref{FIG:sublattice-modelinclusion}, which depicts the relevant portion of the Hasse diagram of the model inclusion lattice. In lattice theory, the sublattice of Figure~\ref{FIG:sublattice-modelinclusion} is said to have the diamond structure, and its presence in a Hasse diagram causes the lattice to be non-distributive \citep[see][Theorem~4.10]{davey2002introduction}. The diamond structure, that is thus present in the Hasse diagram of the \PDRCON\ model inclusion lattice, but not in that of the lattice of \GGM{s}, requires additional attention in the implementation of the principle of coherence of \citet{gabriel1969simultaneous}, because it may give rise to a type of incoherence that involves $\preceq_{s}$--incomparable neighbouring submodels. To see this, consider Figure~\ref{FIG:sublattice-modelinclusion} and  assume that $\P(\G_{2})$ is rejected. In this case, under the above interpretation of the coherence principle also $\P(\G_{5})\subseteq \P(\G_{2})$ should be considered rejected. On the other hand, nothing can be said with respect to $\P(\G_{3})$ and $\P(\G_{4})$ because  $\P(\G_{2})$, $\P(\G_{3})$ and $\P(\G_{4})$  are neighbouring submodels of $\P(\G_{1})$ and are thus model-inclusion incomparable. It is therefore possible that $\P(\G_{3})$ and $\P(\G_{4})$ are both accepted whereas  $\P(\G_{2})$, and thus  $\P(\G_{5})$, are rejected. However, this is clearly against the coherence principle because  $\G_{5}=\G_{3}\wedge_{s}\G_{4}$. More generally, it holds that
\begin{align*}
	\G_{5}=\G_{2}\wedge_{s}\G_{3}=\G_{3}\wedge_{s}\G_{4}=\G_{2}\wedge_{s}\G_{4},
\end{align*}
and therefore it would be incoherent to reject any of the models  $\P(\G_{2})$, $\P(\G_{3})$ or $\P(\G_{4})$ while accepting the remaining two. Consider now the Hasse diagram of the twin lattice for the same models, given in  Figure~\ref{FIG:sublattice-twinorder}. The three neighbouring submodels  of $\P(\G_{1})$ are now partitioned into a two-layer structure that can be exploited to correctly implement the coherence principle. Following the structure of the twin lattice, model $\P(\G_{2})$ is tested first and, if it is accepted, then models  $\P(\G_{3})$ and $\P(\G_{4})$ should be either both accepted or rejected and this hypothesis can be verified directly from $\P(\G_{5})$. On the other hand, if $\P(\G_{2})$ is rejected we can consider  $\P(\G_{3})$ and $\P(\G_{4})$ by recalling that it would be incoherent to accept both.

\begin{figure}[t]
	\centering
	\begin{subfigure}{0.4\textwidth}
		\centering
		\includegraphics[scale=0.65]{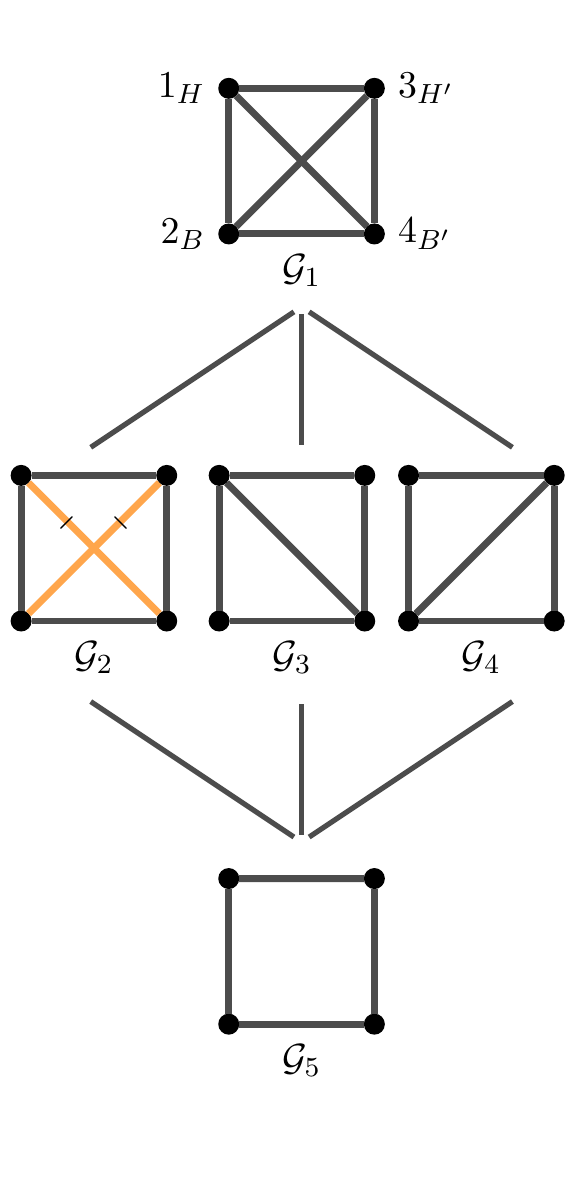}
		\caption{}\label{FIG:sublattice-modelinclusion}
	\end{subfigure}
	\begin{subfigure}{0.4\textwidth}
		\centering
		\includegraphics[scale=0.65]{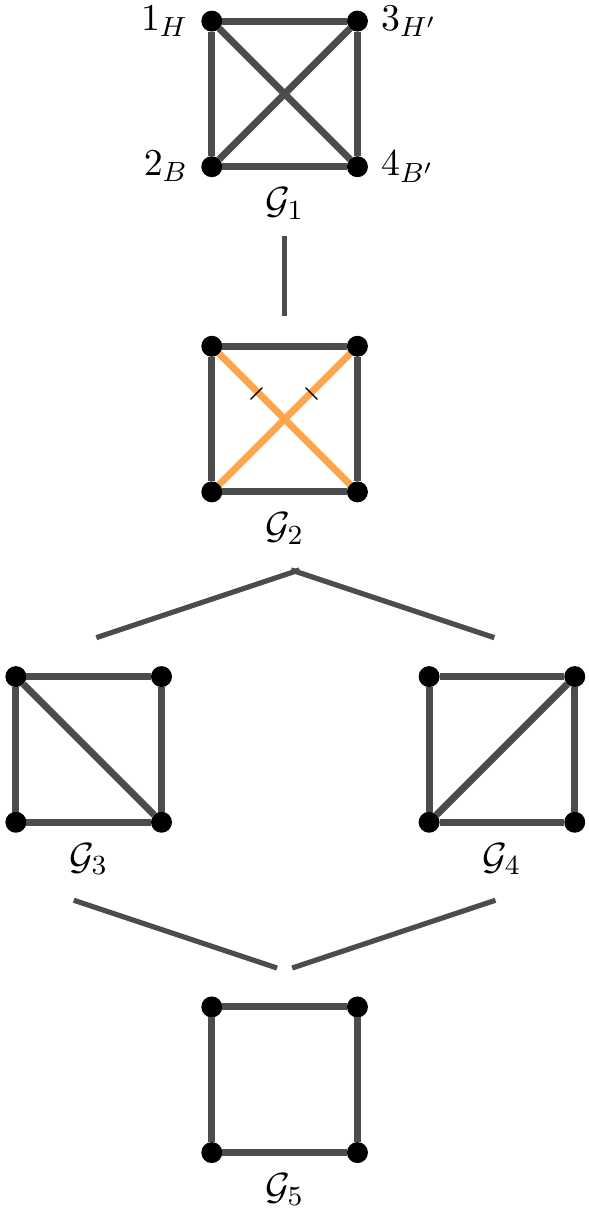}
		\caption{}\label{FIG:sublattice-twinorder}
	\end{subfigure}
	\caption{Frets' Heads example: Comparison of the Hasse diagrams of two sublattices induced by the same five \PDCG{s} under (a) the model inclusion order, and (b) the twin order.}
\end{figure}

It is not difficult to use the result of Proposition~\ref{THM:neighbor-submodels} to generalise this idea to any set of neighbouring submodels.
It follows that, in the implementation of the principle of coherence by using the twin lattice, the general rule is the same as in the model inclusion lattice: if a model is rejected then all the models in the paths descending from it should be excluded from the set of candidate models. In addition, there are specific rules for upper-layer models. If an upper-layer model is rejected then all the lower-layer models directly linked to it cannot be excluded from the candidate models. On the other hand, if it is accepted, one can exclude all the lower-layer models directly linked to it, and this can greatly reduce the number of candidate models, as shown by the simulations in Section~\ref{SEC:simulations}.

Finally, it is worth remarking that the organization of the neighbouring submodels into the two-layer structure, provided by the twin lattice, is useful in the implementation of the coherence principle, but it has also the advantage of eliminating the diamond structure, thereby making the lattice distributive.


\section{A Coherent Greedy Search Procedure} \label{SEC:stepwise.procedure}


We introduce a stepwise backward elimination procedure that exploits the twin lattice both for the computation of the meet operation and the implementation of the coherence principle.

Every step of the procedure starts from a model, defined by a graph labelled as $\G^{\best}$, and computes a set of candidate neighbouring submodels so as to obtain a set $\mathcal{A}$ of accepted models, according to a pre-established criterion. Specifically, we label as accepted the models with $p$-value of the likelihood ratio test against the saturated model, computed on the asymptotic chi-squared distribution, larger than $0.05$. Then a new $\P(\G^{\best})$ is selected from $\mathcal{A}$. In this implementation, we choose as best model in $\mathcal{A}$ the model with the largest $p$-value. We set the saturated model as a starting point and then the procedure is iterated until either $\mathcal{A}$ is empty or a maximum number of iterations is reached. The pseudocode of the procedure is given in Algorithms~\ref{ALG:main-short} and \ref{ALG:next.step-short} where, in order to make the code more readable, we have kept the technical level low. The pseudocode with all the technical details can be found in Algorithms~\ref{ALG:main-detailed} and \ref{ALG:next.step-detailed}   of the Appendix.

\begin{algorithm}
	\caption{Coherent stepwise backward elimination procedure}\label{ALG:main-short}
	\begin{algorithmic}[1]
		\State $\G^{\best}\gets$ complete graph with all colour classes atomic
		\State $\mathcal{A}\gets \emptyset$
		\State $K\gets$ maximum number of steps
		\State $k\gets 1$
		\ForAll{upper-layer neighbouring submodels $\Hs$ of $\G^{\best}$}
		\lIf{$\Hs$ is accepted} $\mathcal{A}\gets \mathcal{A}\cup \{\Hs\}$
		\EndFor
		\ForAll{lower-layer neighbouring submodels $\Hs$ of $\G^{\best}$ which are $\preceq_{t}$--incomparable with\\
			\phantom{\textbf{ for all}}every graph in $\mathcal{A}$}
		\lIf{$\Hs$ is accepted}
		$\mathcal{A}\gets \mathcal{A}\cup \{\Hs\}$
		\EndFor
		\While{$\mathcal A\neq\emptyset$ and $k<K$}
		\State $\mathcal{A},\;\G^{\best}\gets\;$\Call{Update}{$\mathcal{A},\; \G^{\best}$}  \Comment{see Algorithm~\ref{ALG:next.step-short}}
		\State $k\gets k+1$
		\EndWhile
		\lIf{$\mathcal A\neq\emptyset$ and $k=K$}
		$\G^{\best}\gets$ best model in $\mathcal{A}$
		\State \textbf{return} $\G^{\best}$
	\end{algorithmic}
\end{algorithm}

A key issue concerns the identification, at every step, of the candidate models, which are all the coherent neighbouring submodels of $\G^{\best}$. At the first step, efficiency is achieved by considering the upper layer first, and then the lower layer, so as to apply the principle of coherence as described in Section~\ref{SEC:dimension.and.choerence} and implemented in lines 5--11  of Algorithm~\ref{ALG:main-short}. This can significantly reduce the dimension of the initial set of candidate models and, in turn, the dimension of the sets of candidate models of all the subsequent steps. In addition, the implementation of the coherence principle from the first step implies that the models in $\mathcal{A}$ are pairwise $\preceq_{t}$--incompatible and therefore, by Corollary~\ref{THM:equivalent.meet}, the next set of candidate models can be computed by using the more efficient $\wedge_{t}$ meet operation. Furthermore, Corollary~\ref{THM:iterative.application.meet} guarantees that the same can be done at every step of the procedure; see lines 14--17 of Algorithm~\ref{ALG:next.step-short}.

\begin{algorithm}
	\caption{\textsc{Update()} procedure called by Algorithm~\ref{ALG:main-short}} \label{ALG:next.step-short}
	\begin{algorithmic}[1]
		\Procedure{Update}{$\mathcal{A},\; \G^{\best}$}
		\State $\G^{\old}\gets \G^{\best}$
		\State $\G^{\best}\gets$ best model in $\mathcal{A}$
		\If{$\G^{\best}$ is obtained from $\G^{\old}$ by removing exactly one edge $(i,j)$ with $i\neq\tau(j)$}
		\State $\Hs\gets$  graph obtained by removing from $\G^{\old}$ the edge $\tau(i,j)$
		\State $\mathcal{A}\gets \mathcal{A}\setminus \{\Hs\}$
		\ElsIf{$\G^{\best}$ is obtained from $\G^{\old}$ by merging two atomic classes to obtain the\\
			\phantom{\textbf{proelse if}}twin-pairing class $\{(i,j), \tau(i,j)\}$}
		\State $\Hs\gets$ graph obtained by removing from $\G^{\old}$ the edges $(i,j)$ and $\tau(i,j)$
		\State $\mathcal{A}\gets \mathcal{A}\cup \{\Hs\}$
		\EndIf
		\State $\mathcal{A^{\old}}\gets \mathcal{A}\setminus \{\G^{\best}\}$
		\State $\mathcal{A}\gets\emptyset$
		\ForAll{$\F\in\mathcal{A}^{\old}$}
		\State $\Hs\gets \F\wedge_{t}\G^{\best}$
		\lIf{$\Hs$ is accepted} $\mathcal{A}\gets \mathcal{A}\cup \{\Hs\}$
		\EndFor
		\State \textbf{return} $\mathcal{A}$ and $\G^{\best}$
		\EndProcedure
	\end{algorithmic}
\end{algorithm}
%


\section{Applications}\label{SEC:applications}
The greedy search procedure of the previous section has been implemented in the program language \textsf{R}, and here we describe its application to both synthetic and real-world data, including an empirical comparison with the penalized likelihood method of \citet{ranciati2023application}.

\subsection{Simulations}\label{SEC:simulations}
In this section, we analyse the behaviour of the search procedure of Section~\ref{SEC:stepwise.procedure} on synthetic data. More specifically, we
compare it with the stepwise backward elimination procedure given in \citet{roverato2022modelinclusion} that does not exploit the twin lattice for the computation of the set of candidate models, and where the principle of coherence is naively implemented by only considering the submodel relationship.

We considered two scenarios that differ for the sparsity degree, computed as $|E|/|F_{V}|$. For each of the two scenarios we generated four \PDCG{s} with $p=8,12,16,20$ and density degrees approximatively equal to $0.18$ for scenario $A$ and to $0.35$ for scenario $B$. Next, for every \PDCG\ $\G$ we randomly generated a concentration matrix $\Theta$ such that the normal distribution $N(0, \Theta^{-1})\in \P(\G)$ and,  finally, we randomly selected $20$ samples of size $100$ from such normal distribution; see Section~\ref{SUP.SEC:simulations} of the Appendix for details.

For each of the 160 synthetic data sets generated, we ran the two greedy search procedures, which always terminated before the maximum number of iterations was reached. The performance of the two procedures is summarized in Table~\ref{TAB:measure-simmulation} of the Appendix.
The point of main interest is the comparison in terms of efficiency, that we quantify with respect to the average execution time and the average number of fitted models. These can be found in the last two columns of Table~\ref{TAB:measure-simmulation} and, furthermore, the growth rates of these measurements are displayed in Figures~\ref{FIG:computation-time} and \ref{FIG:numberOfModels}.
We can see that the procedure on the twin lattice is considerably more efficient. Specifically, the procedure on the twin lattice was more than five times faster, requiring between  $16\%$ and $20\%$ of the time required by  the procedure on the model inclusion lattice. Furthermore, the proper implementation of the coherence principle allowed us to fit a much smaller number of models, ranging between $37\%$ and $54\%$ of the models fitted under the naive implementation of  principle of coherence. It is also interesting to notice that the latter proportions appear to decrease as $p$ increases. Table~\ref{TAB:measure-simmulation}  also gives the average values over the 20 samples of the positive predicted value, the true-positive rate and the true-negative rate, both for the edges and for the colour classes of the selected graphs. These have satisfying values and there are not relevant differences between the two procedures, thereby showing that the increase in efficiency is not achieved at the cost of a lower level of performance of the selected model.
\begin{figure}[t]
	\centering
	\begin{subfigure}{0.47\textwidth}
		\centering
		\includegraphics[scale=0.4]{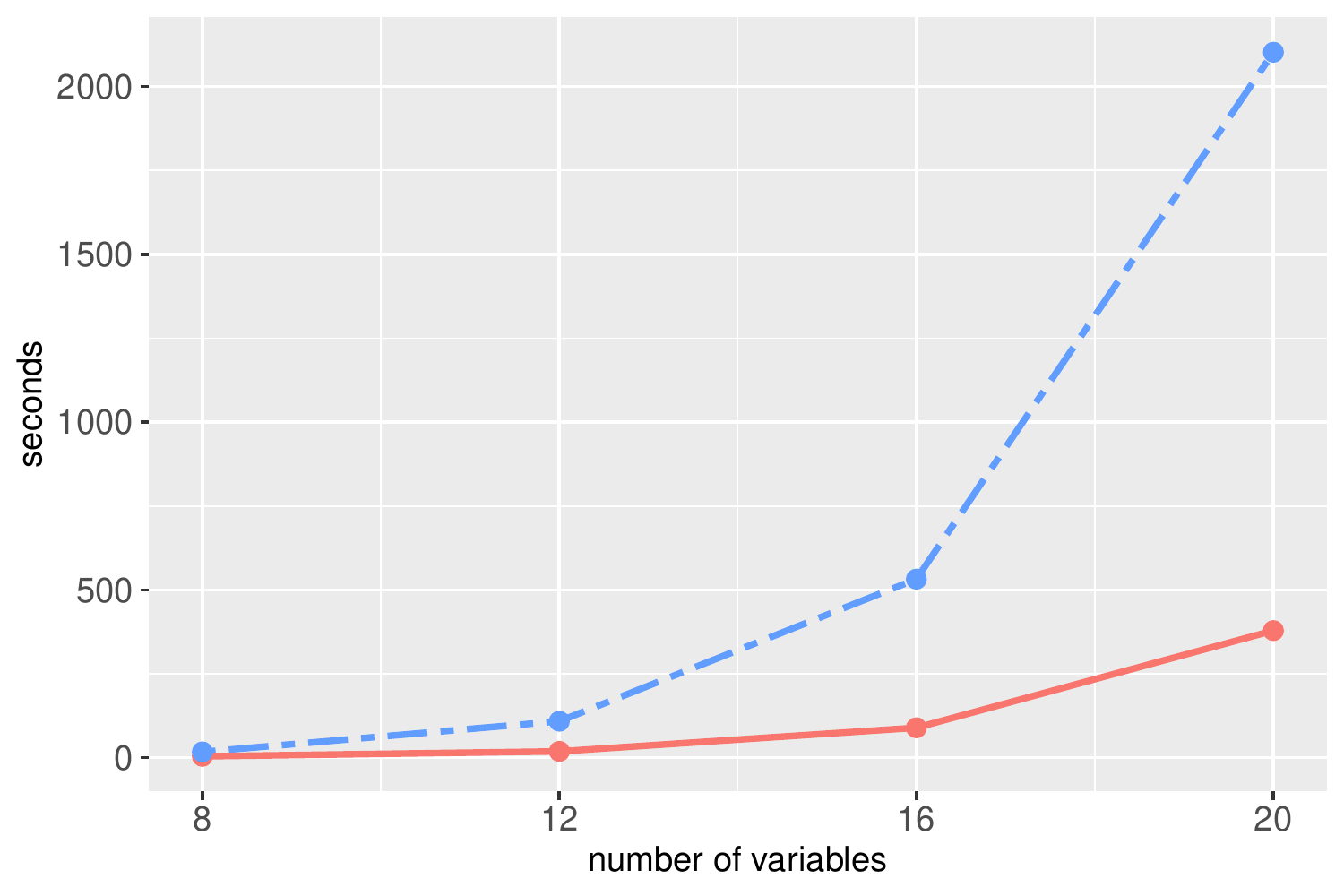}
	\end{subfigure}
	\begin{subfigure}{0.45\textwidth}
		\centering
		\includegraphics[scale=0.4]{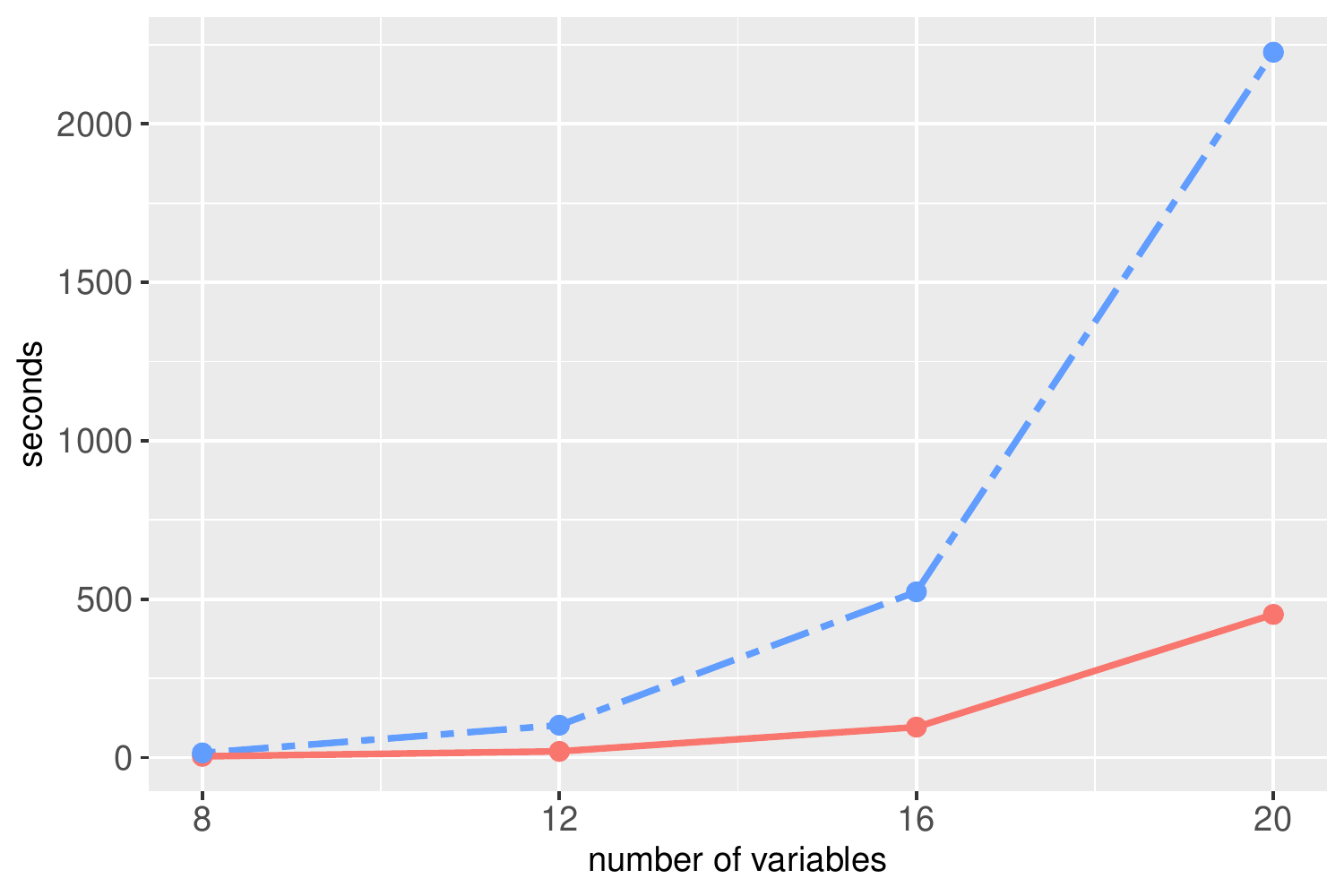}
	\end{subfigure}
	\caption{Average computational time in the simulations for the procedures based on the twin lattice (full red line) and on the model inclusion lattice (dashed blue line), for scenario A (left panel) and B (right panel).}\label{FIG:computation-time}
\end{figure}

%
\begin{figure}[b]
	\centering
	\begin{subfigure}{0.47\textwidth}
		\centering
		\includegraphics[scale=0.4]{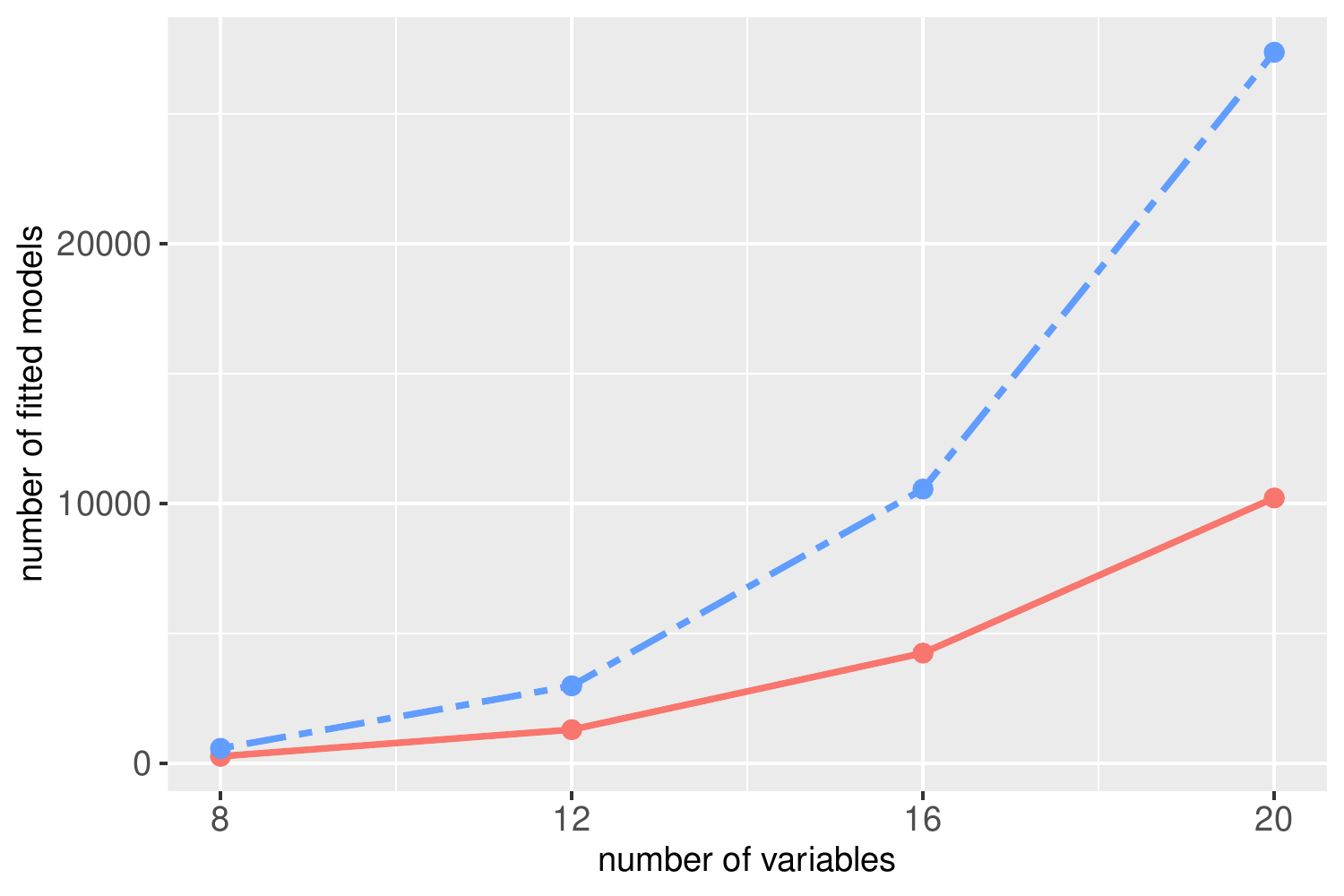}
	\end{subfigure}
	\begin{subfigure}{0.45\textwidth}
		\centering
		\includegraphics[scale=0.4]{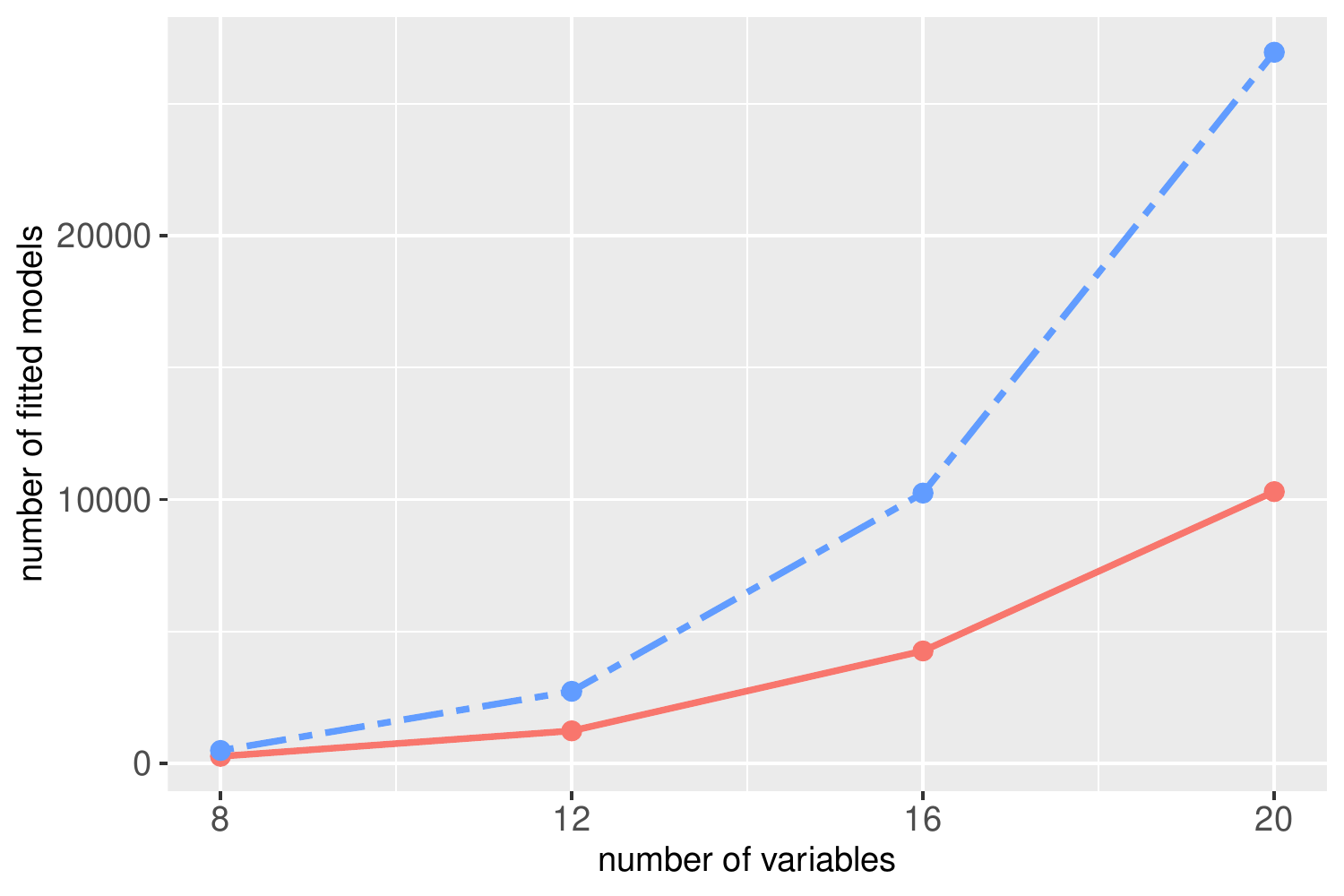}
	\end{subfigure}
	\caption{Averaged numbers of fitted models in the simulations for the procedures based on the twin lattice (full red line) and on the model inclusion lattice (dashed blue line), for scenario A (left panel) and B (right panel).}\label{FIG:numberOfModels}
\end{figure}
%


\subsection{Brain Networks from fMRI Data}\label{SEC:real.data}


Functional MRI is a non invasive technique for collecting data on brain activity that measures the increase in the oxygenation level at some specific brain region, as long as an increase in blood flow occurs, due to some brain activity. The construction of a network from fMRI data requires first the identification of a set of functional vertices, such as spatial regions of interest (ROIs), and then the analysis of connectivity patterns across ROIs. The data set we use for this application comes from a pilot study of the Enhanced Nathan Kline Institute-Rockland Sample project that are time series  recorded on $70$ ROIs at $404$ equally spaced time points. A detailed description of the project, scopes, and technical aspects can be found at \url{http://fcon_1000.projects.nitrc.org/indi/enhanced/}. Following \cite{ranciati2021fused} we apply our method to the residuals estimated from the vector autoregression models, carried out to remove the temporal dependence. We consider two subjects indexed by $14$ and $15$, who have the same psychological traits with no neuropsychiatric diseases and right-handedness. The main difference is that subject $14$ is  $19$ years old whereas the subject $15$ is $57$ years old.
The human brain has a natural symmetric structure so that for every spatial ROI on the left hemisphere there is an homologous ROI on the right hemisphere. Accordingly, we identify the left and right groups with the left and the right hemispheres, respectively, and consider  $36$ cortical brain regions, that are $22$ regions in the frontal lobe and $14$ regions in the anterior temporal lobe. We have therefore $|V|=36$ with $|L|=|R|=18$.

This section describes the analysis for the subject $15$ whereas the analysis of subject $14$ is given in Section~\ref{SEC:Supp-fMRI} of the Appendix. We applied the procedure of Section~\ref{SEC:stepwise.procedure} and selected the model defined by the \PDCG\ given in Figure~\ref{FIG:36regionSub15}. The graph of the selected model, denoted by $\G$, has density equal to $52.2\%$, the number of edges is 329 and there are $5$ vertex and $85$ edge twin-pairing colour classes, respectively. It provides an adequate fit with a $p$-value$=0.064$, computed on the asymptotic chi-squared distribution on  275 degrees of freedom of the likelihood ratio test for the comparison with the saturated model.
Note that Figure~\ref{FIG:36regionSub15} splits the selected graph into 4 panels. The top-left panel gives the edges of $\G$ that belong to $E_{L}$ and form atomic colour classes, and similarly for the top-right panel that gives the atomic colour classes in $E_{R}$. Furthermore, the bottom-left and right panels give the edge twin-pairing colour classes between and across groups, respectively. We use the gray colour for both vertex and edge twin-pairing classes and, finally, we have omitted to represent the edges in $E_{T}$ because this set is almost complete, in the sense that $E_{T}=F_{T}\setminus\{(L4, R4)\}$. We deem that this representation can effectively illustrate the main features of the model, for example highlighting a high degree of symmetry in $\G$, given that $51.67\%$ of the edges of $\G$ belong to twin-pairing classes. The selected model has $275$ parameters and it is considerably more parsimonious than the saturated model that has  $666$ parameters. It is also interesting to notice that the \GGM\ defined by the undirected version of $\G$  has $90$ parameters more than the selected \PDRCON\ model.
\begin{figure}[tb]
	\centering
	\begin{subfigure}{0.4\textwidth}
		\centering
		\includegraphics[scale=0.45]{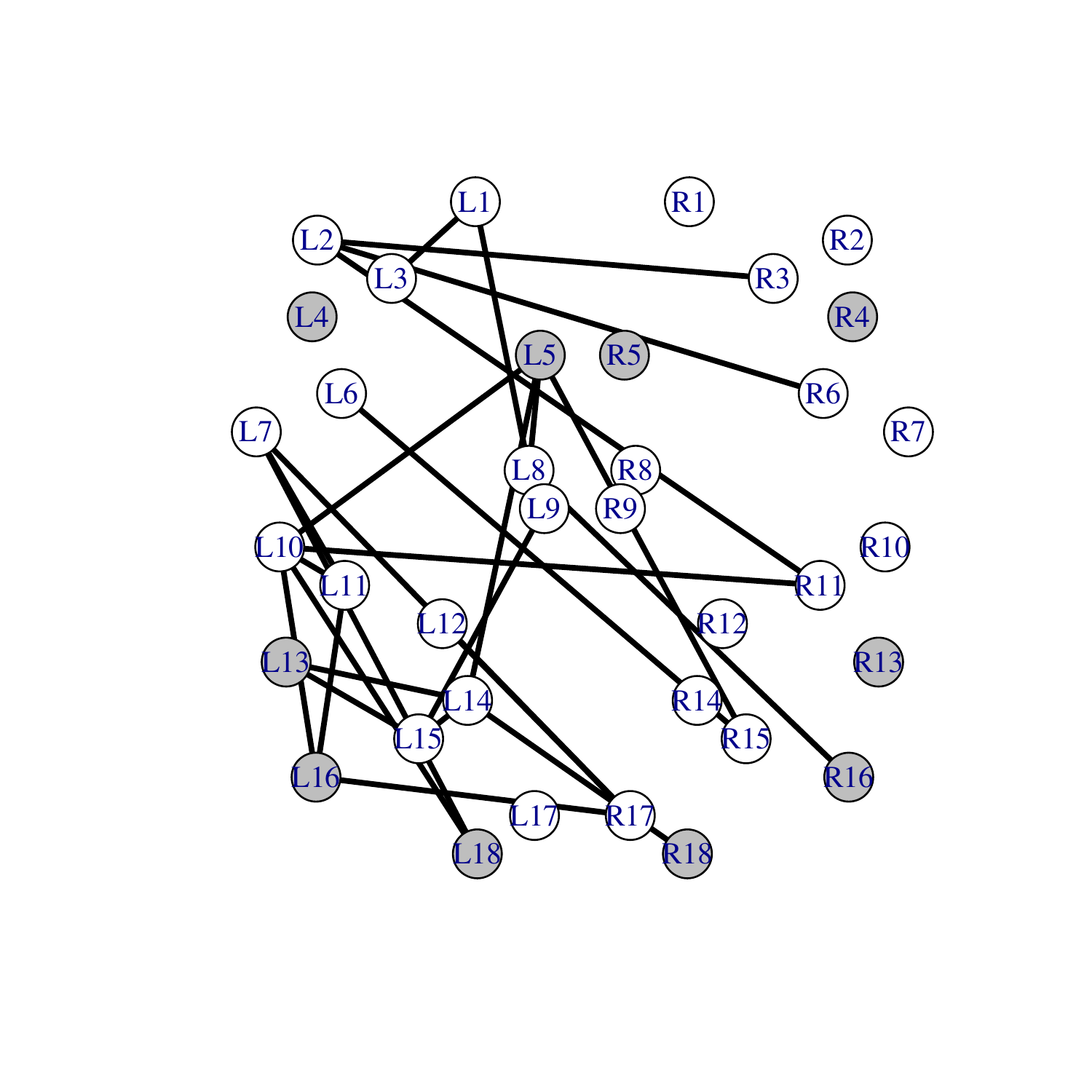}
		\caption{}
	\end{subfigure}
	\begin{subfigure}{0.4\textwidth}
		\centering
		\includegraphics[scale=0.45]{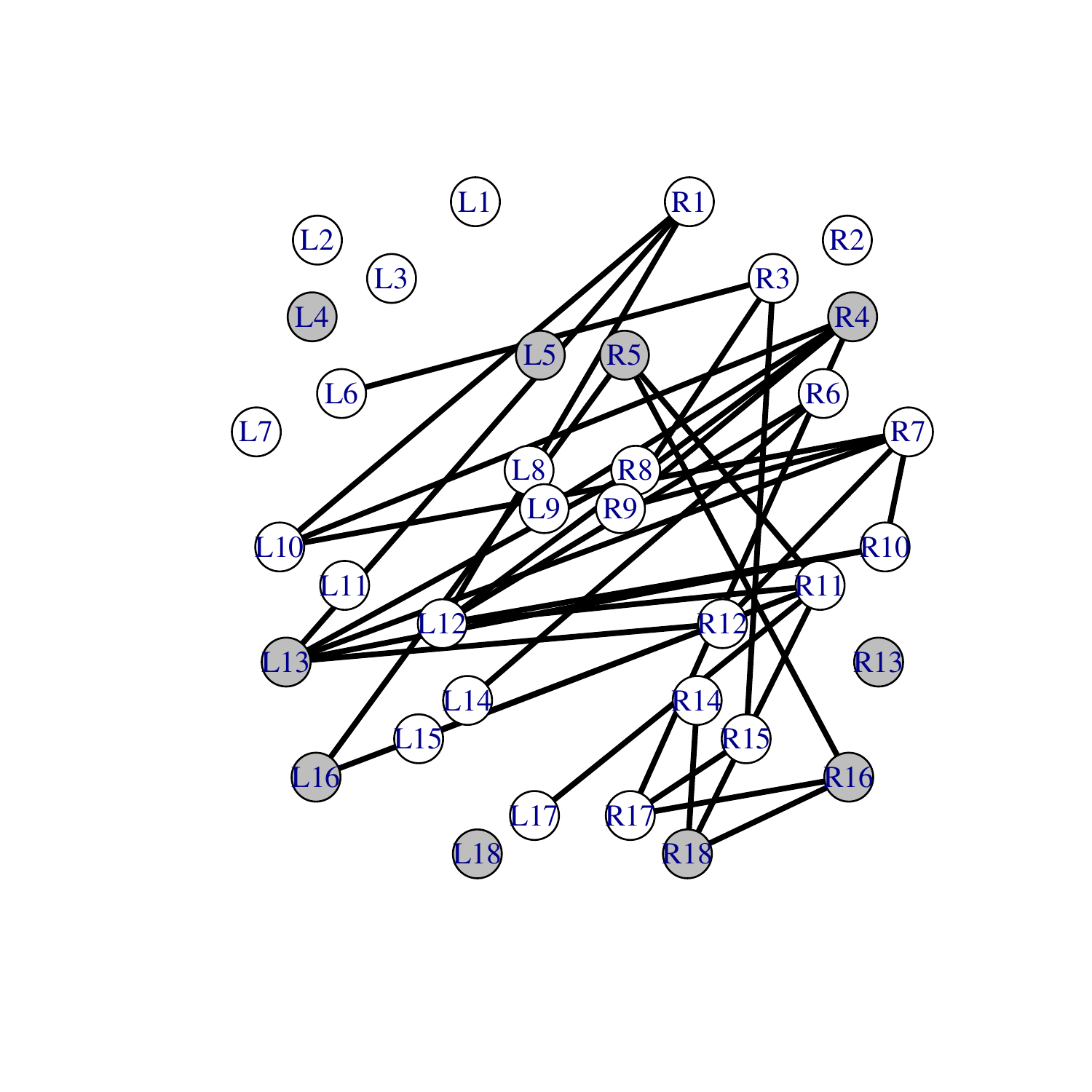}
		\caption{}
	\end{subfigure}

	\begin{subfigure}{0.4\textwidth}
		\centering
		\includegraphics[scale=0.45]{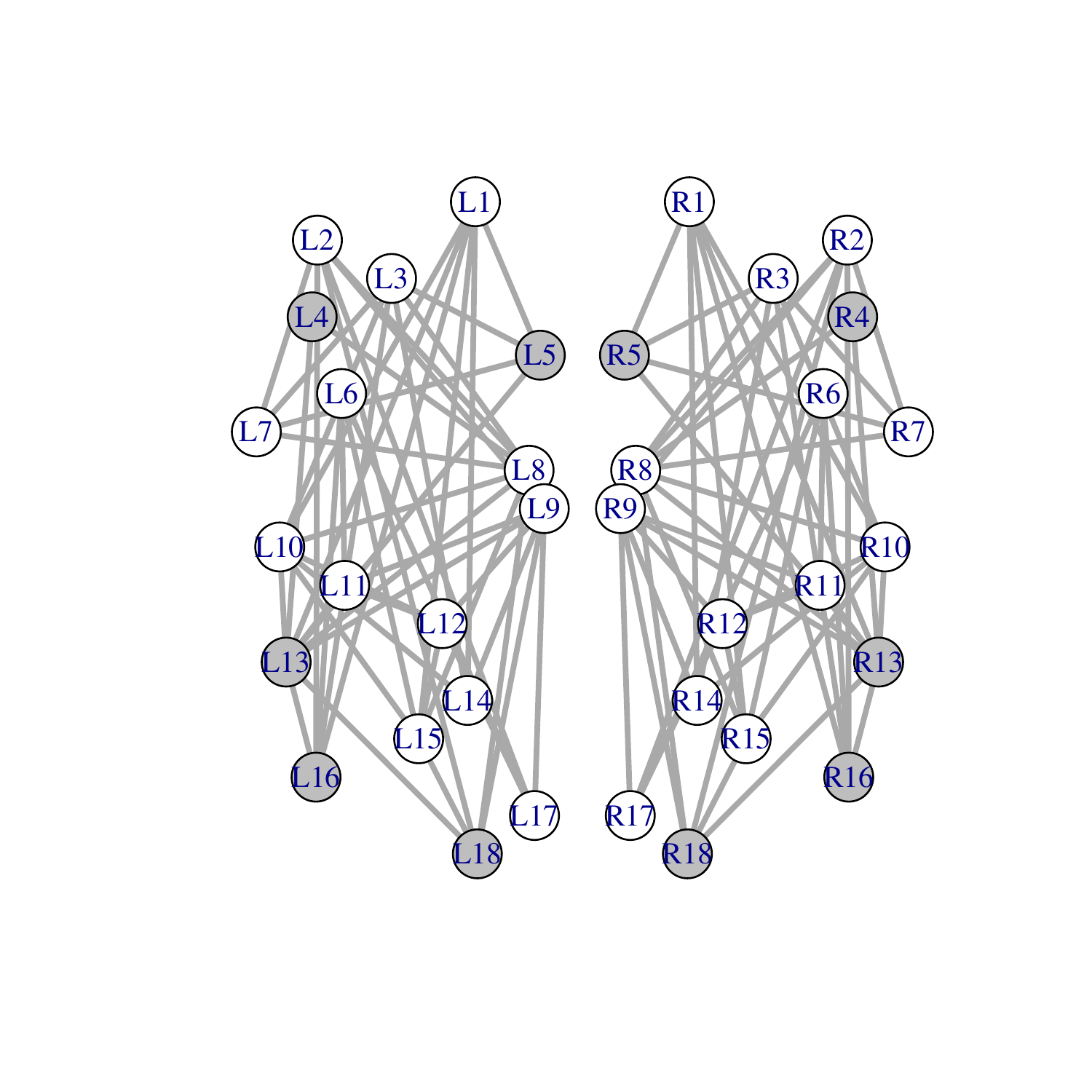}
		\caption{}
	\end{subfigure}
	\begin{subfigure}{0.4\textwidth}
		\centering
		\includegraphics[scale=0.45]{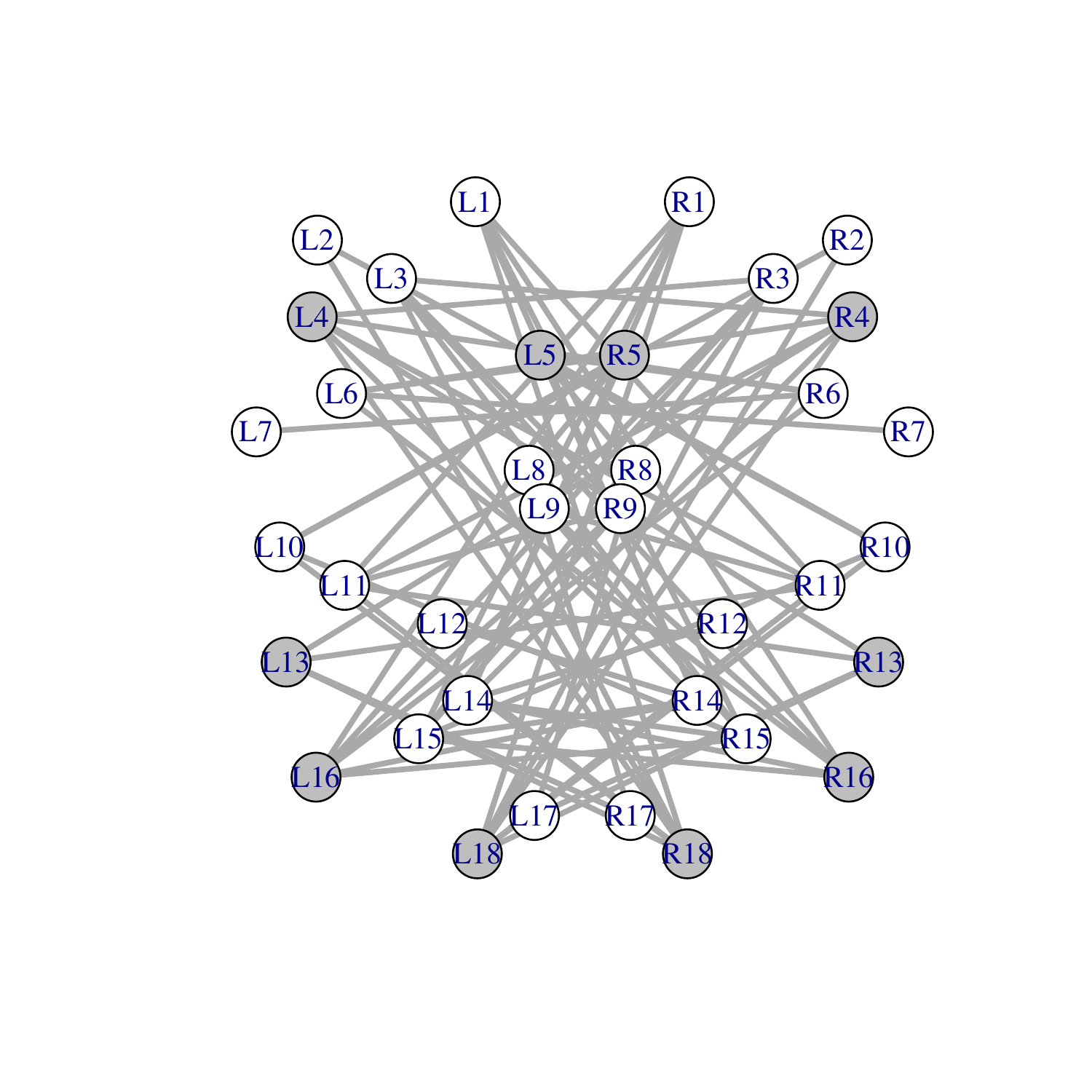}
		\caption{}
	\end{subfigure}
	\caption{Selected \PDCG\ for subject 15 with separate panels for: (a) edges in $E_{L}$ forming atomic classes; (b)  edges in $E_{R}$ forming atomic classes; (c) twin-pairing classes between groups; (d) twin-pairing classes across groups. Twin-pairing classes are depicted in gray, also for vertices, and the edges in $E_{T}$ are not depicted.} \label{FIG:36regionSub15}
\end{figure}

\subsection{Comparison with Penalized Likelihood Methods}\label{SEC:air.quality}
We now carry out a comparison of our greedy search procedure with the graphical lasso for paired data (pdglasso) method of \cite{ranciati2023application}, with special attention to the role played by the scale of the variables. The methods are applied to an Air Quality dataset containing average hourly measurements from a multi-sensor gas device for one year \citep{de2008field}. The device was located in a heavily polluted area of an Italian city, at road level. Additional details, including the dataset, can be found at \url{https://archive.ics.uci.edu/ml/datasets/Air+Quality}.

We consider 6 variables, relative to the 4 substances $CO$, $C6H6$, $NO2$, $O3$, and the 2 meteorological measurements $RH$ (relative humidity) and $AH$ (absolute humidity). In order to analyse the different behaviour during the night and day hours, for every day the measure at 1am is matched to that at 1pm, so that $|V|=12$ with $|L|=|R|=6$. The dataset obtained  after removing missing values is made up of 373 observations, and we model the residual structure of a lag-1 autoregression model \citep{epskamp2018gaussian}.

These data present a considerable difference in the scale of the variables, with a much larger scale for variables  $CO$, $C6H6$, $NO2$ and $O3$ compared to the scale of $RH$ and $AH$. As  explained in Section~\ref{SEC:air.quality}, in this case the application of the pdglasso is problematic and, in order to clarify this issue, we apply the method to both the unscaled and the standardized data. Models are chosen on the basis of the Bayesian information criterion (BIC).  Figure~\ref{airqualityB} gives the model obtained from the application of the pdglasso to the sample covariance matrix, and it is evident that the result strongly depends on the scale of the variables with the variables $RH$ and $AH$ independent of both each other and all the remaining variables. This seems unrealistic if one notes, for example, that the sample correlation of $RH$ and $AH$ is equal to $0.66$ at 1am, and to $0.61$ at 1pm. Indeed, the model selected from the sample correlation matrix, in Figure~\ref{airqualityC}, has a denser structure with the variables $RH$ and $AH$  connected both between them and with some of the remaining variables. This model has 46 parameters and, compared with the saturated model, its $p$-value is almost equal to 0. Furthermore, no twin-pairing class is identified, but it is not clear how to interpret this result if one recalls that  standardization may affect the structure of twin-pairing classes. Interestingly, the model in Figure~\ref{airqualityA}, selected by the greedy search procedure, has a fully symmetric structure, thereby suggesting that there is no difference between the association structure in the night and day hours. The selected graph looks denser than that in Figure~\ref{airqualityC}, nevertheless the model of Figure~\ref{airqualityA} is more parsimonious with 36 parameters, and it shows an adequate fit, with $p$-value equal to $0.39$.
\begin{figure}[t]
	\centering
	\begin{subfigure}{0.3\textwidth}
		\centering
		\includegraphics[scale=0.35]{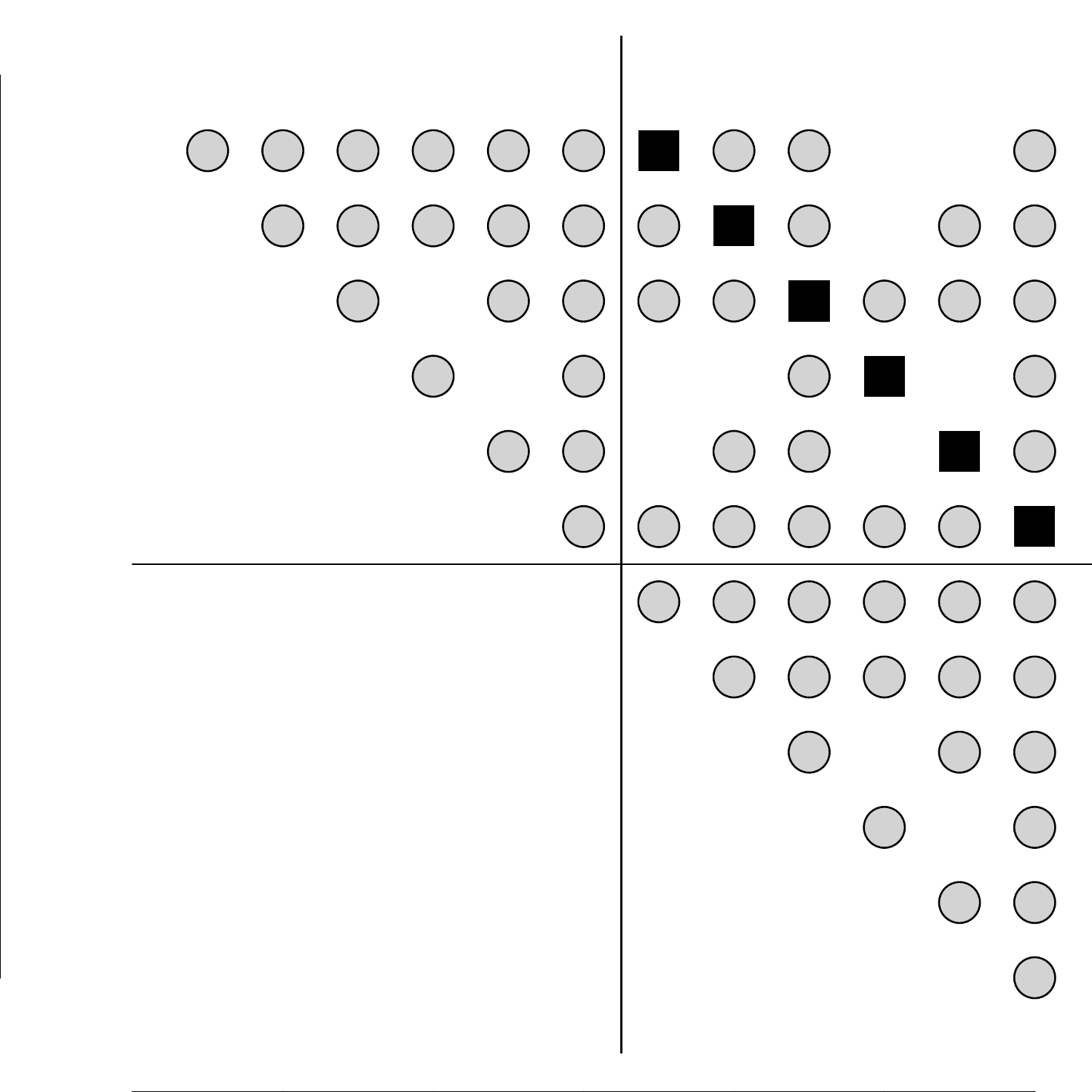}
        \caption{}\label{airqualityA}
	\end{subfigure}
	\begin{subfigure}{0.35\textwidth}
		\centering
		\includegraphics[scale=0.35]{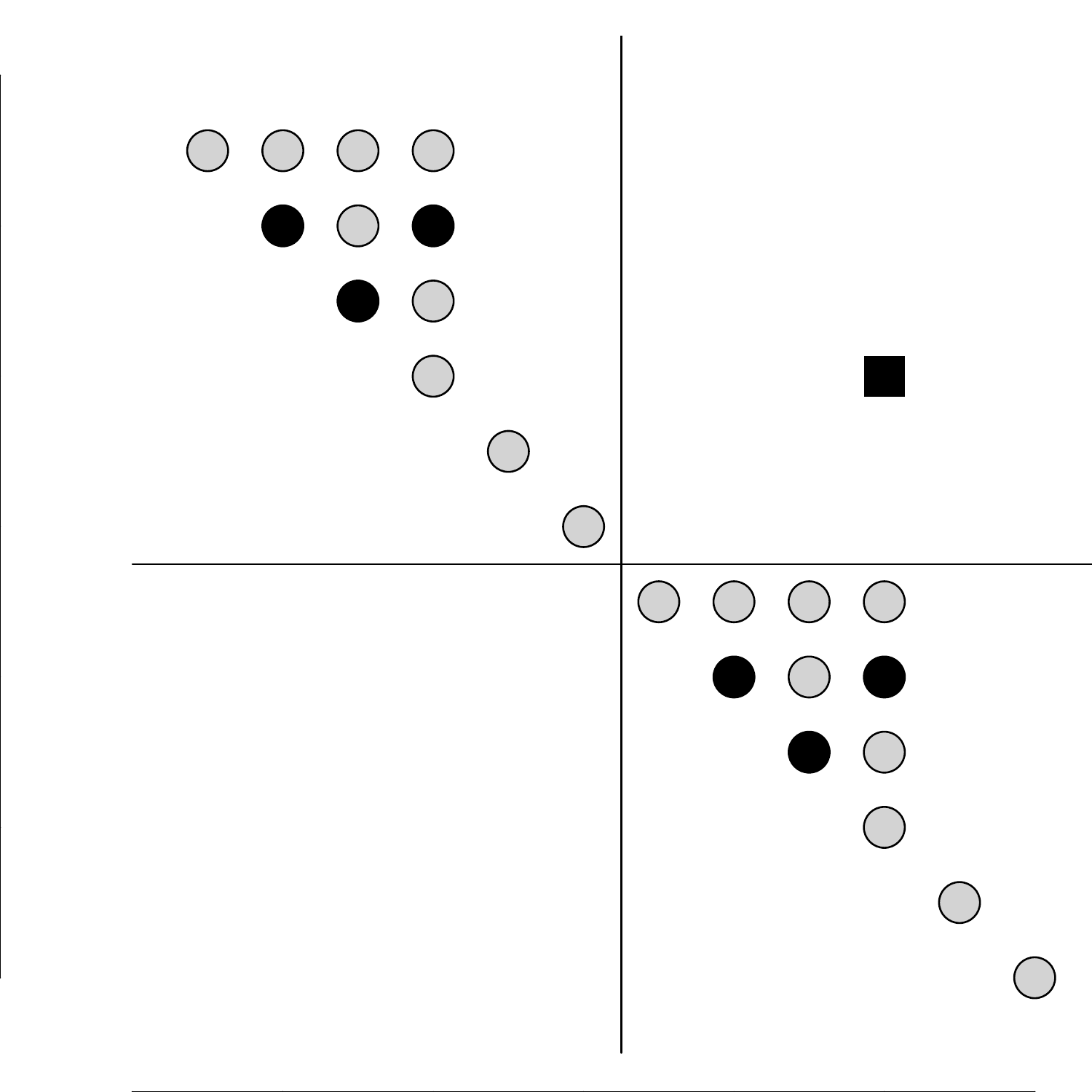}
        \caption{}\label{airqualityB}
	\end{subfigure}
	\begin{subfigure}{0.3\textwidth}
		\centering
		\includegraphics[scale=0.35]{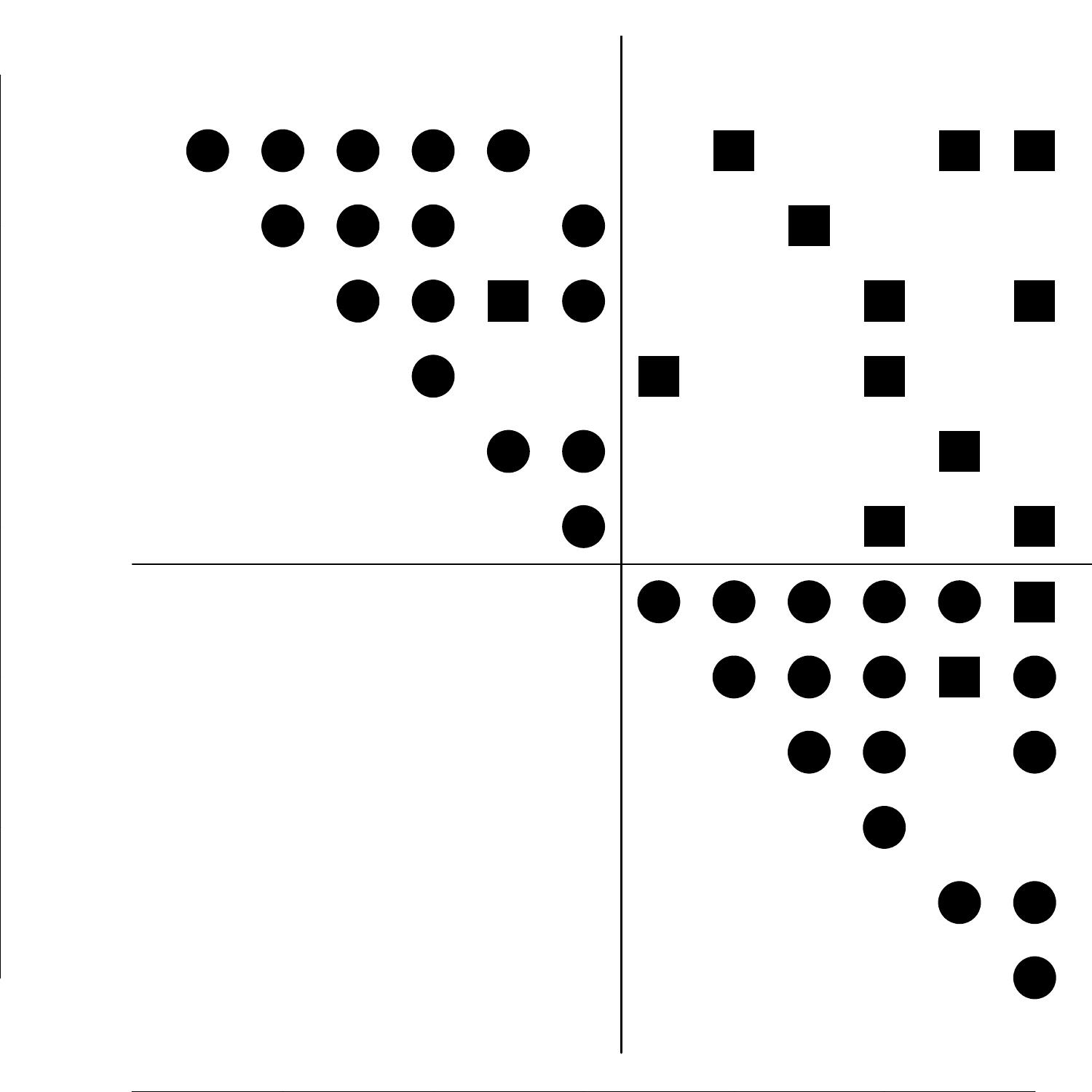}
        \caption{}\label{airqualityC}
	\end{subfigure}
	\caption{Matrix representation of the selected \PDRCON\ models for the Air Quality data by: the greedy search procedure (a), the paired data graphical lasso procedure applied to unscaled (b) and standardized (c) data. The left and right blocks correspond to the 1am and 1pm measurements, respectively, with variable ordering $CO$, $C6H6$, $NO2$, $O3$,  $RH$ and $AH$. The twin-pairing classes are denoted by gray dots (\textcolor{lightgray}{\small\faCircle}), whereas black symbols are used for atomic classes, with black dots (\textcolor{black}{\small\faCircle}) in the case where homologous vertices/edges are both present and black squares ({\small $\blacksquare$}) for edges whose twin edge is not present.}\label{airquality}
\end{figure}


\section{Conclusions}\label{SEC:conclusions}


We have considered the problem of structure learning of \GGM{s} for paired data by focusing on the family of \RCON\ models defined by coloured graphs named \PDCG{s}. The main results of this paper provide insight into the structure of the model inclusion lattice of \PDCG{s}. We have introduced an alternative representation of these graphs that facilitates the computation of neighbouring models. Furthermore, this alternative representation is naturally associated with a novel order relationship that has led to the construction of the twin lattice, whose structure resembles that of the well-known set inclusion lattice, and that facilitates the exploration of the search space. These results can be applied in the implementation of both greedy and Bayesian model search procedures. Here, we have shown how they can be used to improve the efficiency of stepwise backward elimination procedures. This has also made it clear that the use of the twin lattice facilitates the correct application of the principle of coherence. Finally, we have applied our procedure to learn a brain network on 36 variables. This model dimension could be regarded as somehow small, compared with the number of variables that can be dealt with by penalized likelihood methods. This is due to the fact that, as shown in Section~\ref{SEC:dimension.and.choerence}, the number of \PDRCON\ models is much larger than that of \GGM{s} and the same is the number of neighbouring submodels that need to be identified at every step of the algorithm. Furthermore, for every model considered, the computation of the maximum likelihood estimate is not available in closed form, but it involves an iterative procedure. Efficiency improvement is object of current research and could be achieved, for instance, by both implementing a procedure that deals with candidate submodels in parallel, and a procedure for the computation of maximum likelihood estimates explicitly designed for \PDRCON\ models.
We recall, however, that, as explained in Sections~\ref{SEC:related.works} and \ref{SEC:air.quality}, although penalized likelihood methods are considerably more efficient, their use is problematic when variables are not measured on comparable scales. Finally, we also remark that the range of application of our results does not restrict to  \PDRCON\ models. In fact, the colouring of vertices and edges of \PDCG{s} can be associated with different types of equality restrictions, and thus to other types of graphical models for paired data for which penalized likelihood methods are not available. For instance, they could be used to identify a subfamily of RCOR models, which impose equality restrictions between specific partial variances and correlations \citep{hojsgaard2008graphical}.

\newpage

\acks{We would like to thank two anonymous referees for their valuable comments, and Saverio Ranciati and Veronica Vinciotti for useful discussion. Financial support was provided by the MUR-PRIN grant 2022SMNNKY (CUP C53D23002580006).}

\appendix

\section{On the Partition of the Edge Set Induced by the Twin-\- pairing Function}\label{SUP.SEC:edge.set.partition}



In this section we provide a detailed example of the partition of the edge set of an undirected graph $G=(V, E)$ as described in Section~\ref{SEC:Notations}.  We let $p=6$ so that $V=\{1,2,3,4,5,6\}$ and, furthermore, we set $L=\{1,2,3\}$. Hence, $R=\{4,5,6\}$ and we assume that $\tau(i)=3+i$, for every $i\in L$, so that $\tau(1)=4$, $\tau(2)=5$ and $\tau(3)=6$.

In this case, the edge set of the complete graph is $F_{V}=\{(i,j)\mid i,j=1,\ldots,6;\; i<j\}$, and it can be partitioned into the three sets $F_{L}$,  $F_{R}$ and $F_{T}$ as follows,
\begin{eqnarray*}
	F_{L}&=&\{(1,2),(1,3),(2,3),(1,5),(1,6),(2,6)\},\\
	F_{R}&=&\{(4,5),(4,6),(5,6),(2,4),(3,4),(3,5)\},\\
	F_{T}&=&\{(1,4),(2,5),(3,6)\},
\end{eqnarray*}
which are graphically represented in Figure~\ref{FIG:FLFRFI}.
\begin{figure}[p]
	\centering
	\begin{subfigure}{0.33\textwidth}
		\centering
		\includegraphics[scale=.9]{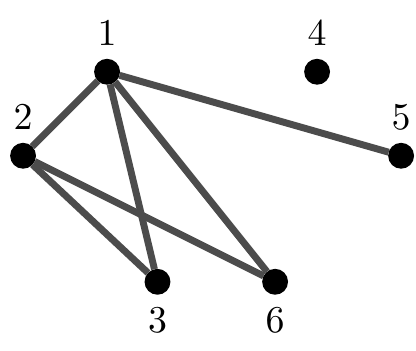}
		\caption{}
	\end{subfigure}
	\begin{subfigure}{0.33\textwidth}
		\centering
		\includegraphics[scale=.9]{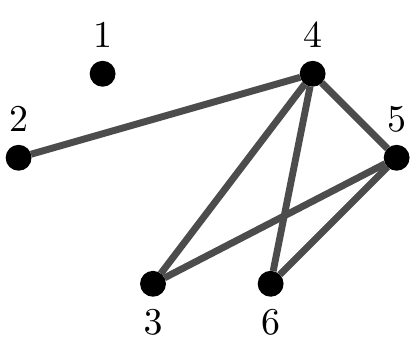}
		\caption{}
	\end{subfigure}
	\begin{subfigure}{0.3\textwidth}
		\centering
		\includegraphics[scale=.9]{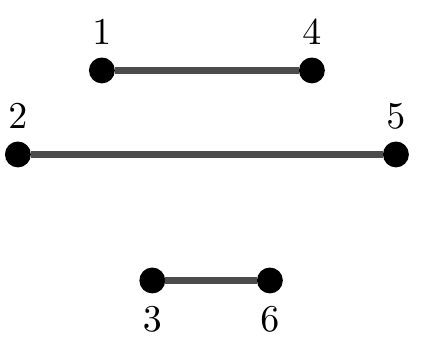}
		\caption{}
	\end{subfigure}
	\caption{Graphs with vertex set $V=\{1,2,3,4,5,6\}$ and edge sets: (a) $F_{L}$, (b) $F_{R}$ and (c) $F_{T}$.}\label{FIG:FLFRFI}
\end{figure}
Note that one can easily check that $F_{R}=\tau(F_{L})$, $F_{L}=\tau(F_{R})$ and $F_{T}=\tau(F_{T})$. In order to understand the meaning of this partition of $F_{V}$ it is useful to note that the set $F_{L}$ can be seen as the union of two disjoint subsets, $F_{L}=F_{L}^{(within)}\cup F_{L}^{(across)}$ where
\begin{eqnarray*}
	F_{L}^{(within)} &=& \{(i,j)\in F_{V}\mid i,j\in L\}\\
	F_{L}^{(across)} &=& \{(i,j)\in F_{V}\mid i\in L, j\in R, i<\tau(j)\}.
\end{eqnarray*}
Dually, $F_{R}=F_{R}^{(within)}\cup F_{R}^{(across)}$ with

\begin{eqnarray*}
	F_{R}^{(within)} &=& \{(i,j)\in F_{V}\mid i,j\in R\}\\
	F_{R}^{(across)} &=& \{(i,j)\in F_{V}\mid i\in L, j\in R, i>\tau(j)\}.
\end{eqnarray*}
Note that $F_{L}^{(within)}=\tau(F_{R}^{(within)})$  and thus $F_{R}^{(within)}=\tau(F_{L}^{(within)})$,  so that every edge in $F_{L}^{(within)}$ has a twin in $F_{R}^{(within)}$.
The edges in $F_{L}^{(within)}$ and $F_{R}^{(within)}$ belong to the left and right group, respectively, and thus the twin-pairing classes characterized by these edges encode similarities involving edges within the two groups. For the example on $p=6$ these two sets become
$F_{L}^{(within)}=\{(1,2),(1,3),(2,3)\}$ and $F_{R}^{(within)}=\{(4,5),(4,6),(5,6)\}$ and the relevant twin-pairing classes are represented in Figure~\ref{FIG:TwinEdges1}.
\begin{figure}
	\centering
	\begin{subfigure}{0.33\textwidth}
		\centering
		\includegraphics[scale=.9]{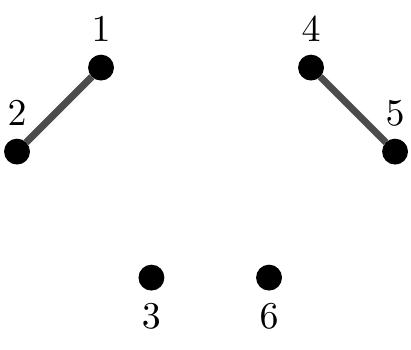}
		\caption{}
	\end{subfigure}
	\begin{subfigure}{0.33\textwidth}
		\centering
		\includegraphics[scale=.9]{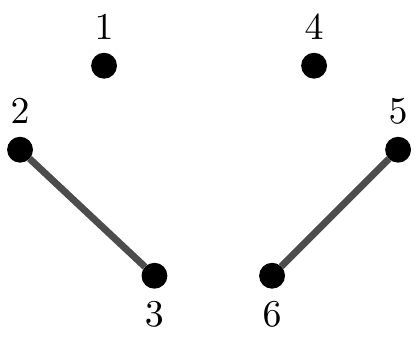}
		\caption{}
	\end{subfigure}
	\begin{subfigure}{0.3\textwidth}
		\centering
		\includegraphics[scale=.9]{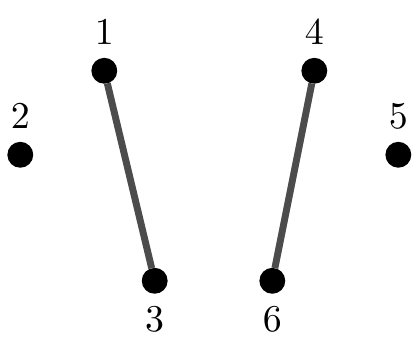}
		\caption{}
	\end{subfigure}
	\caption{Edge twin-pairing classes involving edges within groups: each of the graphs (a), (b) and (c) contains exactly one edge $(i,j)\in F_{L}^{(within)}=\{(1,2),(1,3),(2,3)\}$ and its twin  $\tau(i,j)\in F_{R}^{(within)}=\{(4,5),(4,6),(5,6)\}$.}\label{FIG:TwinEdges1}
\end{figure}

The edges in $F_{L}^{(across)}$ and $F_{R}^{(across)}$ encode the cross-group association structure. Also in this case
$F_{L}^{(across)}=\tau(F_{R}^{(across)})$ and $F_{T}^{(across)}=\tau(F_{T})^{(across)}$ so that every edge in $F_{L}^{(across)}$ has a twin in $F_{R}^{(across)}$. For the case $p=6$ these sets become $F_{L}^{(across)}=\{(1,5),(1,6),(2,6)\}$ and $F_{R}^{(across)}=\{(2,4),(3,4),(3,5)\}$ and the corresponding colour classes are depicted in Figure~\ref{FIG:TwinEdges2}.
\begin{figure}
	\centering
	\begin{subfigure}{0.33\textwidth}
		\centering
		\includegraphics[scale=.9]{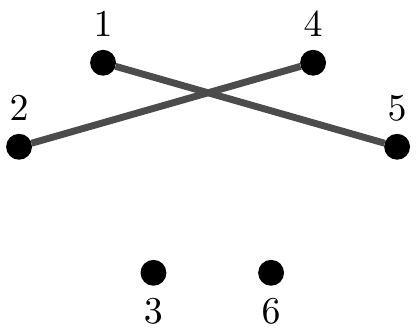}
		\caption{}
	\end{subfigure}
	\begin{subfigure}{0.33\textwidth}
		\centering
		\includegraphics[scale=.9]{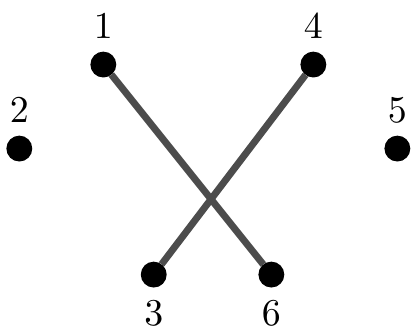}
		\caption{}
	\end{subfigure}
	\begin{subfigure}{0.3\textwidth}
		\centering
		\includegraphics[scale=.9]{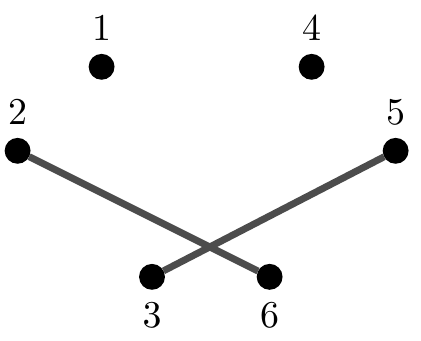}
		\caption{}
	\end{subfigure}
	\caption{Edge twin-pairing classes involving edges across groups: each of the graphs (a), (b) and (c) contains exactly one edge $(i,j)\in F_{L}^{(across)}=\{(1,5),(1,6),(2,6)\}$ and its twin  $\tau(i,j)\in F_{R}^{(across)}=\{(2,4),(3,4),(3,5)\}$.}\label{FIG:TwinEdges2}
\end{figure}

Consider now the \PDCG\ in Figure~\ref{FIG:example-notions}. This is denoted by $\G=(\V, \E)$ with
\begin{eqnarray*}
	\V &=& \{\{1, 4\}, \{2, 5\}, \{3\}, \{6\}\}\\
	\E &=& \{\{(1,2), (4, 5)\},\{(1, 4)\}, \{(1, 6)\}, \{(2,3)\},\{ (3, 4)\}, \{ (3, 5)\}\}.
\end{eqnarray*}
We now compute the equivalent representation $\G=(V, E, \LL, \EE)$. First, we can trivially obtain that $V=\{1,2,3,4,5,6\}$ and $E = \{(1,2), (1, 4), (1, 6), (2,3), (3, 4), (3, 5), (4, 5)\}$. Next, because $L=\{1,2,3\}$ and the only vertex atomic colour classes are $\{3\}$ and $\{6\}$, we can see that $\LL=\{3\}$. Finally, we compute  $E_{L} = E\cap F_{L}=\{(1,2), (1,6), (2,3)\}$ and  $E_{R} = E\cap F_{R}=\{(4, 5), (3, 4), (3,5)\}$. It follows that $\tau(E_{R}) = \{(1, 2), (1, 6), (2, 6)\}$ so that $E_{L} \cap \tau(E_{R}) = \{(1,2), (1, 6)\}$.  Because the edge $(1, 2)$ belongs to a twin-pairing colour class with its twin $(4,5)$, whereas $(1, 6)$ belongs to an atomic colour class, then we obtain that $\EE = \{(1, 6)\}$.
\begin{figure}
	\centering
	\includegraphics[scale=.9]{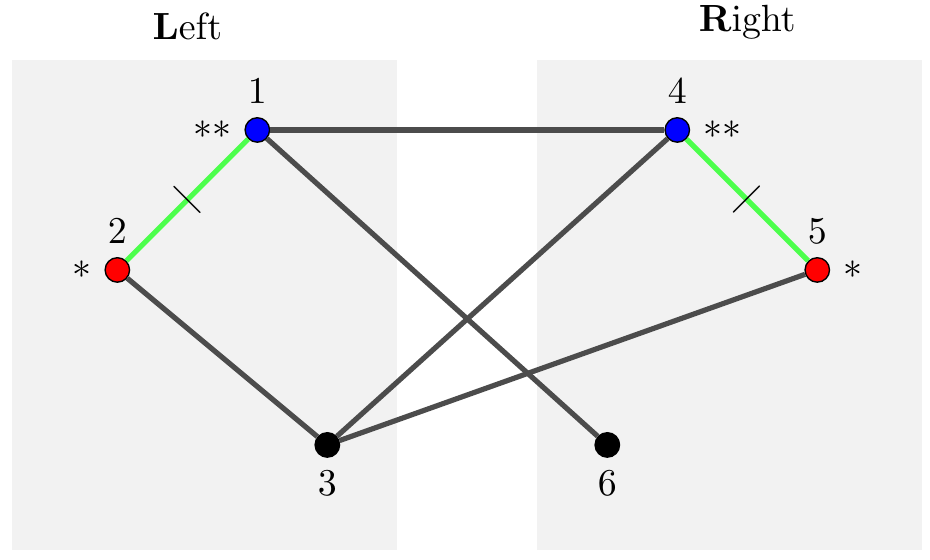}
	\caption{Example of coloured graph for paired data.}\label{FIG:example-notions}
\end{figure}


\section{Dimension of the Search Space of \PDRCON\ Models}\label{SUP.SEC:search.space.dimension}
In this section, we provide a formula for the computation of the number of \PDRCON\ models $\P(V)$ with $|V|=p$.

Firstly, we notice that in a \PDRCON\ model with $p$ variables the maximum numbers of vertex and edge twin-pairing colour classes  are $p/2$ and $|F_{L}| = p(p-2)/4$, respectively. Next, we compute $|\P(V)|$ by steps, starting from the case where there are no twin-pairing colour classes and then increasing the number of twin-twin pairing classes.
\begin{itemize}
	\item the number of \PDCG{s} where all colour classes are atomic coincides with the number of \GGM{s} and is equal to $2^{p \choose 2}$,
	
	\item the number of \PDCG{s} where vertex twin-pairing classes are allowed but edge colour classes are all atomic is $2^{p/2} \times 2^{\binom{p}{2}}$,
	
	\item the number of \PDCG{s}  where vertex twin-pairing classes are allowed  and  at most $1$ edge twin-pairing  class  is present is
	\begin{align*}
		2^{p/2} \binom{p(p-2)/4}{1}2^{\binom{p}{2} - 2},
	\end{align*}
	
	\item  the number of \PDCG{s}  where vertex twin-pairing classes are allowed  and at most $2$ edge twin-pairing  classes  are present is
	\begin{align*}
		2^{p/2} \binom{p(p-2)/4}{2}2^{\binom{p}{2} - 4},
	\end{align*}
	
	\item  the number of \PDCG{s}  where vertex twin-pairing classes are allowed  and at most $i$ edge twin-pairing  classes  are present, with $0 \leq i \leq p(p-2)/4$, is
	\begin{align*}
		2^{p/2} \binom{p(p-2)/4}{i}2^{\binom{p}{2} - 2i}.
	\end{align*}
\end{itemize}

So finally, the number of \PDRCON\ models is
\begin{align*}
	|\P(V)| = 2^{p/2} \sum_{i = 0}^{p(p-2)/4}\binom{p(p-2)/4}{i}2^{\binom{p}{2} - 2i}.
\end{align*}


\section{Technical Details on the Greedy Search procedure}\label{SUP.SEC:stepwise.pseudocode}
Here we provide Algorithms~\ref{ALG:main-detailed} and \ref{ALG:next.step-detailed} which contain a fully detailed, expanded version of the pseudocode given in Algorithms~\ref{ALG:main-short} and \ref{ALG:next.step-short}. Note that the pseudocode calls the procedures \textsc{Is.Accepted()} and \textsc{Best.Model()} that are left unspecified. The procedure \textsc{Is.Accepted()}
takes a model as input and returns \texttt{true} if the model satisfies a given, arbitrary, criterion and \texttt{false} otherwise. The procedure \textsc{Best.Model()} accepts a  non-empty set of models as input and returns one of the models of such set that is identified as ``best'', according to a specified criterion. Finally, an implementation of the procedure in the programming language \textsf{R} can be found at \url{https://github.com/NgocDung-NGUYEN/backwardCGM-PD}.

\begin{algorithm}
	\caption{Expanded version of the pseudocode in Algorithm~\ref{ALG:main-short}.}\label{ALG:main-detailed}
	\begin{algorithmic}[1]
		\Require $V$ and $L$
		\State $\G^{\best}\gets (V, F_{V}, L, F_{L})$
		\State $\mathcal{A}\gets \emptyset$
		\State $K\gets$ maximum number of steps ($\geq 1)$
		\State $k\gets 1$
		\ForAll{$i\in L$}
		\State $\Hs\gets (V, F_{V}, L\setminus \{i\}, F_{L})$
		\If{\Call{Is.Accepted}{$\Hs$}}
		\State $\mathcal{A}\gets \mathcal{A}\cup \{\Hs\}$
		\EndIf
		\State $\Hs\gets (V, F_{V}\setminus \{(i, \tau(i))\}, L, F_{L})$
		\If{\Call{Is.Accepted}{$\Hs$}}
		\State $\mathcal{A}\gets \mathcal{A}\cup \{\Hs\}$
		\EndIf
		\EndFor
		\ForAll{$(i,j)\in F_{L}$}
		\State $\Hs\gets (V, F_{V}, L, F_{L}\setminus \{(i, j)\})$
		\If{\Call{Is.Accepted}{$\Hs$}}
		\State $\mathcal{A}\gets \mathcal{A}\cup \{\Hs\}$
		\Else
		\State $\Hs\gets (V, F_{V}\setminus \{(i, j)\}, L, F_{L}\setminus \{(i, j)\})$
		\If{\Call{Is.Accepted}{$\Hs$}}
		\State $\mathcal{A}\gets \mathcal{A}\cup \{\Hs\}$
		\EndIf
		\State $\Hs\gets (V, F_{V}\setminus \{\tau(i, j)\}, L, F_{L}\setminus \{(i, j)\})$
		\If{\Call{Is.Accepted}{$\Hs$}}
		\State $\mathcal{A}\gets \mathcal{A}\cup \{\Hs\}$
		\EndIf
		\EndIf
		\EndFor
		\While{$\mathcal A\neq\emptyset$ and $k<K$}
		\State $\mathcal{A},\;\G^{\best}\gets\;$\Call{Update}{$\mathcal{A},\; \G^{\best}$}  \Comment{see Algorithm~\ref{ALG:next.step-detailed}}
		\State $k\gets k+1$
		\EndWhile
		\lIf{$\mathcal A\neq\emptyset$ and $k=K$}
		$\G^{\best}\gets$\Call{Best.Model}{$\mathcal{A}$}
		\State \textbf{return} $\G^{\best}$
	\end{algorithmic}
\end{algorithm}
\begin{algorithm}[t]
	\caption{Expanded version of the pseudocode in Algorithm~\ref{ALG:next.step-short}.} \label{ALG:next.step-detailed}
	\begin{algorithmic}[1]
		\Procedure{Update}{$\mathcal{A},\; \G^{\best}$}
		\State $\G^{\old}\gets \G^{\best}$
		\State $\G^{\best}\gets$\Call{Best.Model}{$\mathcal{A}$}
		\Comment{\parbox[t]{.4\linewidth}{Note that:\\[1ex]
				$\G^{\old}=(V, E^{\old}, \LL^{\old}, \EE^{\old})$\\[1ex]
				$\G^{\best}=(V, E^{\best}, \LL^{\best}, \EE^{\best})$\\}}
		\State $e\gets E^{\best}\setminus E^{\old}$
		\If{$|e|=1$ and $e\neq \tau(e)$}
		\State $\Hs\gets (V, E^{\old}\setminus \{\tau(e)\}, \LL^{\old}, \EE^{\old}\setminus \{e, \tau(e)\})$
		\State $\mathcal{A}\gets \mathcal{A}\setminus \{\Hs\}$
		\ElsIf{$e=\emptyset$ and $\EE^{\best}\neq \EE^{\old}$}
		\State $e\gets \EE^{\old}\setminus \EE^{\best}$
		\State $\Hs\gets (V, E^{\old}\setminus \{e, \tau(e)\}, \LL^{\old}, \EE^{\old}\setminus \{e\})$
		\State $\mathcal{A}\gets \mathcal{A}\cup \{\Hs\}$
		\EndIf
		\State $\mathcal{A^{\old}}\gets \mathcal{A}\setminus \{\G^{\best}\}$
		\State $\mathcal{A}\gets\emptyset$
		\ForAll{$\F\in\mathcal{A}^{\old}$}
		\State $\Hs\gets \F\wedge_{t}\G^{\best}$
		\lIf{\Call{Is.Accepted}{$\Hs$}} $\mathcal{A}\gets \mathcal{A}\cup \{\Hs\}$
		\EndFor
		\State \textbf{return} $\mathcal{A}$ and $\G^{\best}$
		\EndProcedure
	\end{algorithmic}
\end{algorithm}
%


\section{Additional Details on the Applications}
\subsection{Simulations}\label{SUP.SEC:simulations}

In this section, we provide additional details on the simulations described in Section~\ref{SEC:simulations}; see also \url{https://github.com/NgocDung-NGUYEN/backwardCGM-PD}.

Table \ref{TAB:simulation-experiment} considers the 8 \PDCG{s} used in the simulations and gives details on edges, pairs of twin edges present in the graph and colour classes.

\begin{table}
	\centering
	\begin{tabular}{c c | c c c c c c c c c}
		\hline
		\multirow{2}{*}{$p$} & \multirow{2}{*}{\textbf{Scen.}} & \multicolumn{4}{c}{{\bf Structure}} && \multicolumn{4}{c}{{\bf Symmetries}} \\
		&&den.$_{\%}$ & $|E|$& $|E_{T}|$& $| E_{L} \cap \tau(E_{R})|$ && $|\EE|$ & $|\EE^{c}|$& $|\LL|$& $|\LL^{c}|$ \\
		\hline \hline && ~\\
		\multirow{2}{*}{$8$} &A& $17.85$ &$5$&$0$& $2$ && $1$ & $1$ & $3$ & $1$ \\
		&B& $35.71$ &$10$&$0$& $4$ && $1$ & $3$ & $1$ & $3$ \\
		&& ~\\
		
		\multirow{2}{*}{$12$}&A& $18.18$ &$12$&$0$& $5$ && $4$ & $1$ & $4$ & $2$ \\
		&B & $34.85$ &$23$&$1$& $9$ && $3$ & $6$ & $2$ & $4$ \\
		&& ~\\
		
		\multirow{2}{*}{$16$}&A& $18.33$ &$22$&$1$& $9$ && $7$ & $2$ & $6$ & $2$\\
		&B& $35.00$ &$42$&$3$& $15$ && $5$ & $10$ & $2$ & $6$ \\
		&& ~\\
		
		\multirow{2}{*}{$20$}&A& $17.94$ &$34$&$2$& $14$ && $11$ & $3$ & $8$ & $2$ \\
		&B& $34.74$ &$66$&$6$& $24$ && $8$ & $16$ & $2$ & $8$ \\
		\hline
	\end{tabular}
\caption{Details on the \PDCG{s} used in the simulations: den.$_{\%}$ is the density computed from $|E|/|F_{V}|$; $|E|$ is  the number of edges; $|E_{T}|$ is the number of edges connecting pairs of twin vertices; $|E_{L} \cap \tau(E_{R})|$ is the number of edges for which a twin is present in the graph;  $|\EE|$ and $|\LL|$ are the number of edge and vertex atomic colour classes, respectively;  $|\EE^{c}|=|(E_{L} \cap \tau(E_{R}))\setminus \EE |$ and $|\LL^{c}|=| L \setminus \LL|$ are the number of edge and vertex twin-pairing classes, respectively.}\label{TAB:simulation-experiment}
\end{table}

For each \PDCG\ in Table~\ref{TAB:simulation-experiment} above, we generated a concentration matrix by exploiting the \textsf{R} package \texttt{gRc} \citep{hojsgaard2007inference}. More specifically, we applied the function \texttt{rcox()} by giving in input the $p\times p$ equicorrelation matrix with the diagonal entries equal to $1$, and all off-diagonal entries equal to $0.5$.  Each of resulting concentration matrices was used to generate $20$ random samples of size $n=100$ from the relevant multivariate normal distribution, with mean vector equal to zero. The two procedures were then applied to the $20\times 8 = 160$ datasets so obtained, and we also remark that we exploited the package \texttt{gRc} for the computation of maximum likelihood estimates.

Table~\ref{TAB:measure-simmulation} summarizes the average results over all repetitions of the simulated data of the performance scores which measure the identifications of  the zeros and symmetric structures of the resulting models recovered from the model selection procedure. Specifically, for the accuracy of the graph structures, we considered the quantities such as the edge positive-predicted value (ePPV), the edge true-positive rate (eTPR), and the edge true-negative rate (eTNR), which are formally specified as
\begin{eqnarray*}
	\ePPV = \frac{\eTP}{\# \text{edges}}, \quad \eTPR = \frac{\eTP}{\eP}, \quad \eTNR = \frac{\eTN}{\eN},
\end{eqnarray*}
where
\begin{itemize}
	\item eTP: the number of true edges in the selected graph,
	\item $\#$edges: the number of edges in the selected graph,
	\item eP:  the number of edges in the true graph,
	\item eTN: the number of true missing edges in the selected graph,
	\item eN: the number of missing edges in the true graph.
\end{itemize}

Similarly, for the identification of the symmetries, that is of the twin-pairing classes, we considered the symmetry positive-predicted value (sPPV), the symmetry true-positive rate (sTPR), and the symmetry true-negative rate (sTNR), as defined by \cite{ranciati2021fused}. Specifically, they are computed as
\begin{eqnarray*}
	\sPPV = \frac{\sTP}{\# \text{sym}}, \quad \sTPR = \frac{\sTP}{\sP}, \quad \sTNR = \frac{\sTN}{\sN},
\end{eqnarray*}
where
\begin{itemize}
	\item sTP: the number of pairs of true twin-pairing edges in the selected graph,
	\item $\#$sym: the number of pairs of twin-pairing edges in the selected graph,
	\item sP:  the number of pairs of twin-pairing edges in the true graph,
	\item sTN:the number of pairs of twin-pairing edges that are missing in the selected graph,
	\item sN: the number of pairs of twin-pairing edges that are missing in the true graph.
\end{itemize}

Finally, on the last two columns of the table we report the computational time of the procedures and the total number of fitted models during the execution of the algorithms.

\begin{sidewaystable}[p]
	\centering
	\begin{tabular}{ c c c | c c c c c c c c c r  r}
		\hline
		\multirow{2}{*}{\textbf{S}} & \multirow{2}{*}{$p$} & \multirow{2}{*}{\textbf{Order}} & \multicolumn{4}{c}{{\bf Graph structure}} && \multicolumn{4}{c}{{\bf Symmetries}} &
		\multirow{2}{*}{\textbf{Time$_{\text{(s)}}$}} & \multirow{2}{*}{\textbf{$\#$models}} \\
		&&&$\#$edges & ePPV$_{\%}$& eTPR$_{\%}$& eTNR$_{\%}$ && $\#$sym & sPPV$_{\%}$ & sTPR$_{\%}$& sTNR$_{\%}$ && \\
		\hline \hline && ~\\
		\multirow{11}{*}{\textbf{A}} & \multirow{2}{*}{$8$} &$\preceq_{t}$& $7 (2)$ &$76.68$&$100.00$& $91.52$ && $2(1)$ & $41.67$ & $95.00$ & $89.44$ & $ 4$ & $273$\\
		&&$\preceq_{s}$ & $7(2)$ &$75.41$&$100.00$& $91.30$ && $2(1)$ & $46.67$ & $95.00$ & $85.56$ & $17$& $580$\\
		&& ~\\
		
		&\multirow{2}{*}{$12$}&$\preceq_{t}$& $17(3)$ &$71.22$&$97.92$& $90.37$ && $6(1)$ & $15.99$ & $90.00$ & $87.61$ & $19$ & $1300$\\
		&&$\preceq_{s}$ & $17(3)$ &$70.23$&$98.75$& $90.00$ && $5(1)$ & $17.34$ & $90.00$ & $83.91$ & $109$ & $2985$\\
		&& ~\\
		
		&\multirow{2}{*}{$16$}&$\preceq_{t}$& $27(4)$ &$74.83$&$88.64$& $92.70$ && $9(1)$ & $18.53$ & $85.00$ & $89.43$ & $89$ & $4245$\\
		&&$\preceq_{s}$ & $28(4)$ &$70.98$&$87.05$& $91.48$ && $8(1)$ & $19.32$ & $77.50$ & $84.77$ & $532$ & $10554$\\
		&& ~\\
		
		&\multirow{2}{*}{$20$}&$\preceq_{t}$& $44(8)$ &$64.24$&$82.21$& $89.49$ && $16(3)$ & $13.47$ & $70.00$ & $86.18$ & $379$ & $10212$\\\
		&&$\preceq_{s}$ & $46(7)$ &$60.11$&$78.97$& $88.04$ && $13(3)$ & $11.97$ & $51.67$ & $80.00$ & $2102$ & $27356$\\
		\hline
		&& ~\\
		\multirow{11}{*}{\textbf{B}}& \multirow{2}{*}{$8$} &$\preceq_{t}$& $11(2)$ &$84.54$&$89.50$& $89.72$ && $5(1)$ & $64.08$ & $93.33$ & $92.50$ & $4$ & $264$\\
		&&$\preceq_{s}$ & $11(2)$ &$83.59$&$89.00$& $89.44$ && $4(1)$ & $64.83$ & $85.00$ & $85.83$ & $15$ & $486$\\
		&& ~\\
		
		& \multirow{2}{*}{$12$} &$\preceq_{t}$& $23(4)$ &$81.78$&$80.00$& $89.65$ && $9(2)$ & $56.28$ & $79.17$ & $87.35$ & $ 19$ & $1230$\\
		&&$\preceq_{s}$ & $23(4)$ &$81.25$&$78.48$& $89.53$ && $7(2)$ & $63.26$ & $73.33$ & $83.53$ & $102$ & $2729$\\
		&& ~\\
		
		& \multirow{2}{*}{$16$} &$\preceq_{t}$& $34(5)$ &$72.49$&$57.86$& $87.63$ && $12(2)$ & $52.38$ & $64.00$ & $86.09$ & $96$ & $4259$\\
		&&$\preceq_{s}$ & $31(4)$ &$74.50$&$55.24$& $89.49$ && $9(2)$ & $63.36$ & $54.00$ & $82.97$ & $523$ & $10247$\\
		&& ~\\
		
		& \multirow{2}{*}{$20$} &$\preceq_{t}$& $51(9)$ &$69.74$&$53.41$& $87.02$ && $18(2)$ & $48.17$ & $54.38$ & $84.07$ & $452$ & $10300$\\
		&&$\preceq_{s}$ & $48(7)$ &$67.81$&$48.64$& $87.22$ && $12(2)$ & $52.97$ & $39.38$ & $78.98$ & $2226$ & $26961$\\
		\hline
	\end{tabular}
\caption{Performance measures of the model selection procedures on the lattice structures ordered by $\preceq_{t}$ and $\preceq_{s}$. Results are recorded as mean computed across the $20$ replicated datasets for each scenario. Here, the columns of $\#$edges and $\#$sym include the average numbers of edges and of pairs of symmetric edges, respectively, and the numbers in the round bracket are the corresponding standard deviations. The columns of \textbf{Time} and $\#$\textbf{models} measure the efficiency of the procedures that are the average of computational time and the average of numbers of fitted models, respectively.}\label{TAB:measure-simmulation}
\end{sidewaystable}

\subsection{Additional Details on the Application to fMRI Data}\label{SEC:Supp-fMRI}
We applied our method to a multimodal imaging dataset coming from a pilot study of the Enhanced Nathan Kline Institute-Rockland Sample project, and provided by Greg Kiar and Eric Bridgeford from NeuroData at Johns Hopkins University, who  pre-processed the raw DTI and R-fMRI imaging data available at \url{http://fcon_1000.projects.nitrc.org/indi/CoRR/html/nki_1.html}, using the pipelines \text{ndmg} and \text{C-PAC}. Particularly, the R-fMRI monitors brain functional activity at different regions via dynamic changes in the blood oxygenation level dependent (BOLD) signal, when, in this study, the subjects are simply asked to stay awake with eyes open. The data sets that we apply our methods are residuals estimated from the vector autoregression models carried out to remove the temporal dependence, see \cite{ranciati2021fused}.

We considered $36$ cortical brain regions including  $22$ regions in the frontal lobe and  $14$ regions in the anterior temporal lobe for two subjects 14 and 15. In this section, we provide the analysis of subject 14 while the analysis of subject 15 is given in Section~\ref{SEC:real.data}. In particular, the graph of the selected model by our method, called by $\G$, has density equal to $54.28\%$ with $342$ edges present in which there are $76$ edge twin-pairing colour classes, which to make up approximately $44.44\%$ on present edges, and $6$ vertex twin-pairing colour classes. Based on the likelihood ratio test for the comparison with the saturated model, the model $\G$ has $p$-values equal to $0.055$ with $296$ degrees of freedom computed on the asymptotic chi-squared distribution.

Figure \ref{FIG:36regionSub14} presents the coloured graphical representation of the model $\G$ that is split into $4$ panels. The top-left panel gives the edges of $\G$ that belong to $E_{L}$ and form atomic colour classes, and similarly for the top-right panel that  gives the atomic colour classes in $E_{R}$. Furthermore, the bottom-left and right panels give the edge twin-pairing colour classes between and across groups, respectively. Here, in this case, we have omitted to represent the edges in $E_{T}$ because this set is complete, in the sense that $E_{T} = F_{T}$.

\begin{figure}
	\centering
	\begin{subfigure}{0.5\textwidth}
		\centering
		\includegraphics[scale=0.45]{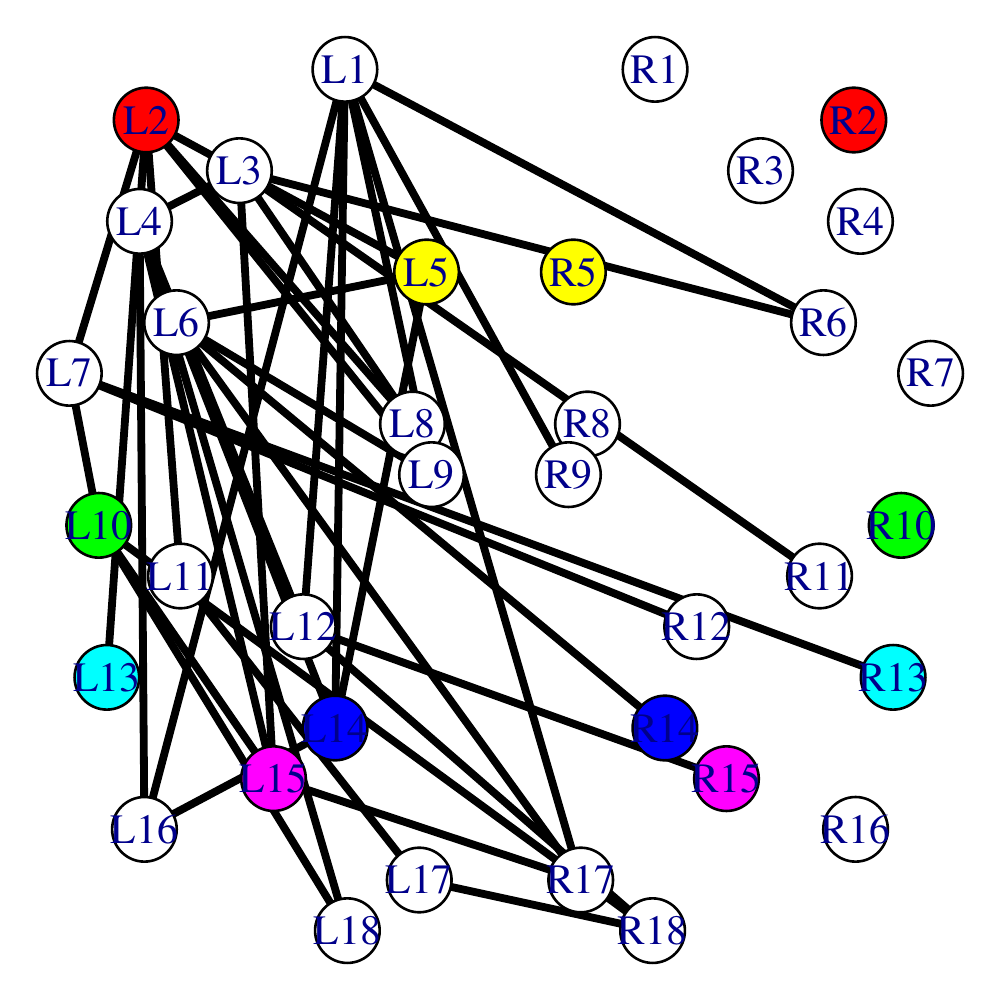}
		\caption{}
	\end{subfigure}
	\begin{subfigure}{0.45\textwidth}
		\centering
		\includegraphics[scale=0.45]{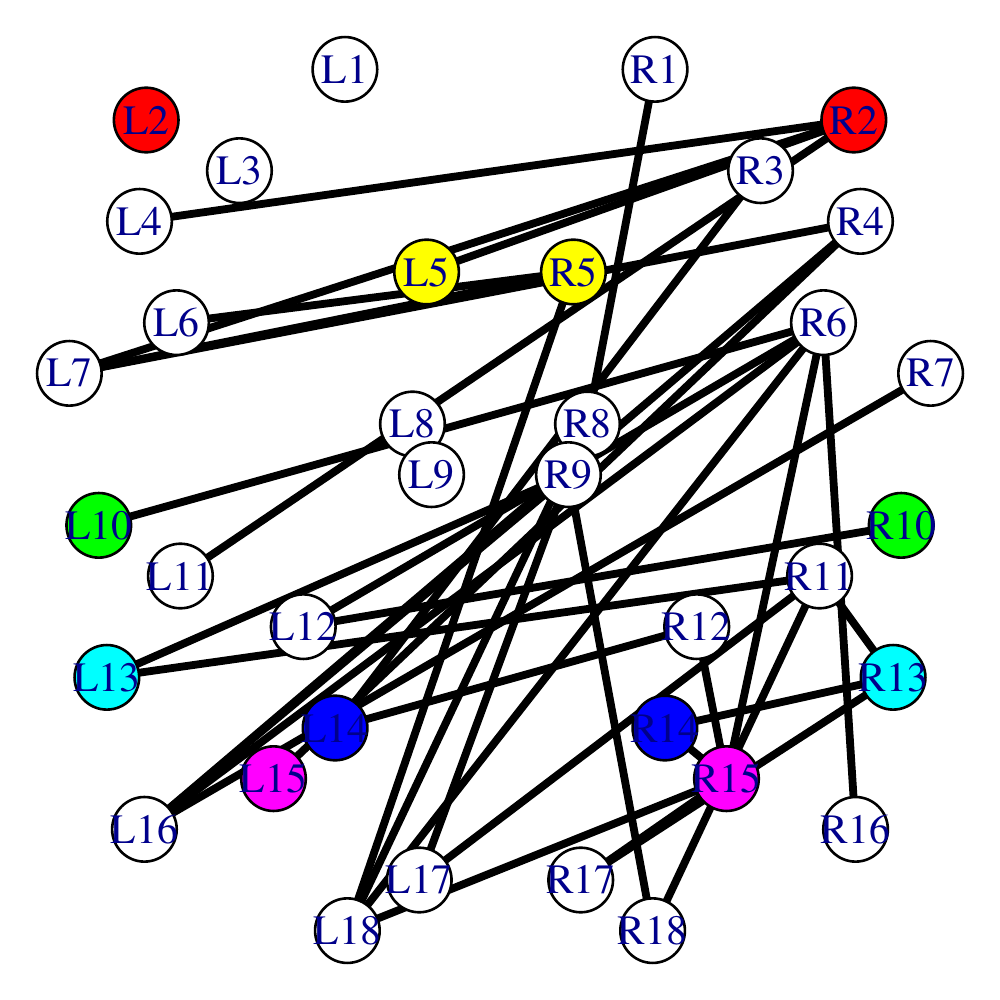}
		\caption{}
	\end{subfigure}

	\begin{subfigure}{0.5\textwidth}
		\centering
		\includegraphics[scale=0.45]{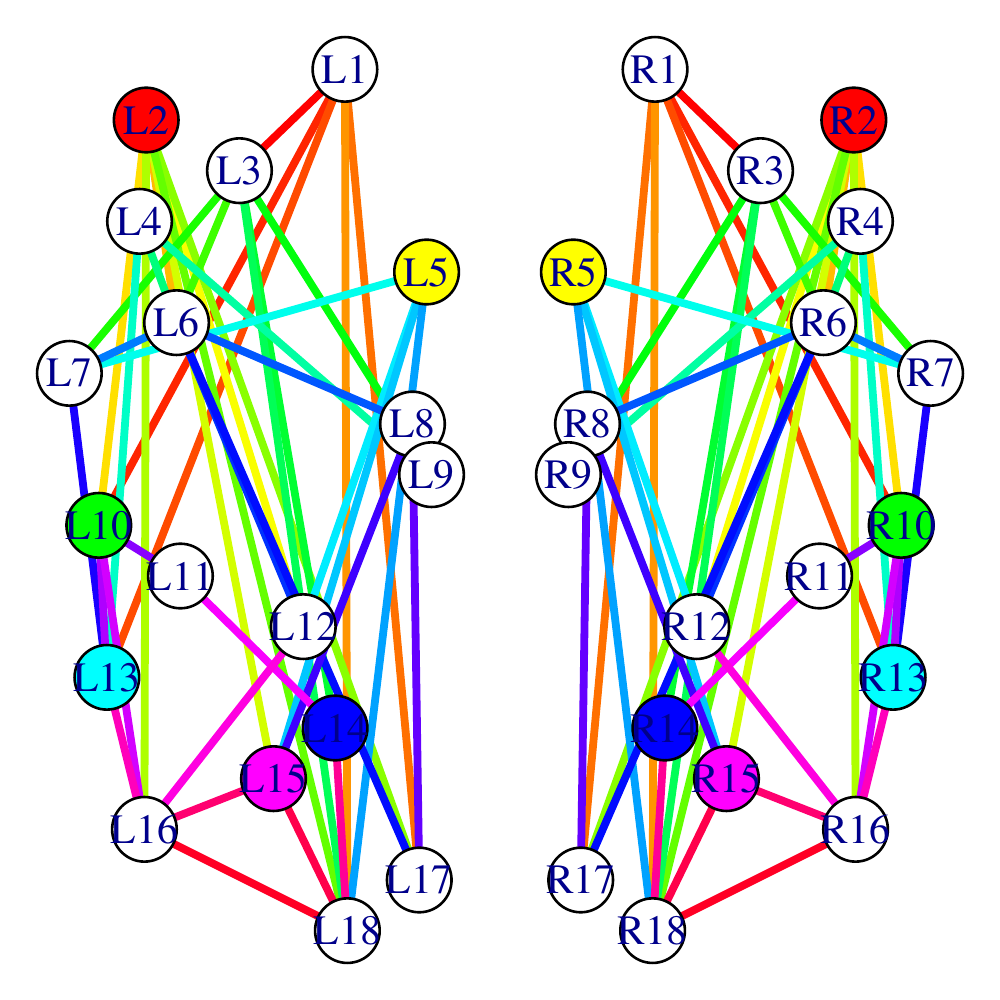}
		\caption{}
	\end{subfigure}
	\begin{subfigure}{0.45\textwidth}
		\centering
		\includegraphics[scale=0.45]{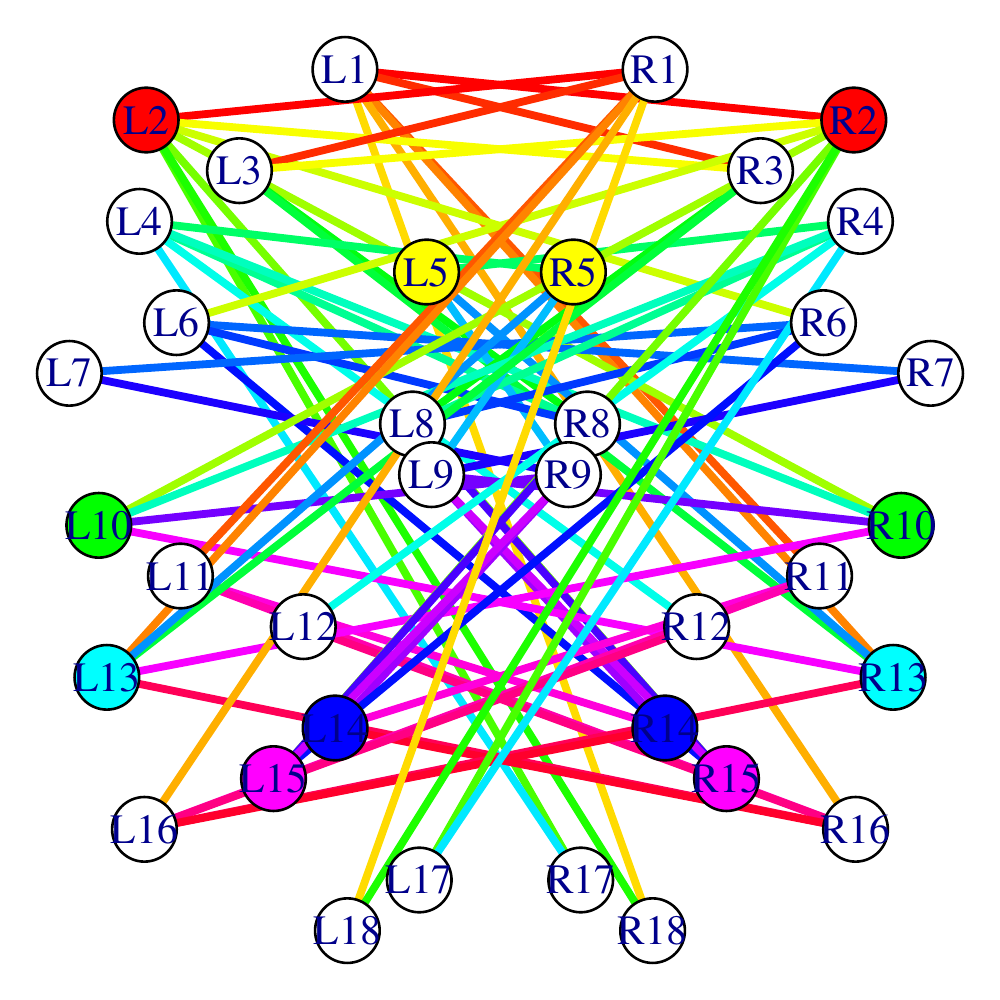}
		\caption{}
	\end{subfigure}
	\caption{The \PDCG\ of the selected models for subject $14$ with separate panels: (a) edges in $E_{L}$ forming atomic classes; (b) edges in $E_{R}$ forming atomic classes; (c) twin-pairing classes within groups; and (d) twin pairing classes across groups.}  \label{FIG:36regionSub14}
\end{figure}

\section{Proofs}\label{SUP.SEC:proofs}

\subsection{Proof of Proposition~\ref{THM:new.notation.inverse.function}}\label{SUP.SEC:proof-inverse-function}

We now show that from the quadruplet $(V,  E, \LL, \EE)$ obtained from \eqref{EQN:uncoloured.version} and \eqref{EQ:bbL.and.BBE} one can recover the representation $\G=(\V, \E)$ by applying (i) and (ii). Recall that every  colour class of both $\V$ and $\E$ is either atomic or twin-pairing. We consider first the vertex colour classes $\V$. Every twin-pairing vertex colour class can be written as $\{i, \tau(i)\}$ with $i\in L$ and  $\tau(i)\in R$, and it follows from the definition of $\LL$ in \eqref{EQ:bbL.and.BBE} that a vertex $i\in L$ belongs to a twin-pairing colour class if and only if $i\in L\setminus \LL$. Hence, from the pair $L$ and $\LL$ we can obtain the twin-pairing colour classes in $\V$ which are the sets $\{i, \tau(i)\}$ for all $i\in L\setminus \LL$. It is straightforward that all the vertices in $V$ not belonging to any of the above twin-pairing classes must form atomic colour classes.
We now turn to the edge colour classes $\E$. The partition $(E_{L}, E_{R}, E_{T})$ of $E$ is obtained in such a way that every twin-pairing edge colour class can be written as $\{(i,j), \tau(i,j)\}$ with $(i,j)\in E_{L}$ and $\tau(i,j)\in E_{R}$ and therefore it holds that  $(i,j)\in E_{L}\cap \tau(E_{R})$. Furthermore, by construction, the set $\EE$ in \eqref{EQ:bbL.and.BBE} is made up of all the edges in $E_{L}\cap \tau(E_{R})$ which form atomic colour classes. It follows that the twin-pairing edge colour classes of $\G$ are given by the pairs  $\{(i,j), \tau(i,j)\}$ for all $(i,j)\in (E_{L}\cap \tau(E_{R}))\setminus \EE$, as required. Also in this case it is straightforward to see that all the edges in $E$ not belonging to any of the above twin-pairing classes must form atomic colour classes.

%

\subsection{Proof of Theorem~\ref{THM:existing-novel-notations}}\label{SUP.SEC:proof-existing-novel-notations}
Let $\G=(\V, \E)\in \P$ be a \PDCG\ and let $(V, E, \LL, \EE)$ the quadruplet obtained from the  application of  \eqref{EQN:uncoloured.version} and \eqref{EQ:bbL.and.BBE}. It follows immediately from the construction of $G=(V, E)$ in  \eqref{EQN:uncoloured.version} and of  $\LL$ and $\EE$ in \eqref{EQ:bbL.and.BBE} that $(V, E, \LL, \EE)$ is compatible with $(L, R)$. Furthermore, the pair $(\V, \E)$ can be recovered from  $(V, E, \LL, \EE)$ as shown in Proposition~\ref{THM:new.notation.inverse.function}.

Consider now a compatible quadruplet $(V, E, \LL, \EE)$ and let $\G=(\V, \E)$ the graph obtained from the application of Proposition~\ref{THM:new.notation.inverse.function}. It follows immediately from the constructing procedure that $\G\in \P$ and we have to show that if we apply  \eqref{EQN:uncoloured.version} and  \eqref{EQ:bbL.and.BBE} to $\G$ then we recover the quadruplet $(V, E, \LL, \EE)$. It is straightforward to see that  \eqref{EQN:uncoloured.version} gives $V_{\G}=V$ and  $E_{\G}=E$. By (i) of Proposition~\ref{THM:new.notation.inverse.function}, the twin-pairing vertex colour classes of $\G$ are $\{i, \tau(i)\}$ for all $i\in L\setminus\LL$ and therefore we obtain from   \eqref{EQ:bbL.and.BBE} that $\LL_{\G}=L\cap\LL=\LL$ because $\LL\subseteq L$ by compatibility. Similarly, we obtain from (ii) of Proposition~\ref{THM:new.notation.inverse.function} and  \eqref{EQ:bbL.and.BBE} that $\EE_{\G}=(E_{L}\cap \tau(E_{R}))\cap \EE=\EE$ because $\EE\subseteq E_{L}\cap \tau(E_{R})$ by compatibility. And this completes the proof.

%

\subsection{Proof of Theorem \ref{THM:LatticePropertiesofOrderTau} }\label{SUP.SEC:proof-LatticePropertiesofOrderTau}
The proof of this theorem is as follows: we firstly prove that $\langle \P, \preceq_{t}\rangle$ is a lattice by determining the explicit forms of the corresponding meet and join operations. We then identify the zero and the unit of the lattice $\langle \P, \preceq_{t}\rangle$, which implies  completeness. Finally, we prove that $\langle \P, \preceq_{t}\rangle$ is distributive.

We first consider \ref{meet.twin.lattice} and show that $\Ls=(V, E_{\G} \cap E_{\Hs}, \LL_{\G}\cap \LL_{\Hs}, \EE_{\G} \cap \EE_{\Hs})$ is the meet of $\G$ and $\Hs$, $\G \wedge_{t} \Hs$.
It can be easily checked that, by construction,  the quadruplet $(V, E_{\G} \cap E_{\Hs}, \LL_{\G}\cap \LL_{\Hs}, \EE_{\G} \cap \EE_{\Hs})$ is compatible so that $\Ls\in \P$. Furthermore, because $E_{\G} \cap E_{\Hs}\subseteq E_{\G}$, $\LL_{\G}\cap \LL_{\Hs}\subseteq \LL_{\G}$ and $\EE_{\G} \cap \EE_{\Hs}\subseteq \EE_{\G}$, then it follows immediately from Definition~\ref{DEF:twin-order} that $\Ls\preceq_{t}\G$ and, similarly, one can show that $\Ls\preceq_{t}\Hs$, so that $\Ls\preceq_{t}\G,\Hs$. In this way, we have shown that $\Ls$ is a lower bound of $\G$ and $\Hs$ and now, in order to prove that $\Ls=\G \wedge_{t} \Hs$, we have to show that $\Ls$ is the greatest lower bound, or infimum, of $\G$ and $\Hs$. More formally, we have to show that if $\F=(V, E_{\F}, \LL_{\F}, \EE_{\F})$ is an arbitrary lower bound of $\G$ and $\Hs$ then $\F\preceq_{t} \Ls$. For every lower bound $\F$ of $\G$ and $\Hs$ it holds that $\F\preceq_{t}\G,\Hs$ and, therefore, that
\begin{enumerate}[label = \arabic*)]
	\item both $E_{\F} \subseteq  E_{\G}$ and  $E_{\F} \subseteq  E_{\Hs}$,
	\item both $\LL_{\F} \subseteq  \LL_{\G}$ and  $\LL_{\F} \subseteq  \LL_{\Hs}$,
	\item both $\EE_{\F} \subseteq  \EE_{\G}$ and  $\EE_{\F} \subseteq  \EE_{\Hs}$.
\end{enumerate}
In turn, this implies that $E_{\F} \subseteq  E_{\G}\cap E_{\Hs}$, $\LL_{\F} \subseteq  \LL_{\G}\cap \LL_{\Hs}$ and $\EE_{\F} \subseteq  \EE_{\G}\cap \EE_{\Hs}$ and therefore that $\F\preceq_{t} \Ls$, as required.

We now consider \ref{join.twin.lattice} and show that $\U=(V, E_{\G} \cup E_{\Hs}, \LL_{\G}\cup \LL_{\Hs}, \EE_{\G} \cup \EE_{\Hs})$ is the join of $\G$ and $\Hs$, $\G \vee_{t} \Hs$.
It can be easily checked that, by construction,  the quadruplet $(V, E_{\G} \cup E_{\Hs}, \LL_{\G}\cup \LL_{\Hs}, \EE_{\G} \cup \EE_{\Hs})$ is compatible so that $\U\in \P$. Furthermore, because $ E_{\G} \subseteq E_{\G}\cup E_{\Hs}$, $\LL_{\G}\subseteq \LL_{\G}\cup \LL_{\Hs}$ and $\EE_{\G}\subseteq\EE_{\G} \cup \EE_{\Hs}$, then it follows immediately from Definition~\ref{DEF:twin-order} that $\G\preceq_{t}\U$ and, similarly, one can show that $\Hs\preceq_{t}\U$ so that $\G,\Hs\preceq_{t}\U$. In this way, we have shown that $\U$ is an upper bound of $\G$ and $\Hs$ and now, in order to prove that $\U=\G \vee_{t} \Hs$, we have to show that $\U$ is the least upper bound, or supremum, of $\G$ and $\Hs$. More formally, we have to show that if $\F=(V, E_{\F}, \LL_{\F}, \EE_{\F})$ is an arbitrary upper bound of $\G$ and $\Hs$ then $\U\preceq_{t} \F$. For every upper bound $\F$ of $\G$ and $\Hs$ it holds that $\G,\Hs\preceq_{t}\F$ and, therefore, that
\begin{enumerate}[label = \arabic*)]
	\item both $E_{\G} \subseteq  E_{\F}$ and  $E_{\Hs} \subseteq  E_{\F}$,
	\item both $\LL_{\G} \subseteq  \LL_{\F}$ and  $\LL_{\Hs} \subseteq  \LL_{\F}$,
	\item both $\EE_{\G} \subseteq  \EE_{\F}$ and  $\EE_{\Hs} \subseteq  \EE_{\F}$.
\end{enumerate}
In turn, this implies that $E_{\G}\cup E_{\Hs}\subseteq E_{\F}$, $\LL_{\G}\cup \LL_{\Hs} \subseteq  \LL_{\F}$ and $\EE_{\G}\cup \EE_{\Hs} \subseteq  \EE_{\F}$ and therefore that $\U\preceq_{t} \F$, as required.

We turn now to the $\hat{0}$ and the $\hat{1}$. It can be easily checked that both $\left(V,\;\emptyset,\;\emptyset,\; \emptyset \right)$ and $\left(V,\; F_{V},\; L,\; F_{L} \right)$ are compatible and therefore they both belong to $\P$. We have then to show that for every $\F\in \P$ it holds that $\hat{0}\preceq_{t} \F\preceq_{t}\hat{1}$ and this follows immediately from the fact the $\emptyset\subseteq E_{\F}\subseteq F_{V}$, $\emptyset\subseteq \LL_{\F}\subseteq L$ and $\emptyset\subseteq \EE_{\F}\subseteq F_{L}$.

Finally, we show that $\langle\P, \preceq_{t}\rangle$ is distributive, that is the operations of meet and the join distribute over each other. Formally, we have to show that for all $\F,\G,\Hs\in \P$ it holds that
\begin{align*}
	\F\wedge_{t} (\G\vee_{t}\Hs)=(\F\wedge_{t} \G)\vee_{t}(\F\wedge_{t} \Hs)
\end{align*}
and this follows from the fact that the operations of union and intersection between sets distribute over each other so that
\begin{align*}
	\F\wedge_{t} (\G\vee_{t}\Hs)=\large(V,\, E_{\F}\cap (E_{\G}\cup E_{\Hs}),\, \LL_{\F}\cap (\LL_{\G}\cup \LL_{\Hs}),\, \EE_{\F}\cap (\EE_{\G}\cup \EE_{\Hs})\large)
\end{align*}
is equal to
\begin{align*}
	&(\F\wedge_{t} \G)\vee_{t}(\F\wedge_{t} \Hs)\\&=\large(V,\, (E_{\F}\cap E_{\G})\cup(E_{\F}\cap E_{\Hs}),\, (\LL_{\F}\cap \LL_{\G})\cup(\LL_{\F}\cap \LL_{\Hs}),\, (\EE_{\F}\cap \EE_{\G})\cup(\EE_{\F}\cap \EE_{\Hs})\large).
\end{align*}

%

\subsection{Proof of Proposition \ref{THM:relation-modelinclusion-twinorder}}\label{SUP.SEC:proof-relation-modelinclusion-twinorder}
We first show \ref{rel.s.t.i} that $\Hs \preceq_{s} \G$  implies  $\Hs \preceq_{t} \G$. The \PDCG{s} $\Hs$ and $\G $ are such that $\Hs \preceq_{s} \G$ if and only if conditions \ref{S1}, \ref{S2} and \ref{S3} hold true, and therefore we have to prove that the latter three conditions imply \ref{iSV}, \ref{iSV} and \ref{iiiSV}. Conditions \ref{S1} and \ref{iSV} are trivially equivalent.  We now show that \ref{S2} implies \ref{iiSV}. By construction, all the vertices in $\LL_{\Hs}$ belong to atomic colour classes in $\V_{\Hs}$ and, more precisely,
for every $i\in \LL_{\Hs}$ it holds that both $\{i\}$ and $\{\tau(i)\}$ are atomic colour classes in $\V_{\Hs}$. Hence, by  \ref{S2}, $\{i\}$ and $\{\tau(i)\}$ are atomic colour classes also in $\V_{\G}$ which implies that $i\in \LL_{\G}$ and therefore that $\LL_{\Hs}\subseteq \LL_{\G}$ as required. Finally, we show that \ref{S1} and \ref{S3} implies \ref{iiiSV}.  By construction,  every edge $(i,j) \in \EE_{\Hs}$ is such that $\tau(i,j)\in E_{\Hs}$ and, furthermore, that $\{(i,j)\}$ and $\{\tau(i,j)\}$ are both atomic classes in $\E_{\Hs}$. It follows by \ref{S1} that $(i, j),\tau(i,j)\in E_{\G}$ and thus, by \ref{S3}, that $\{(i,j)\}$ and $\{\tau(i,j)\}$ are both atomic classes also in $\E_{\G}$. In turn, this implies that  $(i,j) \in \EE_{\G}$ and therefore that  $\EE_{\Hs} \subseteq \EE_{\G}$.

We now consider the statement  \ref{rel.s.t.ii} and show it by contradiction. Assume that both $\Hs \preceq_{s} \G$ and $\Hs \precdot_{t} \G$ but $\Hs \not\precdot_{s} \G$. Then, there exists a graph $\F\in \P$ such that $\Hs\preceq_{s} \F \preceq_{s} \G$ and therefore, by \ref{rel.s.t.i} that $\Hs\preceq_{t} \F \preceq_{t} \G$, and this contradicts the assumption that  $\Hs \precdot_{t} \G$.

%

\subsection{Proof of Proposition~\ref{THM:neighbor-submodels} }\label{SUP.SEC:proof-neighbor-submodels}
The relationship between submodels and coloured graphs is given in \ref{S1}, \ref{S2} and \ref{S3} of Section~\ref{SEC:lattice-structure-RCON}, and we recall that \ref{S1} is in fact redundant. Hence, if we consider a \PDCG\ $\G\in \P$ then the graph $\Hs\in \P$ identifies a submodel of $\P(\G)$ if and only if the colour classes of $(\V_{\Hs}, \E_{\Hs})$ of $\Hs$ are obtained from the classes  $(\V_{\G}, \E_{\G})$ of $\G$ by applying one of the following operations:
\begin{enumerate}[label = (\alph*), ref=(\alph*)]
	\item Exactly one pair of atomic vertex colour classes of $\V_{\G}$ are merged to obtain a vertex twin-pairing colour class;\label{Aa}
	\item exactly one pair of atomic edge colour classes of $\E_{\G}$ are merged to obtain an edge twin-pairing colour class;\label{Ab}
	\item exactly one atomic colour class is removed from $\E_{\G}$;\label{Ac}
	\item exactly one twin-pairing colour class is removed from $\E_{\G}$;\label{Ad}
	\item more than one, not necessarily distinct, of the above operations  \ref{Aa} to \ref{Ad} are applied.\label{Ae}
\end{enumerate}
It is straightforward to see that the graph $\Hs$ obtained from the application of one of the conditions from \ref{Aa} to \ref{Ad} is a neighboring submodel of $\G$, \ie, $\Hs\precdot_{s} \G$ whereas \ref{Ae} produces a graph $\Hs\preceq_{s} \G$ but such that $\Hs\not\precdot_{s} \G$. We now consider each of the cases from \ref{Aa} to \ref{Ad}, in turn, and show that these give the corresponding graphs from (i) to (vii) as given in the text of the proposition. Consider condition \ref{Aa}, and assume, without loss of generality, that  $\{i\},\{\tau(i)\}\in \V_{\G}$ for $i\in \LL_{\G}$. If these two classes are merged to obtain $\{i, \tau(i)\}\in \V_{\Hs}$ then the resulting graph $\Hs$ is such that $\E_{\Hs}=\E_{\G}$ so that both $E_{\Hs}=E_{\G}$ and $\EE_{\Hs}=\EE_{\G}$. On the other hand, $\LL_{\Hs}=\LL_{\G}\setminus \{i\}$, thereby giving (i). Consider condition \ref{Ab} and assume, without loss of generality, that  $\{(i, j)\},\{\tau(i, j)\}\in \E_{\G}$ for $(i,j)\in \EE_{\G}$. If the latter two classes are merged to obtain $\{(i,j), \tau(i,j)\}\in\E_{\Hs}$ then the resulting graph $\Hs$ has the same edge set as $\G$, $E_{\Hs}=E_{\G}$ and, obviously, also  $\V_{\Hs}=\V_{\G}$ so that   $\LL_{\Hs}=\LL_{\G}$. Thus, the only difference between $\G$ and $\Hs$ is that $\EE_{\Hs}=\EE_{\G}\setminus \{(i,j)\}$ thereby giving (ii). Consider now condition \ref{Ac} and assume that the vertex $e\in E_{\G}$ forms an atomic colour class $\{e\}\in \E_{\G}$.
Removing $\{e\}$ from $\E_{\G}$  leaves the vertex classes unchanged so that $\V_{\Hs}=\V_{\G}$ and consequently, $\LL_{\Hs}=\LL_{\G}$ whereas the corresponding edge has to be removed so that
$E_{\Hs}=E_{\G}\setminus e$. Furthermore, if  $e\neq\tau(e)$ and $\{\tau(e)\}\in \E_{\G}$ then either $e\in \EE_{\G}$, so that
$\EE_{\Hs}=\EE_{\G}\setminus e$, or $\tau(e)\in \EE_{\G}$, so that
$\EE_{\Hs}=\EE_{\G}\setminus \tau(e)$. More specifically, we can consider three different types of atomic edge colour classes. The first type of atomic edge colour class  $\{(i, j)\}\in \E_{\G}$ is such that $(i,j)\neq \tau(i,j)$ and $\{\tau(i, j)\}\in \E_{\G}$. Hence, if $(i,j)\in \EE_{\G}$ then we obtain (iii) whereas $\tau(i,j)\in \EE_{\G}$ gives (iv). The second type of atomic edge colour class $\{(i, j)\}\in \E_{\G}$ is such that $\{\tau(i, j)\}\notin \E_{\G}$ so that $E_{\Hs}=E_{\G}\setminus \{(i, j)\}$ and $\EE_{\Hs}=\EE_{\G}$ as in (v). Finally, the third type of atomic edge colour class has the form $\{(i, \tau(i))\}\in \E_{\G}$ for $i\in V$ so that $E_{\Hs}=E_{\G}\setminus \{(i, \tau(i))\}$ and $\EE_{\Hs}=\EE_{\G}$ as in (vi).  We turn now to the case \ref{Ad}.
Removing the colour class  $\{(i, j), \tau(i, j)\}\in \E_{\G}$  leaves the vertex colour classes unchanged so that $\LL_{\Hs}=\LL_{\G}$ and, furthermore, also $\EE_{\Hs}=\EE_{\G}$ because both $(i,j)\notin\EE_{\G}$ and  $\tau(i,j)\notin\EE_{\G}$. On the other hand, $E_{\G}\setminus\{(i,j), \tau(i,j)\}$. Hence, in order to obtain (vii) it is sufficient to recall that  $\{(i, j), \tau(i, j)\}\in \E_{\G}$ if and only if $i\neq \tau(i)$, both $(i,j)\in E_{\G}$ and  $\tau(i,j)\in E_{\G}$
and, furthermore, both $(i,j)\notin \EE_{\G}$ and $\tau(i,j)\notin \EE_{\G}$.

%

\subsection{Proof of Corollary~\ref{THM:equivalent.meet}}\label{SUP.SEC:proof-equivalent.meet}
The first statement can be easily obtained by applying the meet operation as given in point \ref{meet.twin.lattice} of Theorem~\ref{THM:LatticePropertiesofOrderTau} to all the pairs of  $\preceq_{t}$--incomparable  neighbouring submodels of $\G$ as given in  Corollary~\ref{THM:two.layer}, so as to check that the submodel encoding both the constraints of $\Hs_{1}$ and those of $\Hs_{2}$ is given by $\P(\Hs_{1}\wedge_{t}\Hs_{2})$. We turn now to the case where $\Hs_{1}$ and $\Hs_{2}$ are comparable.
As shown in Corollary~\ref{THM:two.layer}, if  $\Hs_{1}\preceq_{t}\Hs_{2}$ then, for a given edge  $(i,j)\in\EE_{\G}$,  $\Hs_{2}$ is obtained from (ii) of  Proposition~\ref{THM:neighbor-submodels} and $\Hs_{1}$ is obtained from either (iii) or (iv) of the same proposition. If  $\Hs_{1}$ is obtained from (iii)  we let $\Hs_{1}^{\prime}$ be the graph obtained from (iv) whereas if  $\Hs_{1}$ is obtained from (iv) we let $\Hs_{1}^{\prime}$ be the graph obtained from (iii). Then, one can check that $\Hs_{1}\wedge_{s}\Hs_{2}=\Hs_{1}\wedge_{s}\Hs_{1}^{\prime}$ and the result follows because $\Hs_{1}$ and $\Hs_{1}^{\prime}$ are  $\preceq_{t}$--incomparable.
%

%

\subsection{Proof of Corollary~\ref{THM:iterative.application.meet}}\label{SUP.SEC:proof-iterative.application.meet}
First, we notice that the equality $\{\F \wedge_{s} \Hs \mid \F \in  \mathcal{A} \setminus \{\Hs\}\}=\{\F \wedge_{t} \Hs \mid \F \in  \mathcal{A} \setminus \{\Hs\}\}$ follows immediately from Corollary~\ref{THM:equivalent.meet} because we the graphs in $\mathcal{A}$ are pairwise $\preceq_{t}$--incomparable. Next we show that,
(a) for every $\F \in  \mathcal{A} \setminus \{\Hs\}$ it holds that $\Hs \wedge_{t} \F \precdot_{s} \Hs$, and (b) for every $\F_{1}, \F_{2} \in  \mathcal{A} \setminus \{\Hs\}$ it holds that $\Hs \wedge_{t} \F_{1}$ and $\Hs \wedge_{t} \F_{2}$ are $\preceq_{t}$--incomparable.

We start from (a). Because $\mathcal{A}$ contains neighbouring submodels of $\G$ it follows that both $\F$ and $\Hs$ are obtained from $\G$ by applying one of the points from (i) to (vii) of Proposition~\ref{THM:neighbor-submodels}. Then, it is easy to check that for all possible combinations of $\F$ and $\Hs$  it holds that $\F\wedge_{t}\Hs\precdot_{s}\Hs$, with the exception where, for a given edge $(i,j)\in \EE_{\G}$, $\F$ and $\Hs$ are obtained one from (ii) and  the other from either (iii) or (iv), but this is not possible because, as shown in Corollary~\ref{THM:two.layer}, in this case $\F$ and $\Hs$ would not be $\preceq_{t}$--incomparable.

We now show (b) contradiction. We have shown above that both $\Hs\wedge_{t}\F_{1}\precdot_{s} \Hs$  and $\Hs\wedge_{t}\F_{2}\precdot_{s}\Hs$ so that, if we assume that $\Hs\wedge_{t}\F_{1}$  and $\Hs\wedge_{t}\F_{2}$ are $\preceq_{t}$--comparable, then  Corollary~\ref{THM:two.layer} implies that, without loss of generality, for a given edge $(i,j)\in \EE_{\Hs}$, $\Hs\wedge_{t}\F_{1}$ is obtained from $\Hs$ by applying (ii) of Proposition~\ref{THM:neighbor-submodels} and
$\Hs\wedge_{t}\F_{2}$ by either (iii) or (iv) of the same proposition. If $\Hs\wedge_{t}\F_{2}$ is obtained from (iii) then,
\begin{eqnarray}
	\label{EQ1}
	\Hs \wedge_{t} \F_{1}
	= (V, E_{\Hs} \cap E_{1}, \LL_{\Hs} \cap \LL_{1}, \EE_{\Hs} \cap \EE_{1} )
	= &(V, E_{\Hs}\phantom{\setminus \{(i,j)\}},\LL_{\Hs}, \EE_{\Hs}\setminus \{(i,j)\}),\\  \label{EQ2}
	\Hs \wedge_{t} \F_{2}
	= (V, E_{\Hs} \cap E_{2}, \LL_{\Hs} \cap \LL_{2}, \EE_{\Hs} \cap \EE_{2})
	=& (V, E_{\Hs}\setminus \{(i,j)\},\LL_{\Hs}, \EE_{\Hs}\setminus \{(i,j)\}),
\end{eqnarray}
where  $\F_{1} = (V, E_{1}, \LL_{1}, \EE_{1})$ and $\F_{2} = (V, E_{2}, \LL_{2}, \EE_{2})$ are both the neighbouring submodels of $\G$.
The equation \eqref{EQ1} implies that $E_{1} = E_{\G}$, $\LL_{1} = \LL_{\G}$ and $\EE_{1} = \EE_{\G}\setminus \{(i,j)\}$ so that
$\F_{1} = (V, E_{\G}, \LL_{\G}, \EE_{\G}\setminus \{(i,j)\})$. Moreover, the equation \eqref{EQ2} implies that $E_{2} = E_{\G}\setminus \{(i,j)\}$, $\LL_{2} = \LL_{\G}$ and $\EE_{2} = \EE_{\G}\setminus \{(i,j)\}$ so that $\F_{2} = (V, E_{\G}\setminus \{(i,j)\}, \LL_{\G}, \EE_{\G}\setminus \{(i,j)\})$. Hence, $\F_{1} \preceq_{t} \F_{2}$ which contradicts the assumption that $\F_{1}$ and $\F_{2}$ are the elements in $\mathcal{A}$ so that they are $\preceq_{t}$--incomparable.  The conclusion is the same when $\Hs\wedge_{t}\F_{2}$ is obtained from (iv).


\vskip 0.2in
\bibliography{TwinLatReference}

\end{document}